\documentclass[11pt,a4paper]{article}
\usepackage[left=2cm,right=2cm,top=3cm,bottom=3cm]{geometry}

\usepackage[utf8]{inputenc} 
\usepackage[T1]{fontenc}    
\usepackage{hyperref}       
\usepackage{url}            
\usepackage{booktabs}       
\usepackage{amsfonts}       
\usepackage{nicefrac}       
\usepackage{microtype}      
\usepackage{amsmath,amsfonts,amsthm,amssymb,xspace,bm}
\usepackage{verbatim,dsfont,mathtools,algorithm,algorithmic, url, color}
\usepackage{booktabs}
\newcommand{\ra}[1]{\renewcommand{\arraystretch}{#1}}
\usepackage{multirow}
\usepackage{xfrac}
\usepackage{wrapfig}
\usepackage{varwidth}
\usepackage[table]{xcolor}
\usepackage{enumitem}

\usepackage{color}
\usepackage{epstopdf}
\usepackage{appendix}
\usepackage{epstopdf}
\usepackage{enumitem}
\usepackage{pifont}

\usepackage{authblk}

\theoremstyle{plain}
\newtheorem{theorem}{Theorem}[section]
\newtheorem{lemma}[theorem]{Lemma}
\newtheorem*{lemma*}{Lemma}

\newtheorem{corollary}[theorem]{Corollary}
\newtheorem{definition}[theorem]{Definition}
\newtheorem{assumption}[theorem]{Assumption}
\theoremstyle{definition}
\newtheorem{remark}{Remark}

\newcommand{\algo}{\texttt{FGD}\xspace}
\newcommand{\palgo}{\texttt{ProjFGD}\xspace}
\newcommand{\R}{\mathbb{R}}

\DeclareMathOperator{\trace}{\textsc{Tr}}
\DeclareMathOperator*{\argmin}{argmin}

\newcommand{\ip}[2]{\left\langle #1, #2 \right\rangle}

\newcommand{\norm}[1]{\left \Vert #1\right \Vert}
\newcommand{\dist}{{\rm{\textsc{Dist}}}}

\newcommand{\obs}{y}
\newcommand{\linmap}{\mathcal{A}}

\newcommand{\C}{\mathcal{C}}
\newcommand{\X}{X}

\newcommand{\Xo}{X^{\star}}

\newcommand{\U}{U}

\newcommand{\f}{f}

\newcommand{\gradf}{\nabla \f}

\newcommand{\Uo}{U^{\star}}

\newcommand{\Q}{Q}

\newcommand{\weta}{\widehat{\eta}}
\newcommand{\Rus}{R_{\U_t}^\star}
\newcommand{\Uw}{\widetilde{\U}}

\def\U{U}
\def\V{V}
\def\Q{Q}
\def\Y{Y}

\def\obs{y}
\def\linmap{\mathcal{A}}

\newcommand{\E}{\mathcal{E}}

\title{Provable Burer-Monteiro factorization \\ for a class of norm-constrained matrix problems}

\author
       {Dohyung Park$^1$, Anastasios Kyrillidis$^1$, Srinadh Bhojanapalli$^{2}$ \\ \vspace{-0.75em}
       Constantine Caramanis$^1$, and Sujay Sanghavi$^1$
       \\ \vspace{1em}
       $^1$The University of Texas at Austin \\
       \{dhpark, anastasios, constantine\}@utexas.edu, \\ sanghavi@mail.utexas.edu\\ 
       $^2$Toyota Technological Institute at Chicago \\
       srinadh@ttic.edu
        }

\begin{document}
\maketitle

%

\begin{abstract}
We study the projected gradient descent method on low-rank matrix problems with a strongly convex objective.
We use the Burer-Monteiro factorization approach to implicitly enforce low-rankness; such factorization introduces non-convexity in the objective.
We focus on constraint sets that include both positive semi-definite (PSD) constraints and specific matrix norm-constraints. 
Such criteria appear in quantum state tomography and phase retrieval applications. 

We show that non-convex projected gradient descent favors local linear convergence in the factored space. 
We build our theory on a novel \emph{descent lemma}, that non-trivially extends recent results on the unconstrained problem.
The resulting algorithm is \emph{Projected Factored Gradient Descent}, abbreviated as \palgo, and shows superior performance compared to state of the art on quantum state tomography and sparse phase retrieval applications.
\end{abstract}

\section{Introduction}
We consider matrix problems of the form:
\begin{equation}{\label{intro:eq_01}}\vspace{-0.1em}
\begin{aligned}
	& \underset{\X \in \R^{n \times n}}{\text{min}}
	& & f(\X) \quad \text{subject to} \quad \X \succeq 0,~ \X \in \C'. \vspace{-0.1em}
\end{aligned} 
\end{equation} 
$f$ is assumed to be strongly convex and have Lipschitz continuous gradients.
The constraint set contains PSD and additional $\C' \subseteq \R^{n \times n}$ convex constraints on $X$. 

There are several algorithmic solutions for \eqref{intro:eq_01}, operating on the variable space $\X \in \R^{n \times n}$. 
We mention \cite{jain2010guaranteed, becker2013randomized, balzano2010online, boumal2011rtrmc, wen2012solving, lee2010admira, kyrillidis2014matrix, jain2013low, tanner2013normalized, hardt2014fast, zhang2015global, lin2010augmented, becker2011templates, cai2010singular, becker2011nesta, chen2014coherent, yurtsever2015universal} and point to references therein.
Most of these schemes focus on the matrix sensing / matrix completion problem, and, thus, are designed for specific instances of $f$.
Moreover, the majority does not directly handle additional constraints.

More importantly, these methods often involve computationally expensive eigen-value/vector computations --at least once per iteration-- in order to satisfy the PSD constraint. 
This constitutes their computational bottleneck in large-scale settings.
Thus, it is desirable to find algorithms that scale well in practice.

\paragraph{Our approach.}
One way to avoid this is by \emph{positive semi-definite factorization} $X = UU^\top$.
In particular, we solve instances of \eqref{intro:eq_01} in the factored form as follows:
\begin{equation}{\label{intro:eq_00}}
\begin{aligned}
	& \underset{\U \in \R^{n \times r}}{\text{minimize}}
	& & f(\U\U^\top) \quad \quad \text{subject to} \quad \U \in \C.
\end{aligned}
\end{equation} 
This parametrization, popularized by Burer and Monteiro \cite{burer2003nonlinear, burer2005local}, naturally encodes the PSD constraint, removing the expensive eigen-decomposition projection step.
$r$ can be set to $\text{rank}(X^\star) := r^\star$, where $\X^\star$ is the optimal solution to \eqref{intro:eq_01}, but is often set to be much smaller than $r^\star$: in that case, factor $\U \in \R^{n \times r}$ contains much less variables to maintain and optimize than $\X = \U\U^\top$. 
Thus, such parametrization also makes it easier to update and store the iterates $\U$. 
By construction, the $\U\U^\top$ is a PSD solution, which could also be low-rank. 

$\C \subseteq \R^{n \times r}$ is a compact convex set, that models well $\C'$ in \eqref{intro:eq_01}. 
While in practice we can assume any such constraint $\C$ with \emph{tractable} Euclidean projection operator\footnote{In general, one could artificially introduce a constraint $\C$ in \eqref{intro:eq_00} --even when no $\C'$ constraint is present in \eqref{intro:eq_01}-- for better interpretation of results. 
}, in our theory we mostly focus on sets $\C$ that satisfy the following assumption.
\begin{assumption}{\label{ass:00}}
For $\X \succeq 0$, there is $\U \in \R^{n \times r}$ and $r \leq n$ such that $\X = \U\U^\top$. 
Then, $\C' \in \R^{n \times n}$ is endowed with constraint set $\C \subseteq \R^{n \times r}$ that $(i)$ for each $\X \in \C'$, there is an subset in $\C$ where each $\U \in \C$ satisfies $\X = \U\U^\top$ 
(see Section \ref{sec:when} for more details), 
and $(ii)$ its projection operator, say $\Pi_{\C}(V)$ for $V \in \R^{n \times r}$, is an entrywise scaling operation on the input $V$.
\end{assumption}

Part $(i)$ in Assumption \ref{ass:00} is required in our analysis in order to claim convergence also in the $\X$ space, through the factored $\U$ space. 
In other case, our theory still leads just to convergence in $\U$.
We defer this discussion to Section \ref{sec:projFGD}.

Criteria of the form \eqref{intro:eq_01} appear in applications from diverse research fields.
As an exemplar, consider density matrix estimation of quantum systems \cite{aaronson2007learnability, gross2010quantum, kalev2015quantum}, where $\C'$ is defined as  $\trace(X) \leq 1$ and satisfies Assumption \ref{ass:00}.
We also experimentally consider problems that only approximately satisfy the above assumption, such as sparse phase retrieval applications \cite{jaganathan2013sparse, shechtman2014gespar, Ohlsson2012CPRL},
and sparse PCA \cite{laue2012hybrid}: in these cases,  $\C'$ contains $\ell_1$-norm constraints on $\X$; see Section \ref{sec:motiv} for more details.


\paragraph{Contributions.}
Our aim is to broaden the results on efficient, non-convex recovery for \emph{constrained} low-rank matrix problems.
Our developments maintain a connection with analogous results in convex optimization, where standard assumptions are made: we consider the common case where 
$f$ is (restricted) smooth and (restricted) strongly convex \cite{agarwal2010fast}. 
We provide experimental results for two important tasks in physical sciences: quantum state density estimation and sparse phase retrieval. 

Some highlights of our results are the following:
\begin{itemize}[leftmargin = 0.4cm] \vspace{-0.2cm}
\item A key property for proving convergence in convex optimization is the notion of \emph{descent}. \emph{I.e.}, given current, next and optimal points in the $\X$ space, say $\X_t, \X_{t+1}, \X^\star$, respectively, and the recursion $\X_{t+1} = \X_t - \eta \nabla f(\X_t)$, the condition $\left\langle \X_t - \X_{t+1}, ~\X_t - \X^\star\right\rangle \geq C$ --where $C > 0$ depends on the gradient norm-- implies that $\X_{t+1}$ ``moves" towards the correct direction.\footnote{To see this, observe that $\X^\star - \X_t$ is the best possible direction to follow, while $\X_t - \X_{t+1}$ is the direction we actually follow. Then, such a condition implies that there is a non-trivial positive correlation between these two directions.} 
In this work, we present a \emph{novel descent lemma} that non-trivially extends such conditions for the constrained case in \eqref{intro:eq_00}, under the assumptions mentioned above.
We hope that this result will trigger more attempts towards more generic convex sets. \vspace{-0.1cm}
\item We propose \palgo, a non-convex projected gradient descent algorithm that solves instances of \eqref{intro:eq_00}. \palgo has favorable local convergence guarantees when $f$ is (restricted) smooth and (restricted) strongly convex.
We also present an initialization procedure with guarantees in the supplementary material. \vspace{-0.1cm}
\item Finally, we extensively study the performance of \palgo on two problem cases: $(i)$ quantum state tomography and $(ii)$ sparse phase retrieval. Our findings show significant acceleration when \palgo is used, as compared to state of the art.
\end{itemize}

\subsection{Related work}{\label{sec:related}}

The work of \cite{chen2015fast} proposes a first-order algorithm for \eqref{intro:eq_00},
where the nature of $\C$ is more generic, and depends on the problem at hand.
The authors provided a set of conditions (local descent, local Lipschitz, and local smoothness) under which one can prove convergence to an $\varepsilon$-close solution with $O(1/\varepsilon)$ or $O(\log(1/\varepsilon))$ iterations. While the convergence proof is general, checking whether the three conditions hold is a non-trivial problem and requires different analysis for each problem. 
We believe this paper complements \cite{chen2015fast}: in the latter, the closest to our constraints are these of ``max-norm" incoherence constraints; however, in our case, the objective function needs to only satisfy standard strongly convex and smoothness assumptions.

\cite{bhojanapalli2015dropping} proposes the \emph{Factored Gradient Descent} (\algo) algorithm for \eqref{intro:eq_00}, where $\C \equiv \R^{n \times r}$. 
\algo is also a first-order scheme. Key ingredient for convergence is a novel step size selection that can be used for any $f$, as long as it is Lipschitz gradient smooth (and strongly convex for faster convergence).
However, \cite{bhojanapalli2015dropping} cannot accommodate any constraints on $\U$. 

Concurrently, \cite{zhao2015nonconvex} presents a new analysis that handles non-square cases in \eqref{intro:eq_00}.
In that case, we look for a factorization $\X = \U\V^\top \in \R^{n \times p}$. 
The idea is based on the inexact first-order oracle, previously used in \cite{balakrishnan2014statistical}.
Similarly to \cite{bhojanapalli2015dropping}, the proposed theory does not handle any constraints. 


\paragraph{Roadmap.} 
Section \ref{sec:prelim} contains some basic definitions and assumptions that are repeatedly used in the main text. 
Section \ref{sec:projFGD} describes \palgo and its theoretical guarantees. 
In Section \ref{sec:motiv}, we motivate the necessity of \palgo via some applications; due to space limitations, only one application is described in the main text (the second application is included in the supplementary material).
This paper concludes with a discussion on future directions in Section \ref{sec:discussion}.
Supplementary material contains further experiments, all proofs of theorems in main text, and a proposed initialization procedure.
\section{Preliminaries}{\label{sec:prelim}}
\noindent \textbf{Notation.} For matrices $\X, \Y \in \R^{n \times n}$, $\ip{\X}{\Y} = \trace\left(\X^\top \Y \right)$ represents their inner product. 
$\X \succeq 0$ denotes $\X$ is a positive semi-definite (PSD) matrix. 
We use $\norm{\X}_F$ and $\sigma_1(\X)$ for the Frobenius and spectral norms of a matrix, respectively; we also use $\|\X\|_2$ to denote the spectral norm.
Moreover, we denote as $\sigma_i(\X)$ the $i$-th singular value of $\X$.
$\X_r$ denotes the best rank-$r$ approximation of $\X$. 
For $\X$ such that $\X = \U\U^\top$, the gradient of $f$ with respect to $\U$ is $\left(\gradf(\U \U^\top) + \gradf(\U \U^\top)^\top\right)\U$. 
If $f$ is also symmetric, \emph{i.e.}, $f(\X) = f(\X^\top)$, then $ \gradf(\X) = 2\gradf(\X) \cdot \U$. 

An important issue in optimizing $f$ over the factored space is the existence of non-unique possible factorizations. 
We use the following rotation invariant distance metric:
\begin{definition}{\label{prelim:def_04}}
Let matrices $\U, \V \in \R^{n \times r}$. Define:
\begin{align*}
\dist\left(U, V\right) :=\min_{R: R \in \mathcal{O}} \norm{U - V R}_F, 
\end{align*} where $\mathcal{O}$ is the set of $r \times r$ orthonormal matrices $R$.
\end{definition}

\noindent \textbf{Assumptions.} We consider applications that can be described by \emph{strongly} convex functions $f$ with \emph{gradient Lipschitz continuity}.\footnote{Our ideas can be extended in a similar fashion to the case of restricted smoothness and restricted strong convexity \cite{agarwal2010fast}.} We state these standard definitions below for the square case.

\begin{definition}{\label{prelim:def_00}}
Let $f: \R^{n \times n} \rightarrow \R$ be convex and differentiable. $f$ is $\mu$-strongly convex if $\forall \X, \Y \in  \R^{n \times n}$
\begin{equation}\label{eq:sc}
f(\Y) \geq f(\X) + \ip{\gradf\left(\X\right)}{\Y - \X} + \tfrac{\mu}{2} \norm{Y - \X}_F^2.
\end{equation}
\end{definition}

\begin{definition}{\label{prelim:def_01}}
Let $f: \R^{n \times p} \rightarrow \R$ be a convex differentiable function. $f$ is gradient Lipschitz continuous with parameter $L$ (or $L$-smooth) if $\forall \X, \Y \in  \R^{n \times n}$
\begin{equation}
\norm{\gradf\left(\X\right) - \gradf\left(\Y\right)}_F \leq L \cdot \norm{\X - \Y}_F.
\end{equation}
\end{definition} 

For our proofs, we will also make the \emph{faithfulness} assumption, as in \cite{chen2015fast}:
\begin{definition}{\label{prelim:def_02}}
Let $\E$ denote the set of equivalent factorizations that lead to a rank-$r$ matrix $\Xo \in \R^{n \times n}$; 
i.e.,
$\E := \left\{ \Uo \in \R^{n \times r}~:~ \Xo = \Uo\U^{\star\top} \right\}.$
Then, we assume $\E \subseteq \C$, \emph{i.e.}, the resulting convex set $\C$ in \eqref{intro:eq_00} (from $\C'$ in \eqref{intro:eq_01}) \emph{respects the structure of $\E$}.
\end{definition}
This assumption is necessary for arguments regarding the quality of solution obtained in the factored $\U$ space, w.r.t. the original $\X$ space.
\section{The Projected Factored Gradient Descent (\palgo) algorithm}{\label{sec:projFGD}}

Let us first describe the \palgo algorithm, a projected, first-order scheme. 
The discussion in this part holds for any constraint set $\C$; later in the text, in order to obtain theoretical guarantees, we make further assumptions---such as Assumption \ref{ass:00}.

The pseudocode is provided in Algorithm \ref{algo:projFGD}.
Let  $\Pi_{\mathcal{C}}\left(\V\right)$ denote the projection of an input matrix $\V \in \R^{n \times r}$ onto the convex set $\C$.
For initialization, the starting point is computed as follows: we first compute $\X_0 := \sfrac{1}{\widehat{L}} \cdot \Pi_{+}\left(-\nabla f(0) \right)$, where $\Pi_{+}(\cdot)$ denotes the projection onto the set of PSD matrices and $\widehat{L}$ represents an approximation of $L$. 
Then, \palgo requires a top-$r$ SVD calculation, \emph{only once}, to compute $\widetilde{\U}_0 \in \R^{n \times r}$, such that $\X_0 = \widetilde{\U}_0 \widetilde{\U}_0^\top$; using $\widetilde{\U}_0$, the initial point $\U_0$ satisfies $\U_0 = \Pi_{\C}\left( \widetilde{\U}_0 \right)$, in order to accommodate constraints $\C$.

The main iteration of \palgo applies the simple rule:
\begin{align*}
\U_{t+1} = \Pi_{\mathcal{C}}\left(\U_t - \eta \gradf(\U_t\U_t^\top)  \cdot \U_t\right),
\end{align*} 
with step size:
\begin{align}\label{eq:step_size}
\eta \leq \tfrac{1}{128\left(L\norm{\X_0}_2 + \norm{\gradf(X_0) }_2\right)}.
\end{align}
Here, one can use $\widehat{L}$ to approximate $L$.

\begin{algorithm}[!t]
	\caption{\palgo method}\label{algo:projFGD}
	\begin{algorithmic}[1]		
		\STATE \textbf{Input:} Function $f$, target rank $r$, \# iterations $T$. 
		\STATE Compute $\X_0 := \sfrac{1}{\widehat{L}} \cdot \Pi_{+}\left(-\nabla f(0) \right)$.
		\STATE Set $\widetilde{\U}_0 \in \R^{n \times r}$ such that $\X_0 = \widetilde{\U}_0 \widetilde{\U}_0^\top$.
		\STATE Compute $\U_0 = \Pi_{\C}\left( \widetilde{\U}_0 \right)$.
		\STATE Set step size $\eta$ as in \eqref{eq:step_size}.
		\FOR {$t=0$ to $T-1$}
			\STATE $\U_{t+1} = \Pi_{\mathcal{C}}\left(\U_t - \eta \gradf(\U_t\U_t^\top)  \cdot \U_t\right)$.
		\ENDFOR
		\STATE \textbf{Output:} $\X = \U_T\U_T^\top$. 	
	\end{algorithmic}
\end{algorithm}

Key ingredients to achieve provable convergence are initialization--so that initial point $\U_0$ leads to $\dist(\U_0, \Uo)$ sufficiently small-- and the step size selection. 
For the initialization, apart from the procedure mentioned above, we could also use more specialized spectral methods --see \cite{chen2015fast, zheng2015convergent}-- or even run algorithms on \eqref{intro:eq_01} for only a few iterations --this requires further full or truncated SVDs \cite{tu2015low}.
The discussion regarding our initialization and what type of guarantees one obtains is deferred to the supplementary material.

\subsection{When constrained non-convex problems can be scary?}{\label{sec:when}}

In stark contrast to the \emph{convex} projected gradient descent method,
proving convergence guarantees for \eqref{intro:eq_00} is not a straightforward task. 
First, if we are interested in quantifying the quality of the solution in the factored space w.r.t. $\Xo$,
$\C$ should be \emph{faithful}, according to Definition \ref{prelim:def_02}.
Furthermore, there should exist a mapping $\U \mapsto \X$ 
that relates the constraint set $\C'$, in the original variable space (see \eqref{intro:eq_01}), to the factored one $\C$ (see \eqref{intro:eq_00}). 
In that case, claims about convergence to a point $\Uo$, in the factored space, can be ``transformed" into claims about convergence to a point close to $\X^\star$, in the original space, that also satisfies the constraints.
This is the case for the following constraint case: for any $\X = \U\U^\top$, $\trace(\X) \leq \lambda \Leftrightarrow \|\U\|_F^2 \leq \lambda$, and, thus, satisfying $\|\U\|_F^2 \leq \lambda$, for any $\U$, guarantees that $\trace(\X) \leq \lambda$ for $\X = \U\U^\top$.
Apart from the example above, other characteristic cases include Schatten norms. 

Contrary to this example, consider the case
$\C' := \left\{\X \in \R^{n \times n}:~ \|\X\|_1 \leq \lambda'\right\}$, where, $\|\X\|_1 = \sum_{ij} |\X_{ij}|$.
A natural choice for $\C$ would be $\C := \left\{\U \in \R^{n \times r}:~ \|\U\|_1 \leq \lambda\right\}$, for $\lambda, \lambda' > 0$; however, depending on the selection of $\lambda$, points in $\U \in \C$ might result into points $\X = \U\U^\top$ that $\X \not\in \C'$.
In this case, $\Uo$ of \eqref{intro:eq_00} could be $\not\in \E$ and, thus, convergence guarantees to $\Uo$ might have no meaning in the convergence in $\X$ space.
However, as we show in Section 6.1, in this case $\C$  ``simulates" well $\C'$: if $\U$ is sparse enough, then $\X = \U\U^\top$ could also be sparse, so proper selection of $\lambda$ plays a key role.
Even in this case, \palgo performs competitively compared to state-of-the-art approaches. 

Second, 
the projection step itself complicates considerably the non-convex analysis, as we show in the supplementary material.
In our theory, we focus on convex sets $\C$ that satisfy \eqref{intro:eq_proj} where $\Pi_\C(\V)$ can be equivalently seen as scaling the input.
\emph{E.g.}, when $\C \equiv \left\{\U \in \R^{n \times r}:~\|\U\|_F \leq \lambda\right\}$, $\Pi_\C(\V) = \xi(\V) \cdot \V$ where 
$\xi(\V) := \tfrac{\lambda}{\|\V\|_F}$, for $\V \not\in \C$. 
Our theory highlights that, even for this simple case, proving convergence is not a straightforward task.

\subsection{Theoretical guarantees of {\rm \palgo} for $\C := \{ U \in \R^{m \times r} : \|U\|_F \le \lambda \}$}

We provide theoretical guarantees for \palgo in the case where the constraint satisfies 
\begin{align}{\label{intro:eq_proj}}
\Pi_\C(\V) = \argmin_{\U \in \C} \tfrac{1}{2} \|\U - \V\|_F^2 = \left\{
	\begin{array}{ll}
		\V  & \mbox{if } \V \in \C, \\
		\xi(V) \cdot \V & \mbox{if } \V \not\in \C,
	\end{array}
\right.
\end{align} 
\emph{i.e.}, the projection operation is an \emph{entry-wise scaling}.
Such settings include the Frobenius norm constraint $\C = \{ U \in \R^{m \times r} : \|U\|_F \le \lambda \}$, which appears in quantum state tomography. 
Moreover, for this case, the constraint has one-to-one correspondence with the trace constraint in the original $X$ space; thus any argument in the $\U$ space applies for the $\X$ space also.

We assume the optimum $\X^\star$ satisfies $\text{rank}(\X^\star) = r^\star$. 
For our analysis, we will use the following step sizes: 
\begin{align*}
\widehat{\eta} &= \tfrac{1}{128(L\|\X_t\|_2 + \|Q_{\U_t} Q_{\U_t}^\top\gradf(\X_t)\|_2)}, \\ \eta^\star &= \tfrac{1}{128(L\|\X^\star\|_2 + \|\gradf(\X^\star)\|_2)}, 
\end{align*} 
where $Q_A$ is a basis for column space of  $A$.
By Lemma A.5 in \cite{bhojanapalli2015dropping}, we know that $\widehat{\eta} \geq \tfrac{5}{6} \eta$ and $\tfrac{10}{11}\eta^\star \leq \eta \leq \tfrac{11}{10} \eta^\star$. 
Due to such relationships, in our proof we will work with step size $\widehat{\eta}$: this is equivalent --up to constants-- to the original step size $\eta$, used in the algorithm. 
Thus, any results below will automatically imply similar results hold for $\eta$, by using the bounds between step sizes.

\begin{theorem}[(Local) Convergence rate for restricted strongly convex and smooth $f$]\label{thm:projFGD_guarantees}
Let $\C \subseteq \R^{n \times r}$ be a convex, compact, and faithful set, with projection operator satisfying \eqref{intro:eq_proj}. 
Let $\U_t \in \C$ be the current estimate and $\X_t = \U_t\U_t^\top$.
Assume current point $\U_t$ satisfies $ \dist(\U_t, \Uo) \leq \rho' \sigma_{r}(\Uo)$, for $\rho' := c \cdot \tfrac{\mu}{L} \cdot \tfrac{\sigma_r(\X^\star)}{\sigma_1(\X^\star)}, ~c \leq \tfrac{1}{200}$, and given $\xi_t(\cdot) \gtrsim 0.78 $ per iteration, the new estimate of \palgo, $\U_{t+1} = \Pi_{\mathcal{C}}\left(\U_t - \widehat{\eta} \gradf(\U_t\U_t^\top)  \cdot \U_t\right) = \xi_t \cdot \left(\U_t - \widehat{\eta} \gradf(\U_t\U_t^\top)  \cdot \U_t\right)$ satisfies
\begin{equation}
\dist(\U_{t+1}, \Uo)^2 \leq \alpha \cdot \dist(\U_t, \Uo)^2, \label{conv:eq_01}
\end{equation}
where $\alpha := 1 - \frac{\mu \cdot \sigma_r(\Xo)}{550(L \|\Xo\|_2 + \|\gradf(\Xo)\|_2)} < 1$. Further, $\U_{t+1}$ satisfies $ \dist(\U_{t+1}, \Uo) \leq \rho' \sigma_{r}(\Uo). $
\end{theorem}
 
The complete proof of the theorem is provided in the supplementary material.
The assumption $ \dist(\U_t, \Uo) \leq \rho' \sigma_{r}(\Uo)$ only leads to a local convergence result.  
\cite{chen2015fast} provide some initialization procedures for different applications, where we can find an initial point $\U_0$ such that $ \dist(\U_0, \Uo) \leq \rho' \sigma_{r}(\Uo)$ is satisfied.
In the supplementary material, we present a similar generic initialization procedure that results in exact recovery of the optimum, 
under further assumptions.
We borrow such procedure in Section \ref{sec:motiv} for our experiments.

\noindent \textbf{$\xi_t(\cdot)$ requirement.} The assumption $\xi_t(\cdot) \gtrsim 0.78$ implies the iterates of \palgo (before the projection step) are retained relatively close to the set $\C$.\footnote{Intuitively, we expect the estimates $\U_t$, before the projection, to be further from $\C$ during the first steps of \palgo; as the number of iterations increases, the sequence of solutions gets closer to $\Uo$ and thus $\xi_t(\cdot) \rightarrow 1$.}
For some cases, this can be easily satisfied by setting the step size small enough, as indicated below; the proof can be found in Section 7.

\begin{corollary}{\label{cor:projFGD_schatten}}
If $\C = \left\{\U \in \R^{n \times r}: ~\|\U\|_F \leq \lambda \right\}$, 
then \palgo inherently satisfies $\tfrac{128}{129} \leq \xi_t(\cdot) \leq 1$, for every $t$. I.e., it guarantees \eqref{conv:eq_01} without assumptions on $\xi_t(\cdot)$.
\end{corollary}

We conjecture that the lower bound on $\xi_t(\cdot)$ could possibly be improved with a different analysis.

\noindent \textbf{Key lemma.} 
The proof of the above theorem primarily depends on the following ``descent" lemma for \palgo. 

\begin{lemma}[Descent lemma]\label{lem:gradU,U-U_r_ bound}
Let $\widetilde{\U}_{t+1} = \U_t - \widehat{\eta} \gradf(\X_t^\top)  \cdot \U_t$.
For $f$ $L$-smooth and $\mu$-strongly convex, and under the same assumptions with Theorem \ref{thm:projFGD_guarantees}, the following inequality holds true:
\begin{align*}
&2 \weta \big \langle \gradf(\U_t\U_t^\top) \cdot \U_t, ~\U_t - \Uo R_{\U_t}^\star \big \rangle +  \|\U_{t+1} - \widetilde{\U}_{t+1}\|_F^2 \geq \nonumber \\ 
&\weta^2\|\gradf(\U_t\U_t^\top) U_t\|_F^2 +  \tfrac{3\weta \mu}{10} \cdot \sigma_r(\Xo) \cdot \dist(\U_t, \Uo)^2.
\end{align*}
\end{lemma}

\subsection{Main differences with \cite{bhojanapalli2015dropping}}
In this subsection, we highlight the main differences with the analysis of \cite{bhojanapalli2015dropping}.

As we already mentioned, the proof in \cite{bhojanapalli2015dropping} does not handle constraints. 
In particular, one key factor is the gradient at an optimal point does not vanish and thus the following descent lemma bound (Lemma 6.1 in \cite{bhojanapalli2015dropping}) does not hold:
\begin{align}
&2 \weta \big \langle \gradf(\U_t\U_t^\top) \cdot \U_t, ~\U_t - \Uo R_{\U_t}^\star \big \rangle \geq \nonumber \\ &\tfrac{4\weta^2}{3}\|\gradf(\U_t\U_t^\top) U_t\|_F^2 +  \tfrac{3\weta \mu}{20} \cdot \sigma_r(\Xo) \cdot \dist(\U_t, \Uo)^2. \label{eq:FGD_descent_eq}
\end{align}
To see this, in the unconstrained case, if $\U_t \equiv \Uo$ (up to some rotation), the following holds for \cite{bhojanapalli2015dropping}
\begin{small}
\begin{align*}
&0 := 2 \weta \big \langle \gradf(\Xo) \cdot \Uo, ~\Uo - \Uo R_{\U_t}^\star \big \rangle \geq \tfrac{4\weta^2}{3}\|\gradf(\Xo) \Uo\|_F^2 \\ &+  \tfrac{3\weta \mu}{20} \cdot \sigma_r(\Xo) \cdot \dist(\Uo, \Uo)^2 =: 0,
\end{align*}
\end{small} 
since at the optimum we have $\gradf(\Xo) \Uo = 0.$ 
When the latter does not hold, this descent lemma  does not hold. 
As a simple example where this happens, consider the noisy matrix sensing setting: $y = \mathcal{A}(\X^\star) + w$, where $w$ is a non-negligible noise term. 
Then, for objective $f(X) = \|y - \mathcal{A}(X)\|_2^2$, observe that $$\nabla f(X^\star) = -2\mathcal{A}^*\left(y - \mathcal{A}(\X^\star)\right) = -2\mathcal{A}^*\left(w\right) \neq 0,$$ where $\mathcal{A}^*$ is the adjoint operator for $\mathcal{A}$.
Finally, in \cite{bhojanapalli2015dropping}, Lemma 6.3 (pp.23) assumes zero gradient at the optimum, while our Lemma \ref{lem:DD_bound_sc} (supp.material) follows a different approach. 
\section{Applications}{\label{sec:motiv}}

We present two characteristic applications. 
For each application, we define the problem, enumerate state-of-the-art algorithms and provide numerical results.
We refer the reader to Section 6 for additional experiments. 

\subsection{Quantum state tomography}
Building on Aaronson's work on quantum state tomography (QST) \cite{aaronson2007learnability}, 
we are interested in learning the (almost) \emph{pure}\footnote{Purity is a structural property of the density matrix: A quantum systems is \emph{pure} if its density matrix is rank one and, \emph{almost pure} if it can be well-approximated by a low rank matrix.} $q$-bit state of a quantum system --known as the density matrix-- via a limited set of measurements. 
In math terms, the problem can be cast as follows. 
Let us define the \emph{density matrix} $\X^{\star} \in \mathbb{C}^{n\times n}$ of a $q$-bit quantum system as an unknown Hermitian, positive semi-definite matrix that satisfies $\text{rank}(\X^{\star}) = r$ and is normalized as $\trace(\X^{\star}) = 1$ \cite{gross2010quantum}; here, $n = 2^q$. 
Our task is to recover $\X^\star$ from a set of QST measurements $\obs \in \R^m,~m \ll n^2,$ that satisfy $\obs = \linmap(\X^{\star}) + \eta$.
Here, $(\linmap(\X^{\star}))_i = \trace(E_i \X^{\star})$ and $\eta_i$ can be modeled as independent, zero-mean normal variables. 
The operators $E_i \in  \R^{n \times n}$ are typically the tensor product of the $2\times 2$ Pauli matrices\footnote{\cite{liu2011universal} showed that, for almost all such tensor constructions --of $m = O(r n \log^c n), ~c > 0,$ Pauli measurements-- satisfy the so-called rank-$r$ restricted isometry property (RIP) for all $\X \in \left\{\X~:~\X \succeq 0, \text{rank}(\X) \leq r, \|\X\|_*\le \sqrt{r} \|\X\|_F\right\}$: 
\begin{equation}\label{eq: RIP-Pauli}
\left(1-\delta_r\right)\|\X\|_F^2 \leq \|\linmap(\X)\|_F^2 \leq \left(1+\delta_r\right)\|\X\|_F^2,
\end{equation}
where $\|\cdot\|_*$ is the nuclear norm (i.e., the sum of singular values), which reduces to $\trace(\X)$ since $\X \succeq 0$.} 
\cite{liu2011universal}. 

The above lead to the following \emph{non-convex} problem formulation\footnote{As pointed out in \cite{kalev2015quantum}, it is in fact advantageous in practice to choose $\trace(\X) \neq 1$, as it improves the robustness to noise. Here, we force $\trace(\X) \leq 1$.}:
\begin{equation}{\label{eq: orig QT}}
\begin{aligned}
	& \underset{\X \succeq 0}{\text{minimize}}
	& &  \|\linmap(\X) - \obs\|_F^2 \quad \quad \text{subject to} \quad  \text{rank}(\X)= r, \;\trace(\X)\leq1.
\end{aligned}
\end{equation} 


%
\begin{figure*}[t!]
\centering
\includegraphics[width=0.33\textwidth]{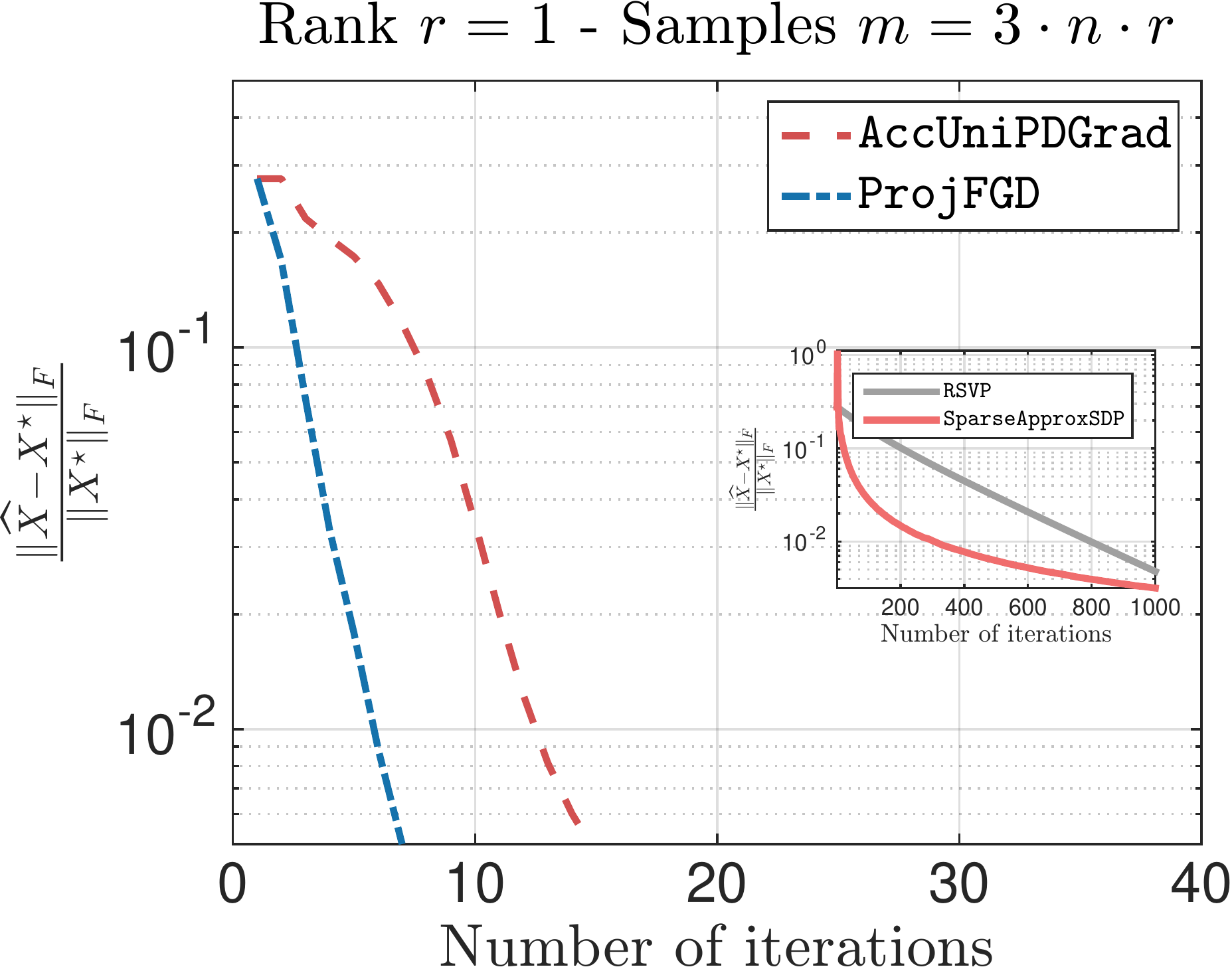} \hspace{-0.2cm}
\includegraphics[width=0.33\textwidth]{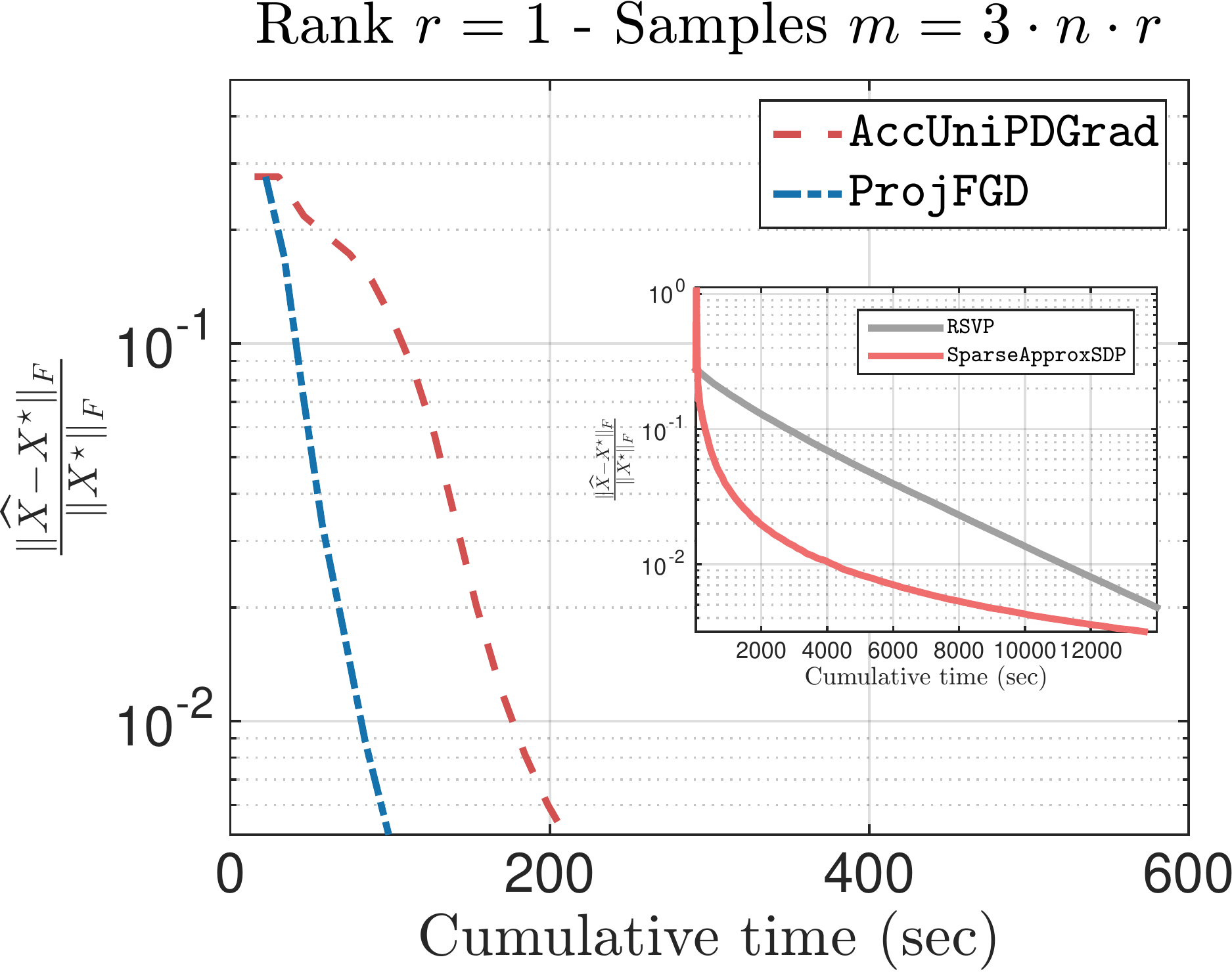} \hspace{-0.2cm}
\includegraphics[width=0.33\textwidth]{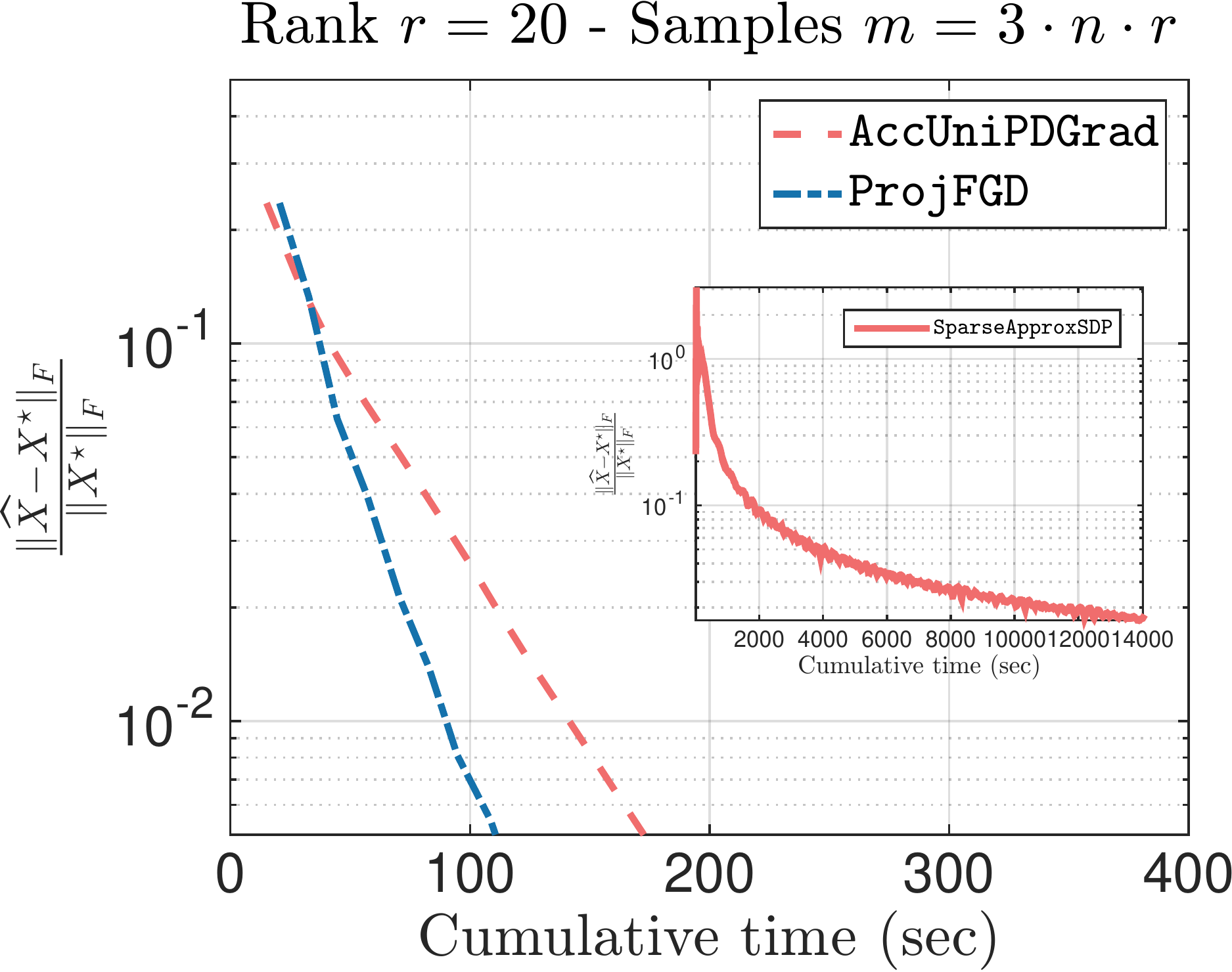} 
\caption{\small{\textbf{Left and middle panels}: Convergence performance of algorithms under comparison w.r.t. $\tfrac{\|\widehat{X} - \X^\star\|_F}{\|\X^\star\|_F}$ vs. $(i)$ the total number of iterations (left) and $(ii)$ the total execution time. Both cases correspond to $C_{\rm sam} = 3$, $r = 1$ (pure setting) and $q = 12$ (\textit{i.e.}, $n = 4096$). \textbf{Right panel}: Almost pure state ($r = 20$). Here, $C_{\rm sam} = 3$. }
}
\label{fig:exp1}
\end{figure*}

\noindent \textbf{State-of-the-art approaches.}
One of the first provable algorithmic solutions for the QST problem was through convexification \cite{recht2010guaranteed}:
this includes nuclear norm minimization approaches \cite{gross2010quantum}, as well as proximal variants, as follows:
\begin{equation} \label{eq:conventional}
\begin{aligned}
	& \underset{\X \succeq 0}{\text{minimize}}
	& & \|\linmap(\X) - \obs\|_F^2 + \lambda \|\X\|_*.
\end{aligned}
\end{equation}
Here, $\|\cdot\|_*$ reduces to $\trace(\X)$ since $\X \succeq 0$. 
This approach is considered in the seminal work \cite{gross2010quantum} and is both tractable and amenable to theoretical analysis. 
The approach does not include any constraint on $\X$.\footnote{\emph{E.g.}, in order to take the trace constraint $\trace(\X) = 1$ into account, either $\lambda$ should be precisely tuned to satisfy this constraint or the final estimator is normalized heuristically to satisfy this constraint \cite{flammia2012quantum}.}
As one of the most recent algorithms, we mention the work of \cite{yurtsever2015universal} where a universal primal-dual convex framework is presented, with the QST problem as application.

From a non-convex perspective, \cite{hazan2008sparse} presents \texttt{SparseApproxSDP} algorithm that solves \eqref{eq: orig QT},
when the objective is a generic gradient Lipschitz smooth function.
\texttt{SparseApproxSDP} solves \eqref{eq: orig QT} by updating a putative low-rank solution with rank-1 refinements, coming from the gradient. 
This way, \texttt{SparseApproxSDP} avoids computationally expensive operations per iteration, such as full SVDs.
In theory, at the $r$-th iteration, \texttt{SparseApproxSDP}
is guaranteed to compute a $\tfrac{1}{r}$-approximate solution, with rank at most $r$, \textit{i.e.}, achieves a sublinear 
$O\left(\tfrac{1}{\varepsilon}\right)$ convergence rate. 
However, depending on $\varepsilon$, \texttt{SparseApproxSDP} might not return a low rank solution. 
Finally, \cite{becker2013randomized} propose Randomized Singular Value Projection (\texttt{RSVP}), a projected gradient descent algorithm for \eqref{eq: orig QT}, which merges gradient calculations with truncated SVDs via randomized approximations for computational efficiency. 

Since the size of these problems grows exponentially with the number of quantum bits,
designing fast algorithms that minimize the computational effort required for \eqref{eq: orig QT} or \eqref{eq:conventional} is mandatory.

\noindent \textbf{Numerical results.} 
In this case, the factorized version of \eqref{eq: orig QT} can be described as:
\begin{equation}{\label{eq:fac_QT}}
\begin{aligned}
	& \underset{\U \in \R^{n \times r}}{\text{minimize}}
	& &  \|\linmap(\U\U^\top) - \obs\|_F^2 \quad \quad \text{subject to} \quad  \|\U\|_F^2 \leq 1.
\end{aligned}
\end{equation} 
We compare $\palgo$ with the algorithms described above; as a convex representative implementation, we use the efficient scheme of \cite{yurtsever2015universal}. 
We consider two settings: $\X^\star \in \R^{n \times n}$ is $(i)$ a pure state (\emph{i.e.}, $\text{rank}(\X^\star) = 1$) and, 
$(ii)$ an almost pure state (\emph{i.e.}, $\text{rank}(\X^\star) = r$, for some $r > 1$). 
For all cases, $\trace\left(\X^\star\right) = 1$ and $\obs = \linmap(\X^\star) + \eta$, with $\|\eta\| = 10^{-3}$.
We use Pauli operators for $\linmap$, as described in \cite{liu2011universal}. 
The number of measurements $m$ satisfy $m = C_{\rm sam} \cdot r \cdot n \log(n) $, for various values of $C_{\rm sam}$. 

For all algorithms, we used the correct rank input and trace constraint parameter. 
All methods that require an SVD routine use \texttt{lansvd}$(\cdot)$ from the \texttt{PROPACK} software package. 
Experiments and algorithms are implemented on \textsc{Matlab} environment; we used non-specialized and non-\texttt{mex}ified code parts for all algorithms. 
For initialization, we use the same starting point for all algorithms, which is either specific (Section 8) or random. 
We set the tolerance parameter to $\texttt{tol} := 5\cdot 10^{-6}$. 

\begin{wrapfigure}{r}{0.45\textwidth} 
\begin{minipage}{0.45\textwidth}
   \vspace{0em}
		\centering
		\ra{1.3}
		\begin{scriptsize}
		\rowcolors{2}{white}{black!05!white}
		\begin{tabular}{c c c c c} \toprule
		& \phantom{a} & \multicolumn{3}{c}{Setting: $q = 13$, $C_{\rm sam} = 3$.} \\
		\cmidrule {3-5}
		Algorithm & \phantom{a} & $\tfrac{\|\widehat{\X} - \X^\star\|_F}{\|\X^\star\|_F}$ & \phantom{a}  & Time \\
		\cmidrule{1-1} \cmidrule {3-3} \cmidrule{5-5} 
		\texttt{AccUniPDGrad} & & 7.4151e-02 & & 2354.4552 \\ 
		\texttt{ProjFGD} & & 8.6309e-03 & & 1214.0654 \\ 
		\bottomrule
		\end{tabular}
		\end{scriptsize}
		\caption{\small{Comparison results for reconstruction and efficiency. Time reported is in seconds.}} \label{tbl:small_Comp}
	\end{minipage} \vspace{0em}
\end{wrapfigure}
\emph{Convergence plots.} 
Figure \ref{fig:exp1} (two-leftmost plots) illustrates the iteration and timing complexities of each algorithm under comparison, for a pure state density recovery setting ($r = 1$). 
Here, $q = 12$ which corresponds to a $\tfrac{n(n+1)}{2} = 8,390,656$ dimensional problem; moreover, we assume $C_{\rm sam} = 3$ and thus the number of measurements are $m = 12,288$. 
For initialization, we use the proposed initialization in Section 8 for all algorithms: we compute $-\mathcal{A}^*(y)$, extract factor $U_0$ as the best-$r$ PSD approximation of $-\mathcal{A}^*(y)$, and project $\U_0$ onto $\C$. 

It is apparent that \palgo converges faster to a vicinity of $\X^\star$, compared to the rest of the algorithms; observe also the sublinear rate of \texttt{SparseApproxSDP} in the inner plots, as reported in \cite{hazan2008sparse}.

Figure \ref{tbl:small_Comp} contains recovery error and execution time results for the case $q = 13$ ($n = 8096$); in this case, we solve a $\tfrac{n(n+1)}{2} = 33,558,528$ dimensional problem. 
For this case, \texttt{RSVP} and \texttt{SparseApproxSDP} algorithms were excluded from the comparison. 
Appendix provides extensive results, where similar performance is observed for other values of $q$, $C_{\rm sam}$.

Figure \ref{fig:exp1} (rightmost plot) considers the more general case where $r = 20$ (almost pure state density) and $q = 12$. 
In this case, $m = 245,760$ for $C_{\rm sam} = 3$. 
As $r$ increases, algorithms that utilize an SVD routine spend more CPU time on singular value/vector calculations. 
Certainly, the same applies for matrix-matrix multiplications; however, in the latter case, the complexity scale is milder than that of the SVD calculations. 
Further metadata are also provided in Figure \ref{tbl:small_Comp2}. 

\begin{wrapfigure}{l}{0.55\textwidth} 
\begin{minipage}{0.55\textwidth}
   \vspace{0em}
	\centering
	\ra{1.3}
	\begin{scriptsize}
	\rowcolors{2}{white}{black!05!white}
	\begin{tabular}{c c c c c} \toprule
	& \multicolumn{2}{c}{Setting: $r = 5$.}  & \multicolumn{2}{c}{Setting: $r = 20$.} \\
	\cmidrule {2-3} \cmidrule{4-5}
	Algorithm & $\tfrac{\|\widehat{\X} - \X^\star\|_F}{\|\X^\star\|_F}$ & Time & $\tfrac{\|\widehat{\X} - \X^\star\|_F}{\|\X^\star\|_F}$ & Time \\
	\midrule
	\texttt{RSVP} & 5.15e-02 &  0.78 & 1.71e-02 & 0.38  \\ 
	\texttt{SparseApproxSDP} & 3.17e-02 & 3.74 & 5.49e-02 & 4.38 \\ 
	\texttt{AccUniPDGrad} & 2.01e-02 & 0.36 & 1.54e-02 & 0.33\\ 
	\texttt{ProjFGD} & 1.20e-02 & 0.06 & 7.12e-03 & 0.04 \\ 
	\bottomrule
	\end{tabular}
	\end{scriptsize}
	\caption{\small{Results for reconstruction and efficiency. Time reported is in seconds. For all cases, $C_{\rm sam} = 3$ and $q = 10$.}} \label{tbl:small_Comp2}
\end{minipage}  \vspace{0em}
\end{wrapfigure}
For completeness, in the appendix we also provide results (for the noiseless case) that illustrate the effect of random initialization:
Similar to above, \palgo shows competitive behavior by finding a better solution faster, irrespective of initialization point. 

\emph{Timing evaluation (total and per iteration)}. 
Figure \ref{fig:exp2} highlights the efficiency of our algorithm in terms of time complexity, for various problem configurations.
Our algorithm has fairly low per iteration complexity (where the most expensive operation for this problem is matrix-matrix and matrix-vector multiplications). 
Since our algorithm shows also fast convergence in terms of \# of iterations, this overall results into faster convergence towards a good approximation of $\X^\star$, even as the dimension increases. 
Figure \ref{fig:exp2}(right) shows how the total execution time scales with parameter $r$. 

\begin{figure}[t!]
\centering
\includegraphics[width=0.45\columnwidth]{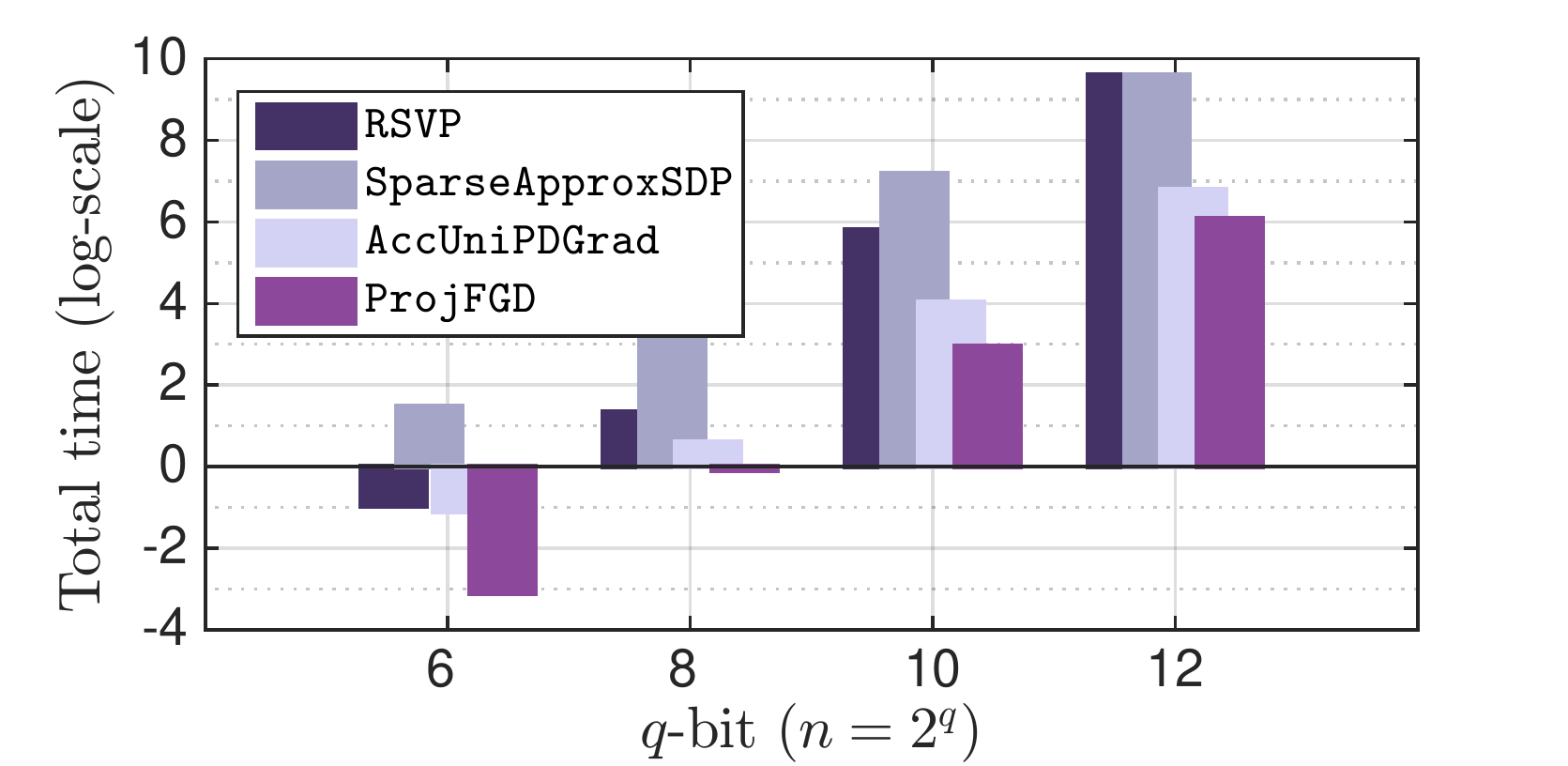} \hspace{-0.8cm}
\includegraphics[width=0.45\columnwidth]{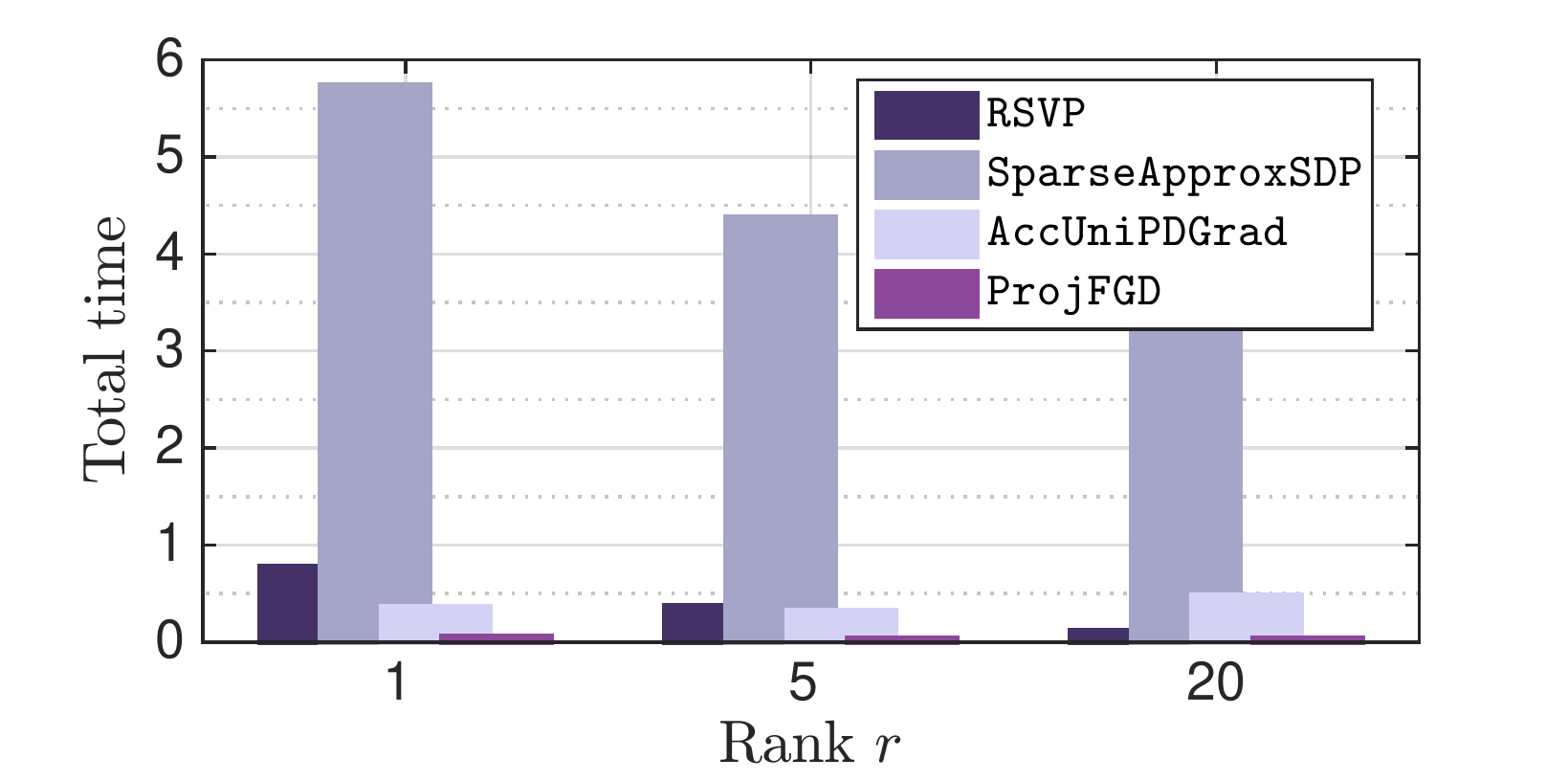}
\caption{\small{Timing bar plot: $y$-axis shows total execution time \textbf{(log-scale)} and $x$-axis corresponds to different $q$ values. Left panel corresponds to $r = 1$ and $C_{\rm sam} = 6$; right panel corresponds to $q = 10$ and $C_{\rm sam} = 6$.}
}
\label{fig:exp2}
\end{figure}


\emph{Overall performance}. 
\palgo shows a competitive performance, as compared to the state-of-the-art algorithms; we would like to emphasize also that projected gradient descent schemes, such as \cite{becker2013randomized}, are also efficient in small- to medium-sized problems, due to their
fast convergence rate. 
Moreover, convex approaches might show better sampling complexity performance (\emph{i.e.}, as $C_{\rm sam}$ decreases).
For more experimental results (under noiseless settings), we defer the reader to Appendix, due to space restrictions.

\subsection{Sparse phase retrieval}{\label{sec:SPT}}
Consider the sparse phase retrieval (SPR) problem \cite{candes2013phaselift, candes2015phase, li2013sparse}: 
we are interested in recovering a (sparse) unknown vector $x^\star \in \mathbb{C}^n$, via its \emph{lifted}, rank-1 representation $\X^\star = x^\star x^{\star H} \in \mathbb{C}^{n \times n}$, from a set of quadratic measurements:
\begin{align*}
y_i = \trace(a_i^H \X a_i) + \eta_i, \quad i = 1, \dots, m.
\end{align*}
Here, $a_i \in \mathbb{C}^n$ are given measurement vectors (often Fourier vectors) and $\eta_i$ is an additive error term.
The above description leads to the following non-convex optimization criterion:
\begin{equation}\label{eq:SPT}
\begin{aligned}
	& \underset{\X \succeq 0}{\text{minimize}}
	& & \|\linmap(\X) - \obs\|_F^2, \\
	& \text{subject to} 
	& & \text{rank}(\X) = 1, \left(\|\X\|_1 \leq \lambda\right).
\end{aligned}
\end{equation} Here, $\linmap: \mathbb{C}^{n \times n} \rightarrow \mathbb{C}^{m}$ such that $(\linmap(\X))_i = \trace(\Phi_i \X)$ where $\Phi_i = a_i a_i^H$. 
In the case where we know $x^\star$ is sparse \cite{li2013sparse, cprl2012Ohlsson}, we can further constrain the lifted variable $\X$ to satisfy $\|\X\|_1 \leq \lambda$, $\lambda > 0$; 
this way we implicitly also restrict the number of non-zeros in its factors and can recover $\Xo$ from a \emph{limited} set of measurements.

\paragraph{Transforming \eqref{eq:SPT} into a factored formulation.} 
Given the rule $\X = uu^H$, where $u \in \mathbb{C}^{n}$, one can consider the factored problem re-formulation:
\begin{equation}{\label{transform2:eq_01}}
\begin{aligned}
	& \underset{u \in \mathbb{C}^n}{\text{{\rm minimize}}}
	& & \|\linmap(uu^H) - \obs\|_F^2 \\
	& \text{{\rm subject to}}
	& &\|u\|_1 \leq \lambda'.
\end{aligned}
\end{equation} 
for some $\lambda' > 0$.

\begin{remark}{\label{rem:transform2_00}}
\textit{In contrast to the QST problem, where there is a continuous map between the constraints in the original $\X$ space and in the factored $\U$ space (\emph{i.e.}, $\trace(\X) \leq \lambda \Leftrightarrow \|\U\|_F^2 \leq \lambda$), this is not true for the SPR problem: 
As we state in the main text, points $\U \in \C$ can result into $\X \not\in \C'$, depending on the selection of $\lambda, \lambda'$ values (i.e., $\C$ is unfaithful). 
In this case, the convergence theorem \ref{thm:projFGD_guarantees} in the $\U$ factor space only proves convergence to a point $\Uo$ in the factored space, which is not necessarily related to the optimal point $\Xo$ in the original space.
However, as we show next, in practice, even in this case \palgo returns a competitive (if not better) solution, compared to state-of-the-art approaches.}
\end{remark} 

\begin{figure*}[t!]
\centering
\includegraphics[width=0.4\textwidth]{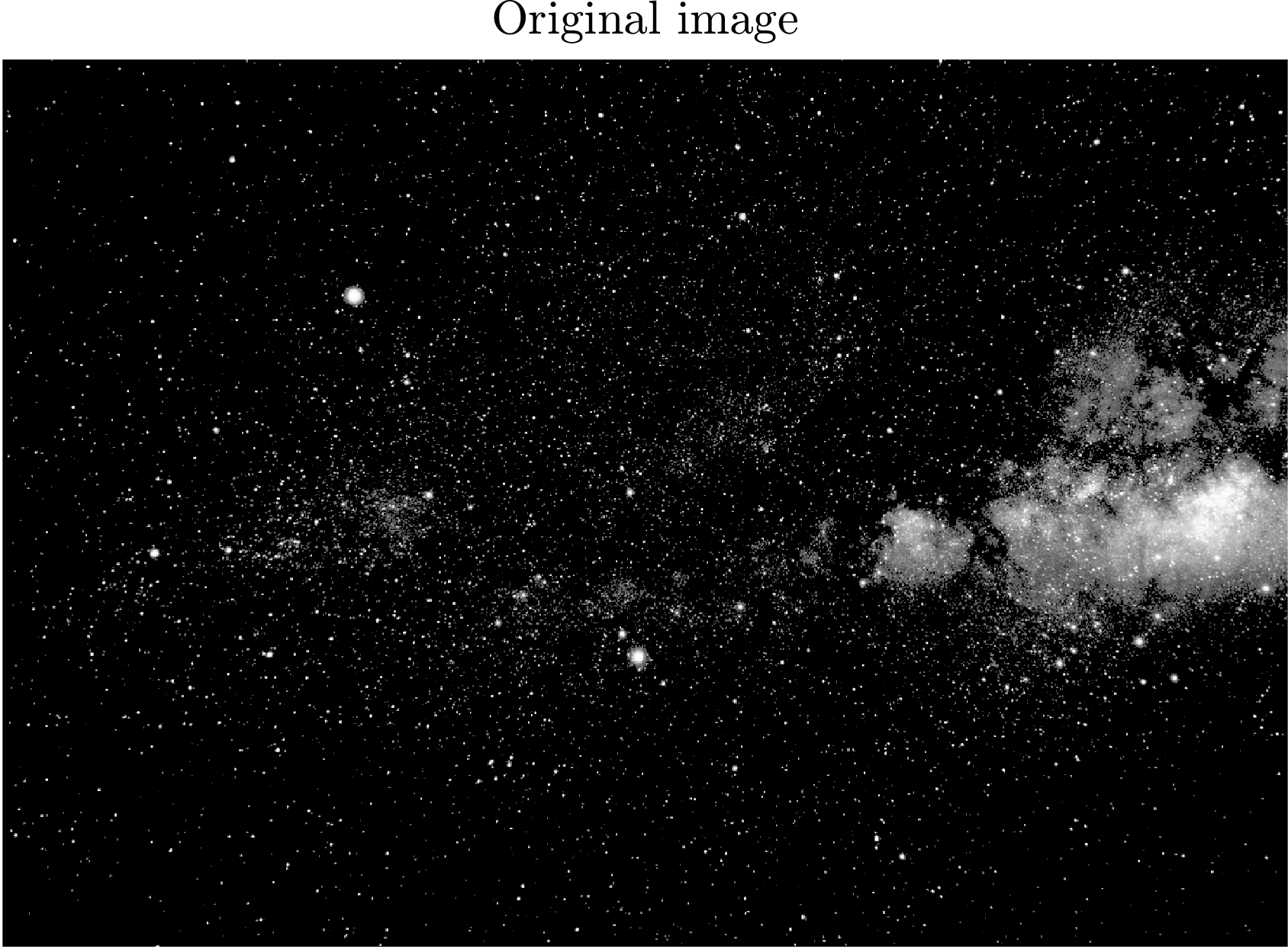} 
\includegraphics[width=0.4\textwidth]{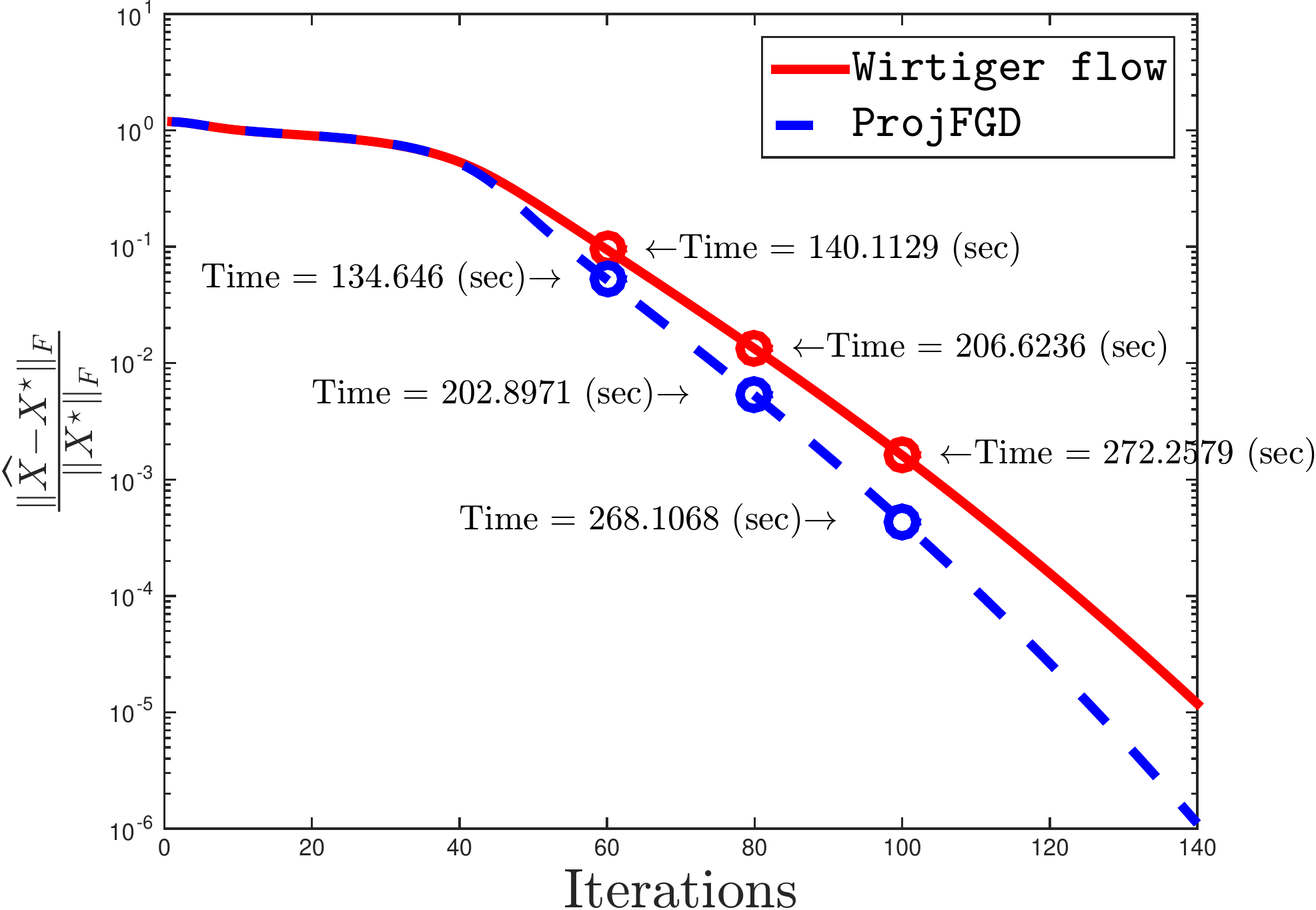} 
\caption{Left panel: Original image for the sparse phase retrieval problem. The dimension of the image is $883 \times 1280 = 1,130,240$ variables. Right panel: Convergence performance of algorithms under comparison w.r.t. $\tfrac{\|\widehat{X} - \X^\star\|_F}{\|\X^\star\|_F}$ vs. the total number of iterations; markers on top of curves indicate the total execution time until that point.
}
\label{fig:app_phase_exp1}
\end{figure*}

\paragraph{State-of-the-art approaches.} 
One of the most widely used methods for the phase retrieval problem comes from the seminal work of Gerchberg-Saxton \cite{gerchberg1972practical} and Fienup \cite{fienup1987phase}: 
they propose a greedy scheme that alternates projections on the range of $\{a_i\}_{i=1}^m$ and on the non-convex set of vectors $b$ such that $b = |Ax|$. 
Main disadvantage of such greedy methods is that often they get stuck to locally minimum points.

An popularized alternative to these greedy methods is via \emph{semidefinite relaxations}. 
\cite{candes2015phase} proposes \texttt{PhaseLift}, where the rank constraint is replaced by the nuclear norm surrogate. 
However, it is well-known that such SDP relaxations can be computationally prohibitive, when solved using off-the-self software packages, even for small problem instances; 
some specialized and more efficient convex relaxation algorithms are given in \cite{fogel2013phase}.

In \cite{candes2015phase2}, the authors present \texttt{Wirtinger Flow} algorithm, a non-convex scheme for solving phase retrieval problems. 
Similar to our approach, \texttt{Wirtinger Flow} consists of three components: 
$(i)$ a careful initialization step using a spectral method, 
$(ii)$ a specialized step size selection and, 
$(iii)$ a recursion where gradient steps on the factored variable space are performed. 
Other approaches include Approximate Message Passing algorithms \cite{schniter2015compressive} and ADMM approaches \cite{cprl2012Ohlsson}.

\paragraph{Numerical results.} 
We test our algorithm on image recovery, according to the description given in \cite[Section 4.2]{candes2015phase2}.
Here, we consider grayscale images that are by nature also sparse (Figure \ref{fig:app_phase_exp1} - left panel). 
This way, we can also consider $\ell_1$-norm constraints, as in the criterion \eqref{transform2:eq_01}. 
We generate $L = 21$ random octanary patterns and, using these $21$ samples, we obtain the coded diffraction patterns using the grayscale image as input.
As dictated by \cite[Section 4.2]{candes2015phase2}, we perform 50 power method iterations for initialization.

For this experiment, we highlight $(i)$ how our algorithm \palgo performs in practice, and $(ii)$ how the additional sparsity constraint could lead to better performance. 
Figure \ref{fig:app_phase_exp1} (right panel) depicts the relative error $\tfrac{\|\widehat{\X} - \Xo\|_F}{\|\Xo\|_F}$ w.r.t. the iteration count for two algorithms: $(i)$ \texttt{Wirtiger flow} \cite{candes2015phase2}, and $(ii)$ \palgo. 
We observe that \palgo shows a slightly better performance, compared to \texttt{Wirtiger flow}, both in terms of iterations --\emph{i.e.}, we reach to a better solution within the same number of iterations-- and in terms of execution time --\emph{i.e.}, given a time wall, \palgo returns an estimate of better quality within the same amount of time.
We note that both algorithms used step sizes that were slightly different in values, while \palgo further performs also a projection step.\footnote{The step size in \texttt{Wirtiger flow} satisfies $\eta := \tfrac{\mu_t}{\|\U_0\|_F^2}$, for $\mu_t = \min\left\{1 - e^{t/t_0},~0.4 \right\}$ and $t_0 \approx 330$.}
Figure \ref{fig:app_phase_exp2} shows some reconstructed images returned by the algorithms under comparison, during their execution.
In all cases, both algorithms perform appealingly, finding a good approximation of the original image in less than 5 minutes; comparing the two algorithms, we note that \palgo returns a solution, within the same number of iterations, with at least 5 dB higher Peak Signal to Noise Ration (PSNR), in less time.


\begin{figure}[t!]
\centering
\includegraphics[width=0.24\textwidth]{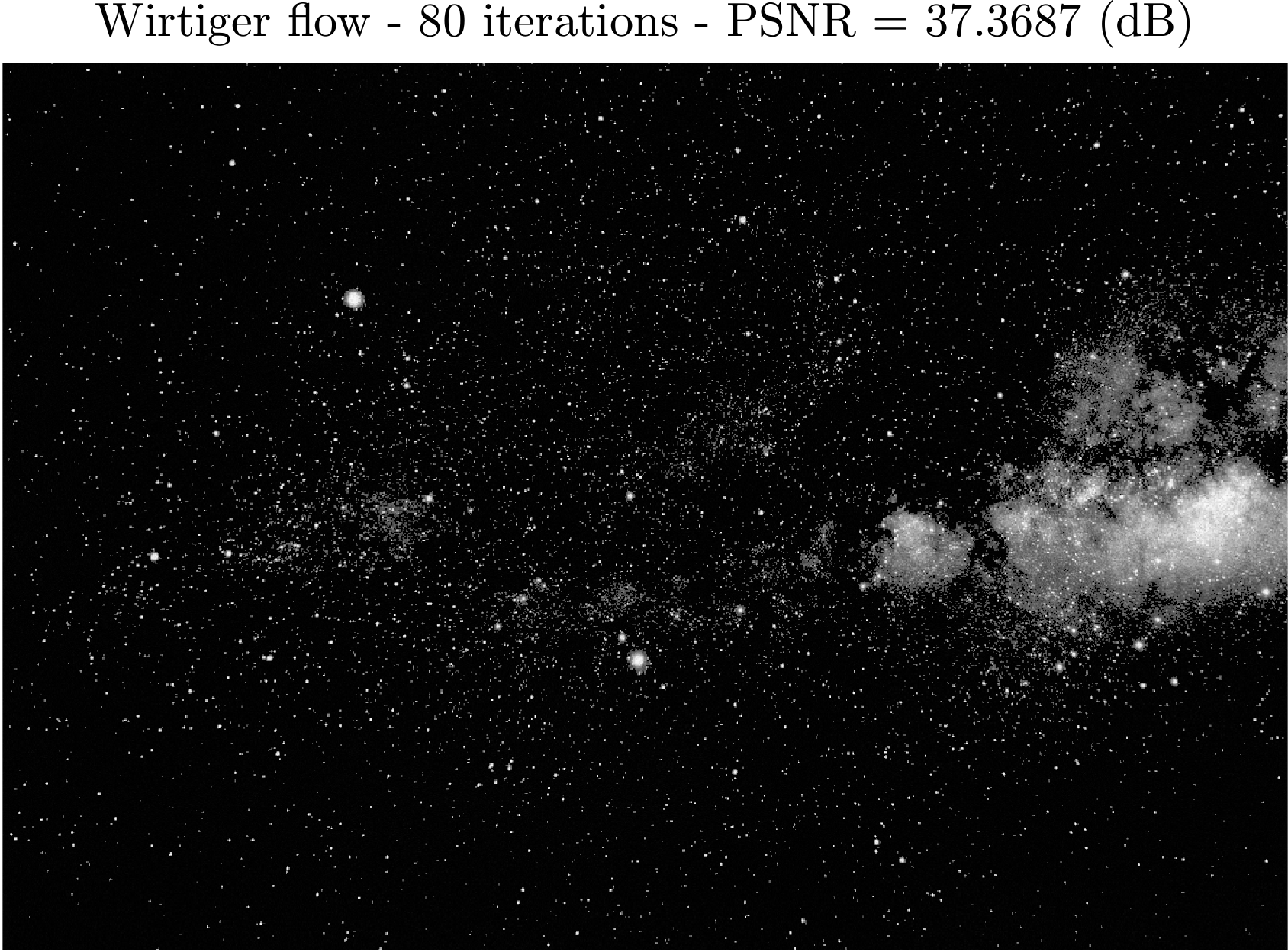}  \includegraphics[width=0.24\textwidth]{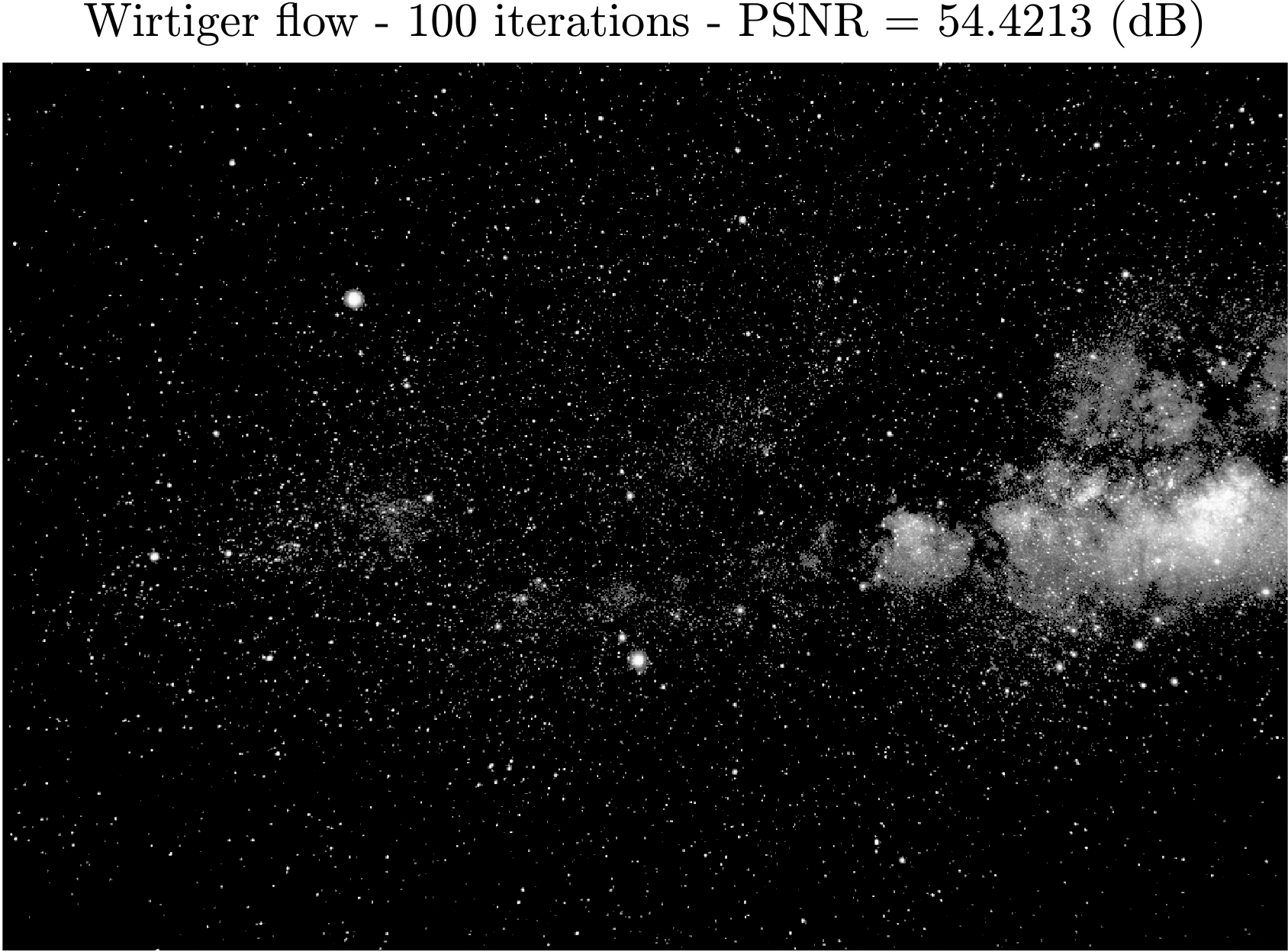} 
\includegraphics[width=0.24\textwidth]{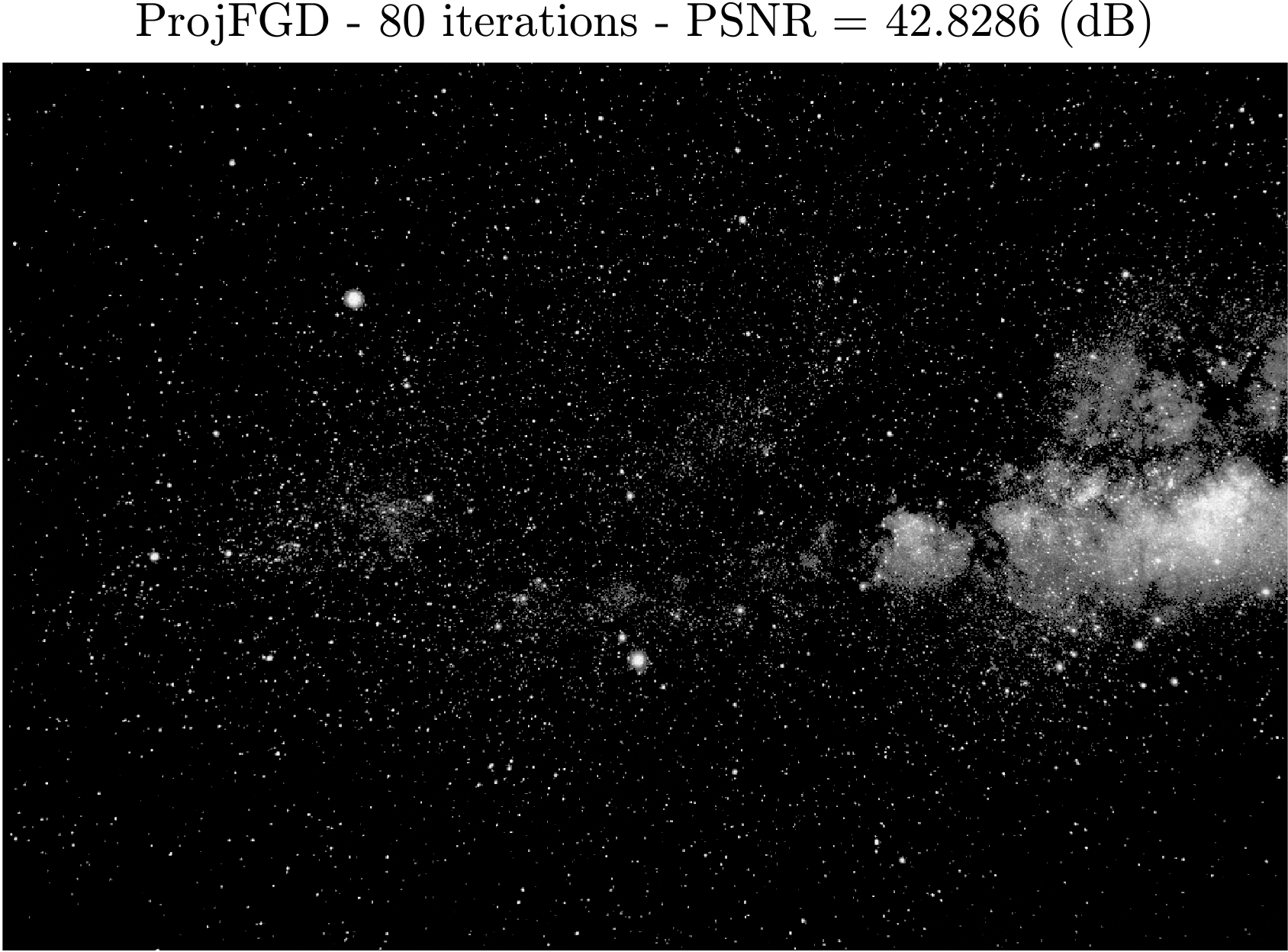}  \includegraphics[width=0.24\textwidth]{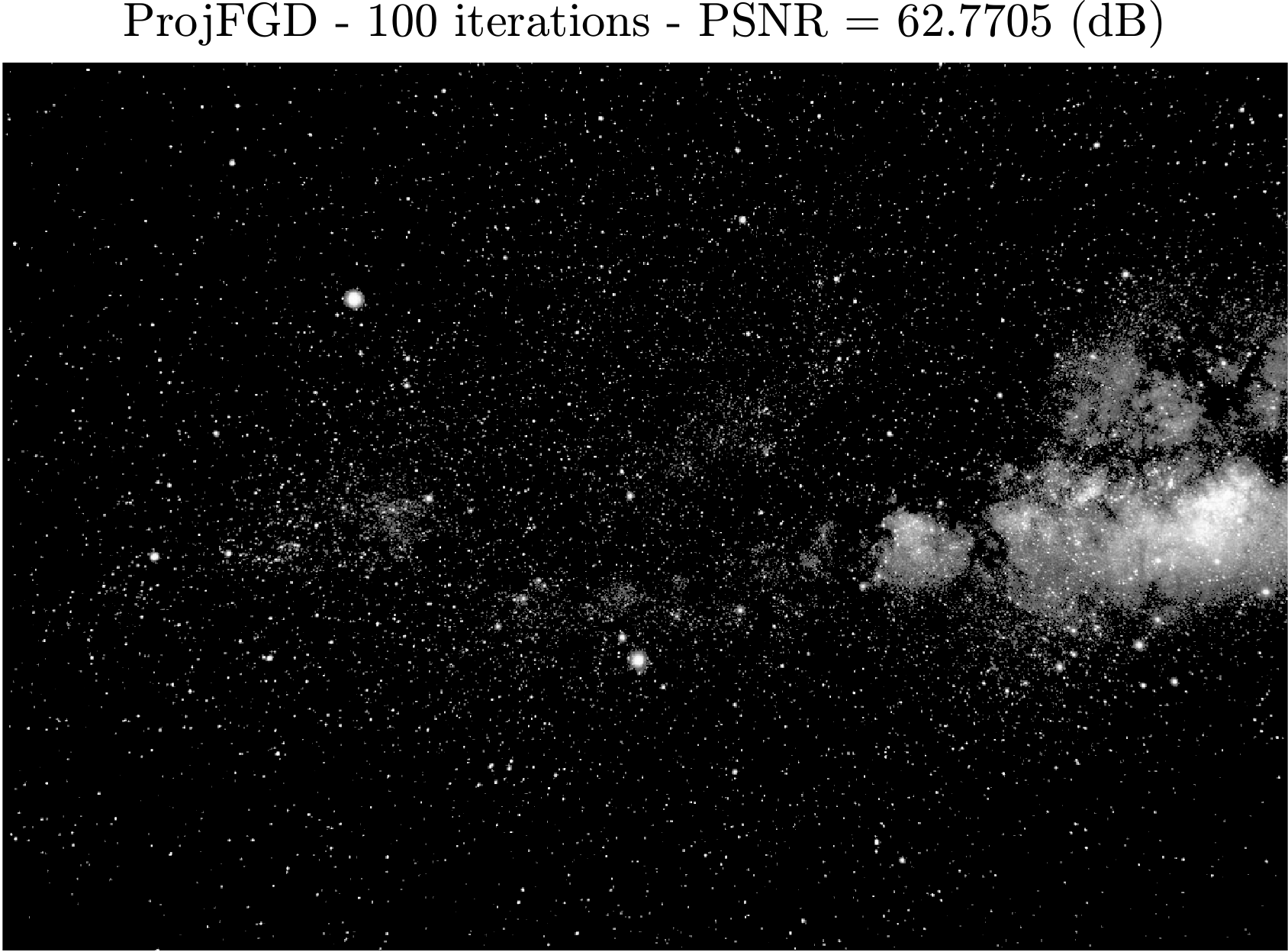}
\caption{First two figures: Reconstructed image by using \texttt{Wirtiger flow} algorithm after 80 iterations (left panel) and 100 iterations (right panel). 
Last two figures: Reconstructed image by using \palgo algorithm after 80 iterations (left panel) and 100 iterations (right panel). 
}
\label{fig:app_phase_exp2}
\end{figure}

\section{Discussion}{\label{sec:discussion}}
We consider a class of low-rank matrix problems where the solution is assumed PSD and further constrained with some matrix constraints, described in the text. 
This paper proposes \palgo, a non-convex projected gradient descent algorithm that operates on the factors of the PSD putative solution. 
When the objective function is smooth and strongly convex in the original variable space, \palgo has (local) linear rate convergence guarantees (which can become global, given a proper initialization).

Our current analysis restricts to the Assumption \ref{ass:00}; extending the proof for more complex constraints sets is one possible research direction for future work, where an analogous of \emph{gradient mapping} \cite{nesterov2013introductory} might be required. 
Furthermore, considering barrier functions in the objective function, in order to accommodate the constraints, could be a possible extension. 
We hope this work will trigger future attempts along these directions.

\begin{small}
\bibliographystyle{plain}
\bibliography{ProjFGD}

\begin{thebibliography}{10}

\bibitem{aaronson2007learnability}
S.~Aaronson.
\newblock The learnability of quantum states.
\newblock In {\em Proceedings of the Royal Society of London A: Mathematical,
  Physical and Engineering Sciences}, volume 463, pages 3089--3114, 2007.

\bibitem{agarwal2010fast}
A.~Agarwal, S.~Negahban, and M.~Wainwright.
\newblock Fast global convergence rates of gradient methods for
  high-dimensional statistical recovery.
\newblock In {\em Advances in NIPS}, pages 37--45, 2010.

\bibitem{balakrishnan2014statistical}
S.~Balakrishnan, M.~Wainwright, and B.~Yu.
\newblock Statistical guarantees for the {EM} algorithm: {F}rom population to
  sample-based analysis.
\newblock {\em arXiv preprint arXiv:1408.2156}, 2014.

\bibitem{balzano2010online}
L.~Balzano, R.~Nowak, and B.~Recht.
\newblock Online identification and tracking of subspaces from highly
  incomplete information.
\newblock In {\em Communication, Control, and Computing (Allerton), 2010 48th
  Annual Allerton Conference on}, pages 704--711. IEEE, 2010.

\bibitem{becker2011nesta}
S.~Becker, J.~Bobin, and E.~Cand{\`e}s.
\newblock {NESTA}: {A} fast and accurate first-order method for sparse
  recovery.
\newblock {\em SIAM Journal on Imaging Sciences}, 4(1):1--39, 2011.

\bibitem{becker2011templates}
S.~Becker, E.~Cand{\`e}s, and M.~Grant.
\newblock Templates for convex cone problems with applications to sparse signal
  recovery.
\newblock {\em Mathematical Programming Computation}, 3(3):165--218, 2011.

\bibitem{becker2013randomized}
S.~Becker, V.~Cevher, and A.~Kyrillidis.
\newblock Randomized low-memory singular value projection.
\newblock In {\em 10th International Conference on Sampling Theory and
  Applications (Sampta)}, 2013.

\bibitem{bhojanapalli2015dropping}
S.~Bhojanapalli, A.~Kyrillidis, and S.~Sanghavi.
\newblock Dropping convexity for faster semi-definite optimization.
\newblock In {\em 29th Annual Conference on Learning Theory}, pages 530--582,
  2016.

\bibitem{boumal2011rtrmc}
N.~Boumal and P.-A. Absil.
\newblock {RTRMC}: {A} {R}iemannian trust-region method for low-rank matrix
  completion.
\newblock In {\em Advances in neural information processing systems}, pages
  406--414, 2011.

\bibitem{bubeck2014theory}
S.~Bubeck.
\newblock Theory of convex optimization for machine learning.
\newblock {\em arXiv preprint arXiv:1405.4980}, 2014.

\bibitem{burer2003nonlinear}
S.~Burer and R.~Monteiro.
\newblock A nonlinear programming algorithm for solving semidefinite programs
  via low-rank factorization.
\newblock {\em Mathematical Programming}, 95(2):329--357, 2003.

\bibitem{burer2005local}
S.~Burer and R.~Monteiro.
\newblock Local minima and convergence in low-rank semidefinite programming.
\newblock {\em Mathematical Programming}, 103(3):427--444, 2005.

\bibitem{cai2010singular}
J.~Cai, E.~Cand{\`e}s, and Z.~Shen.
\newblock A singular value thresholding algorithm for matrix completion.
\newblock {\em SIAM Journal on Optimization}, 20(4):1956--1982, 2010.

\bibitem{candes2015phase}
E.~Candes, Y.~Eldar, T.~Strohmer, and V.~Voroninski.
\newblock Phase retrieval via matrix completion.
\newblock {\em SIAM Review}, 57(2):225--251, 2015.

\bibitem{candes2015phase2}
E.~Candes, X.~Li, and M.~Soltanolkotabi.
\newblock Phase retrieval via {W}irtinger flow: {T}heory and algorithms.
\newblock {\em Information Theory, IEEE Transactions on}, 61(4):1985--2007,
  2015.

\bibitem{candes2013phaselift}
E.~Candes, T.~Strohmer, and V.~Voroninski.
\newblock Phaselift: {E}xact and stable signal recovery from magnitude
  measurements via convex programming.
\newblock {\em Communications on Pure and Applied Mathematics},
  66(8):1241--1274, 2013.

\bibitem{chen2014coherent}
Y.~Chen, S.~Bhojanapalli, S.~Sanghavi, and R.~Ward.
\newblock Coherent matrix completion.
\newblock In {\em Proceedings of The 31st International Conference on Machine
  Learning}, pages 674--682, 2014.

\bibitem{chen2015fast}
Y.~Chen and M.~Wainwright.
\newblock Fast low-rank estimation by projected gradient descent: {G}eneral
  statistical and algorithmic guarantees.
\newblock {\em arXiv preprint arXiv:1509.03025}, 2015.

\bibitem{fienup1987phase}
C.~Fienup and J.~Dainty.
\newblock Phase retrieval and image reconstruction for astronomy.
\newblock {\em Image Recovery: Theory and Application}, pages 231--275, 1987.

\bibitem{flammia2012quantum}
S.~Flammia, D.~Gross, Y.-K. Liu, and J.~Eisert.
\newblock Quantum tomography via compressed sensing: {E}rror bounds, sample
  complexity and efficient estimators.
\newblock {\em New Journal of Physics}, 14(9):095022, 2012.

\bibitem{fogel2013phase}
F.~Fogel, I.~Waldspurger, and A.~d'Aspremont.
\newblock Phase retrieval for imaging problems.
\newblock {\em arXiv preprint arXiv:1304.7735}, 2013.

\bibitem{gerchberg1972practical}
R.~Gerchberg.
\newblock A practical algorithm for the determination of phase from image and
  diffraction plane pictures.
\newblock {\em Optik}, 35:237, 1972.

\bibitem{gross2010quantum}
D.~Gross, Y.-K. Liu, S.~Flammia, S.~Becker, and J.~Eisert.
\newblock Quantum state tomography via compressed sensing.
\newblock {\em Physical review letters}, 105(15):150401, 2010.

\bibitem{hardt2014fast}
M.~Hardt and M.~Wootters.
\newblock Fast matrix completion without the condition number.
\newblock In {\em Proceedings of The 27th Conference on Learning Theory}, pages
  638--678, 2014.

\bibitem{hazan2008sparse}
E.~Hazan.
\newblock Sparse approximate solutions to semidefinite programs.
\newblock In {\em LATIN 2008: Theoretical Informatics}, pages 306--316.
  Springer, 2008.

\bibitem{jaganathan2013sparse}
K.~Jaganathan, S.~Oymak, and B.~Hassibi.
\newblock Sparse phase retrieval: {C}onvex algorithms and limitations.
\newblock In {\em Information Theory Proceedings (ISIT)}, pages 1022--1026.
  IEEE, 2013.

\bibitem{jain2010guaranteed}
P.~Jain, R.~Meka, and I.~Dhillon.
\newblock Guaranteed rank minimization via singular value projection.
\newblock In {\em Advances in Neural Information Processing Systems}, pages
  937--945, 2010.

\bibitem{jain2013low}
P.~Jain, P.~Netrapalli, and S.~Sanghavi.
\newblock Low-rank matrix completion using alternating minimization.
\newblock In {\em Proceedings of the 45th annual ACM symposium on Symposium on
  theory of computing}, pages 665--674. ACM, 2013.

\bibitem{kalev2015quantum}
A.~Kalev, R.~Kosut, and I.~Deutsch.
\newblock Quantum tomography protocols with positivity are compressed sensing
  protocols.
\newblock {\em Nature partner journals (npj) Quantum Information}, 1:15018,
  2015.

\bibitem{kyrillidis2014matrix}
A.~Kyrillidis and V.~Cevher.
\newblock Matrix recipes for hard thresholding methods.
\newblock {\em Journal of mathematical imaging and vision}, 48(2):235--265,
  2014.

\bibitem{laue2012hybrid}
S.~Laue.
\newblock A hybrid algorithm for convex semidefinite optimization.
\newblock In {\em Proceedings of the 29th International Conference on Machine
  Learning (ICML-12)}, pages 177--184, 2012.

\bibitem{lee2010admira}
K.~Lee and Y.~Bresler.
\newblock {ADMiRA}: {A}tomic decomposition for minimum rank approximation.
\newblock {\em Information Theory, IEEE Transactions on}, 56(9):4402--4416,
  2010.

\bibitem{li2013sparse}
X.~Li and V.~Voroninski.
\newblock Sparse signal recovery from quadratic measurements via convex
  programming.
\newblock {\em SIAM Journal on Mathematical Analysis}, 45(5):3019--3033, 2013.

\bibitem{lin2010augmented}
Z.~Lin, M.~Chen, and Y.~Ma.
\newblock The augmented {L}agrange multiplier method for exact recovery of
  corrupted low-rank matrices.
\newblock {\em arXiv preprint arXiv:1009.5055}, 2010.

\bibitem{liu2011universal}
Y.-K. Liu.
\newblock Universal low-rank matrix recovery from {P}auli measurements.
\newblock In {\em Advances in Neural Information Processing Systems}, pages
  1638--1646, 2011.

\bibitem{mirsky1975trace}
L.~Mirsky.
\newblock A trace inequality of {J}ohn von {N}eumann.
\newblock {\em Monatshefte f{\"u}r Mathematik}, 79(4):303--306, 1975.

\bibitem{nesterov2013introductory}
Y.~Nesterov.
\newblock {\em Introductory lectures on convex optimization: {A} basic course},
  volume~87.
\newblock Springer Science \& Business Media, 2013.

\bibitem{Ohlsson2012CPRL}
H.~Ohlsson, A.~Yang, R.~Dong, and S.~Sastry.
\newblock {CPRL} -- an extension of compressive sensing to the phase retrieval
  problem.
\newblock In {\em Advances in Neural Information Processing Systems}, pages
  1376--1384, 2012.

\bibitem{cprl2012Ohlsson}
H~Ohlsson, A.~Yang, R.~Dong, and S.~Sastry.
\newblock {CPRL} -- an extension of compressive sensing to the phase retrieval
  problem.
\newblock In {\em Advances in Neural Information Processing Systems 25}, pages
  1376--1384, 2012.

\bibitem{recht2010guaranteed}
B.~Recht, M.~Fazel, and P.~Parrilo.
\newblock Guaranteed minimum-rank solutions of linear matrix equations via
  nuclear norm minimization.
\newblock {\em SIAM review}, 52(3):471--501, 2010.

\bibitem{schniter2015compressive}
P.~Schniter and S.~Rangan.
\newblock Compressive phase retrieval via generalized approximate message
  passing.
\newblock {\em Signal Processing, IEEE Transactions on}, 63(4):1043--1055,
  2015.

\bibitem{shechtman2014gespar}
Y.~Shechtman, A.~Beck, and Y.~Eldar.
\newblock {GESPAR}: {E}fficient phase retrieval of sparse signals.
\newblock {\em Signal Processing, IEEE Transactions on}, 62(4):928--938, 2014.

\bibitem{tanner2013normalized}
J.~Tanner and K.~Wei.
\newblock Normalized iterative hard thresholding for matrix completion.
\newblock {\em SIAM Journal on Scientific Computing}, 35(5):S104--S125, 2013.

\bibitem{tu2015low}
S.~Tu, R.~Boczar, M.~Soltanolkotabi, and B.~Recht.
\newblock Low-rank solutions of linear matrix equations via {P}rocrustes flow.
\newblock {\em arXiv preprint arXiv:1507.03566}, 2015.

\bibitem{wen2012solving}
Z.~Wen, W.~Yin, and Y.~Zhang.
\newblock Solving a low-rank factorization model for matrix completion by a
  nonlinear successive over-relaxation algorithm.
\newblock {\em Mathematical Programming Computation}, 4(4):333--361, 2012.

\bibitem{yurtsever2015universal}
A.~Yurtsever, Q.~Tran-Dinh, and V.~Cevher.
\newblock A universal primal-dual convex optimization framework.
\newblock In {\em Advances in Neural Information Processing Systems 28}, pages
  3132--3140. 2015.

\bibitem{zhang2015global}
D.~Zhang and L.~Balzano.
\newblock Global convergence of a grassmannian gradient descent algorithm for
  subspace estimation.
\newblock {\em arXiv preprint arXiv:1506.07405}, 2015.

\bibitem{zhao2015nonconvex}
T.~Zhao, Z.~Wang, and H.~Liu.
\newblock A nonconvex optimization framework for low rank matrix estimation.
\newblock In {\em Advances in Neural Information Processing Systems 28}, pages
  559--567. 2015.

\bibitem{zheng2015convergent}
Q.~Zheng and J.~Lafferty.
\newblock A convergent gradient descent algorithm for rank minimization and
  {SDP} from random linear measurements.
\newblock In {\em Advances in NIPS}, pages 109--117, 2015.

\end{thebibliography}
\end{small}


%
%
%

\section*{Additional experiments}{\label{sec:add_exp}}

\subsection{Quantum state tomography -- more results}
Figures \ref{fig:app_exp3}-\ref{fig:app_exp4} show further results regarding the QST problem, where $r = 1$ and $q = 10, 12$, respectively. 
For each case, we present both the performance in terms of number of iterations needed, as well as what is the cumulative time required.
For all algorithms, we use as initial point $U_0 = \Pi_{\C}(\widetilde{\U_0})$ such that $\X_0 = \widetilde{\U}_0 \widetilde{\U}_0^\top$ where $\X_0 = \Pi_{+}\left(-\mathcal{A}^*(y)\right)$ and $\Pi_{+}(\cdot)$ is the projection onto the PSD cone.
Configurations are described in the caption of each figure. 
Table \ref{tbl:Comp} contains information regarding total time required for convergence and quality of solution for all these cases.
Results on almost pure density states, \emph{i.e.}, $r > 1$, are provided in Figure \ref{fig:app_exp5}. 

For completeness, we also provide results that illustrate the effect of initialization. 
In this case, we consider a random initialization and the same initial point is used for all algorithms. 
Some results are illustrated in Figure \ref{fig:app_exp6}; table \ref{tbl:Comp3} contains metadata of these experiments. 
Similar to above, \palgo shows competitive behavior by finding a better solution faster, irrespective of initialization point. 


\begin{figure*}[t!]
\centering
\includegraphics[width=0.32\textwidth]{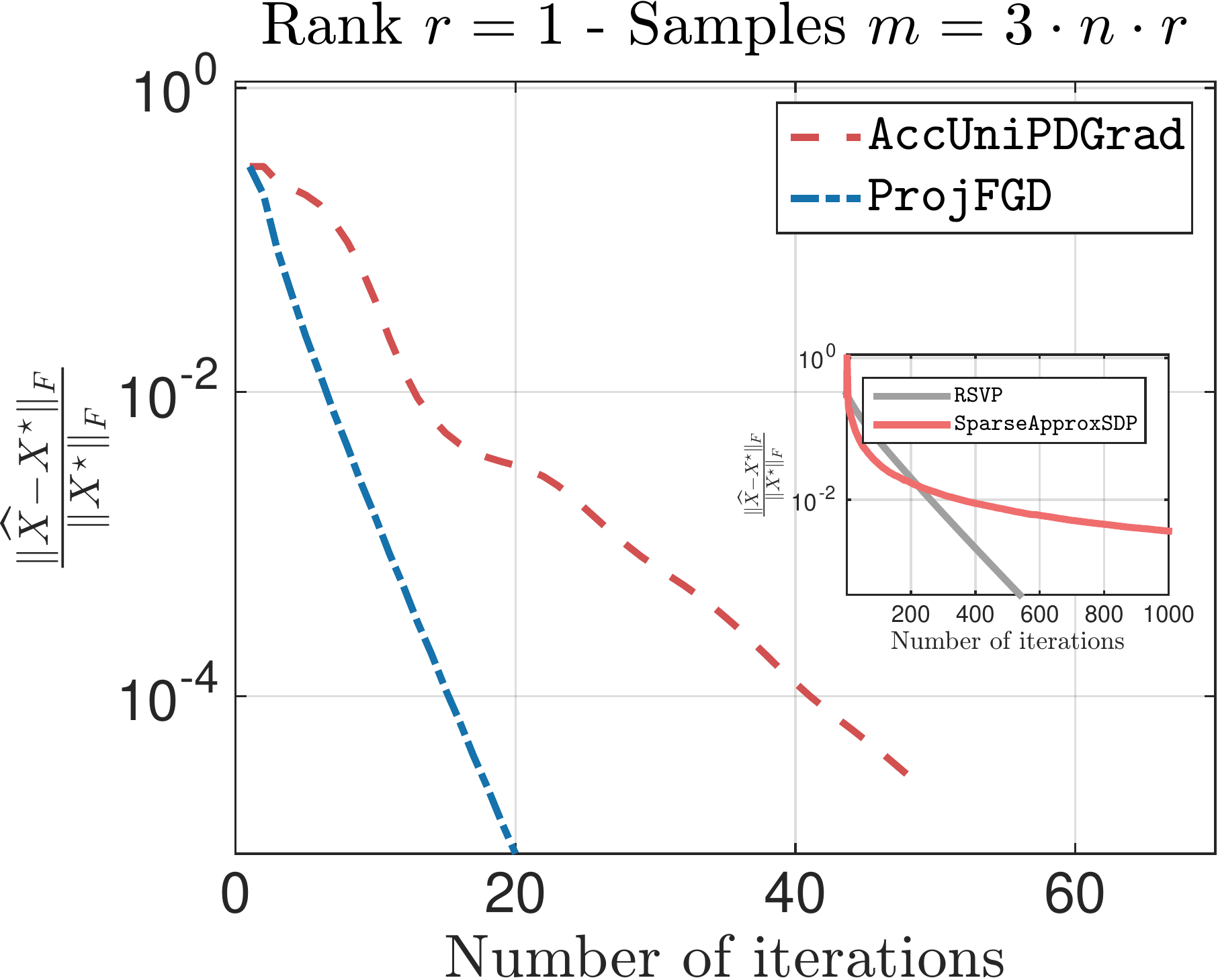} 
\includegraphics[width=0.32\textwidth]{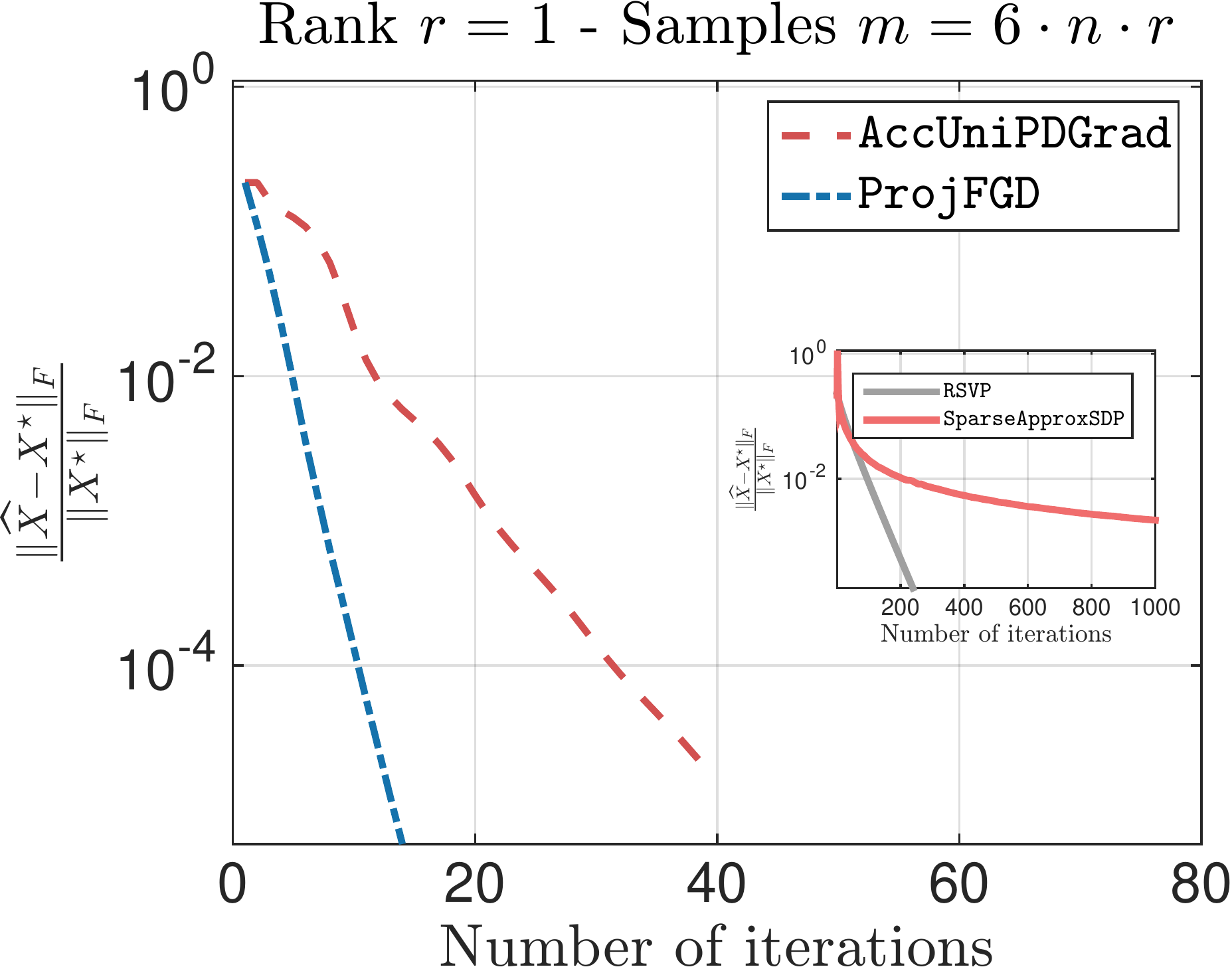} 
\includegraphics[width=0.32\textwidth]{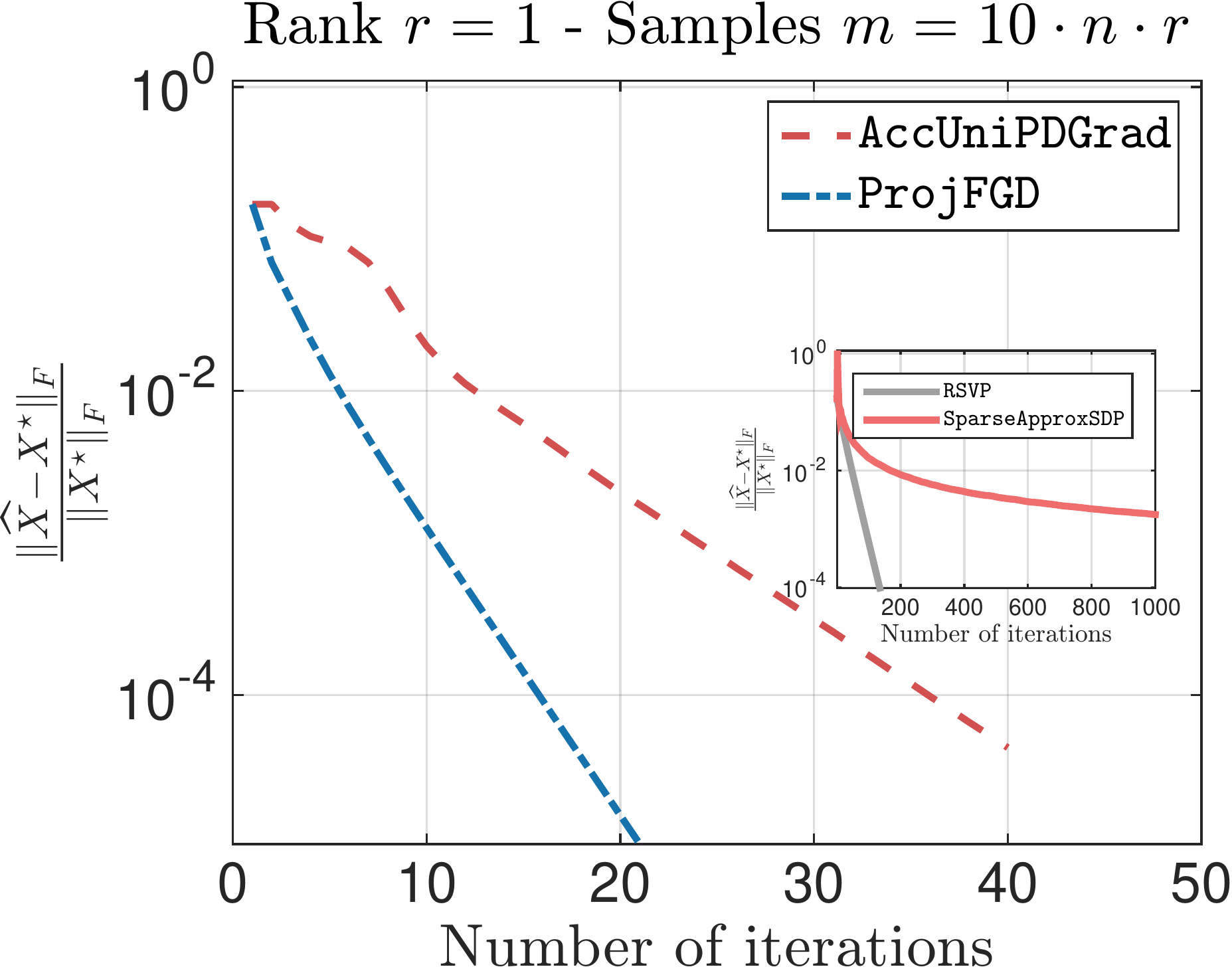} \\
\includegraphics[width=0.32\textwidth]{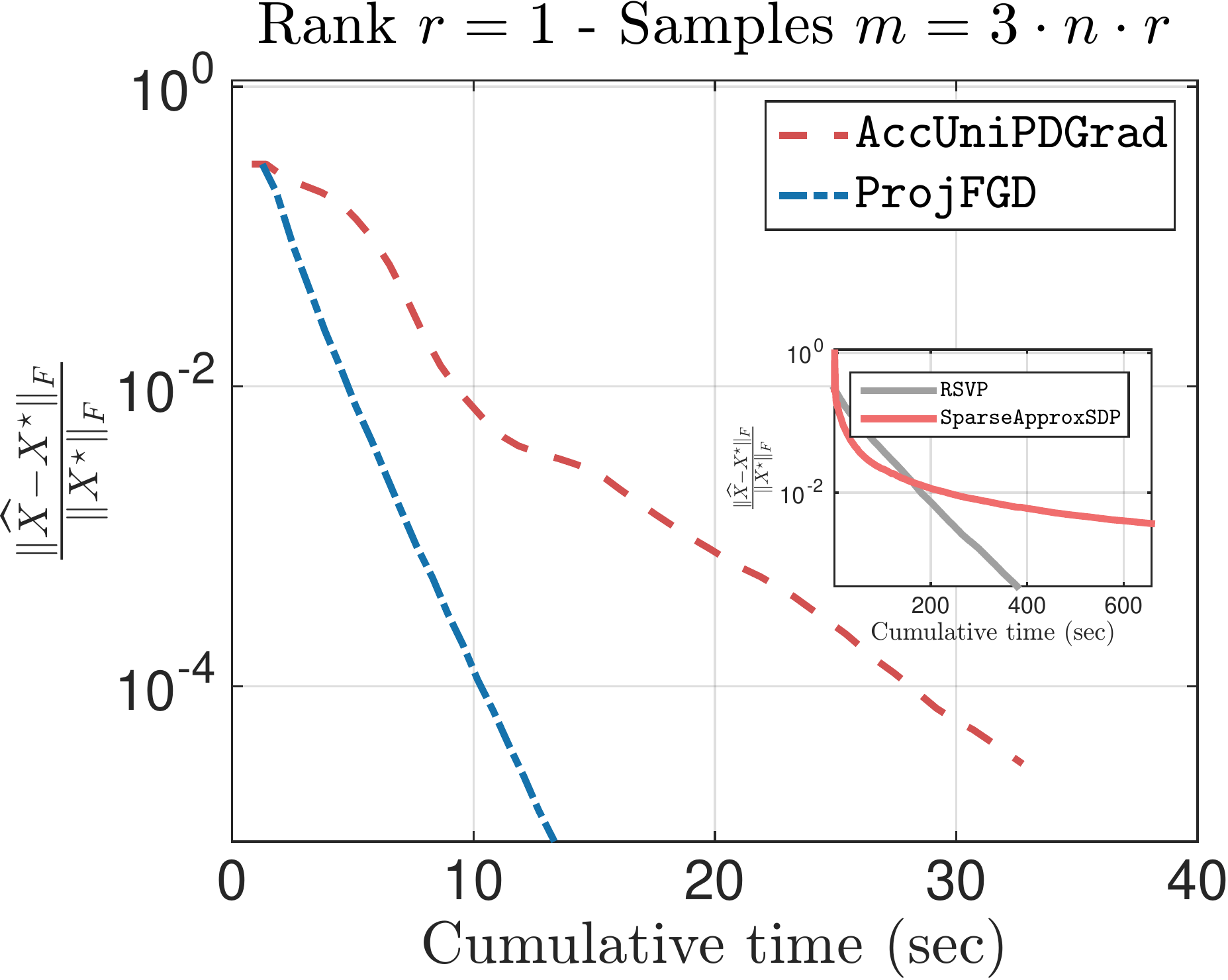} 
\includegraphics[width=0.32\textwidth]{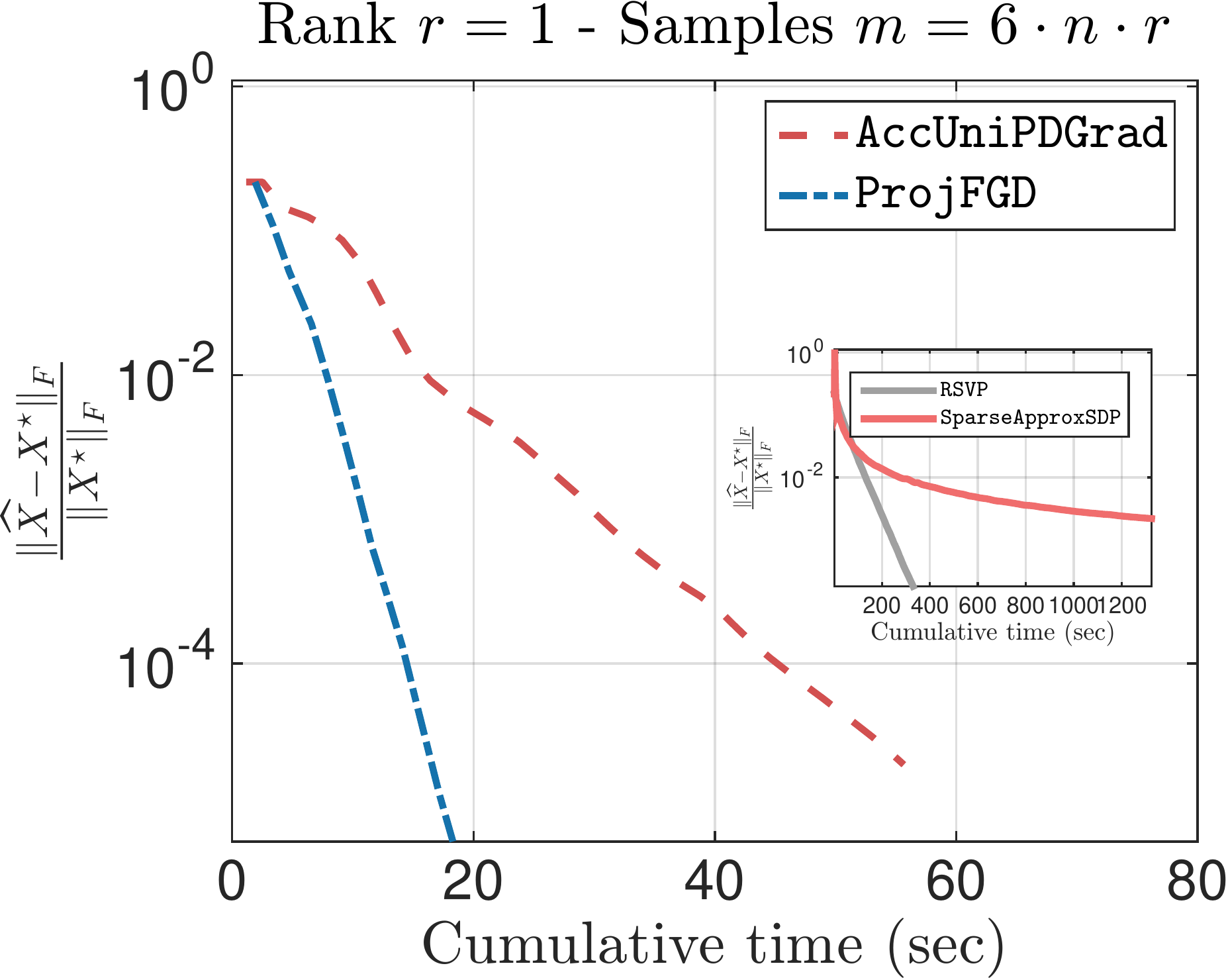} 
\includegraphics[width=0.32\textwidth]{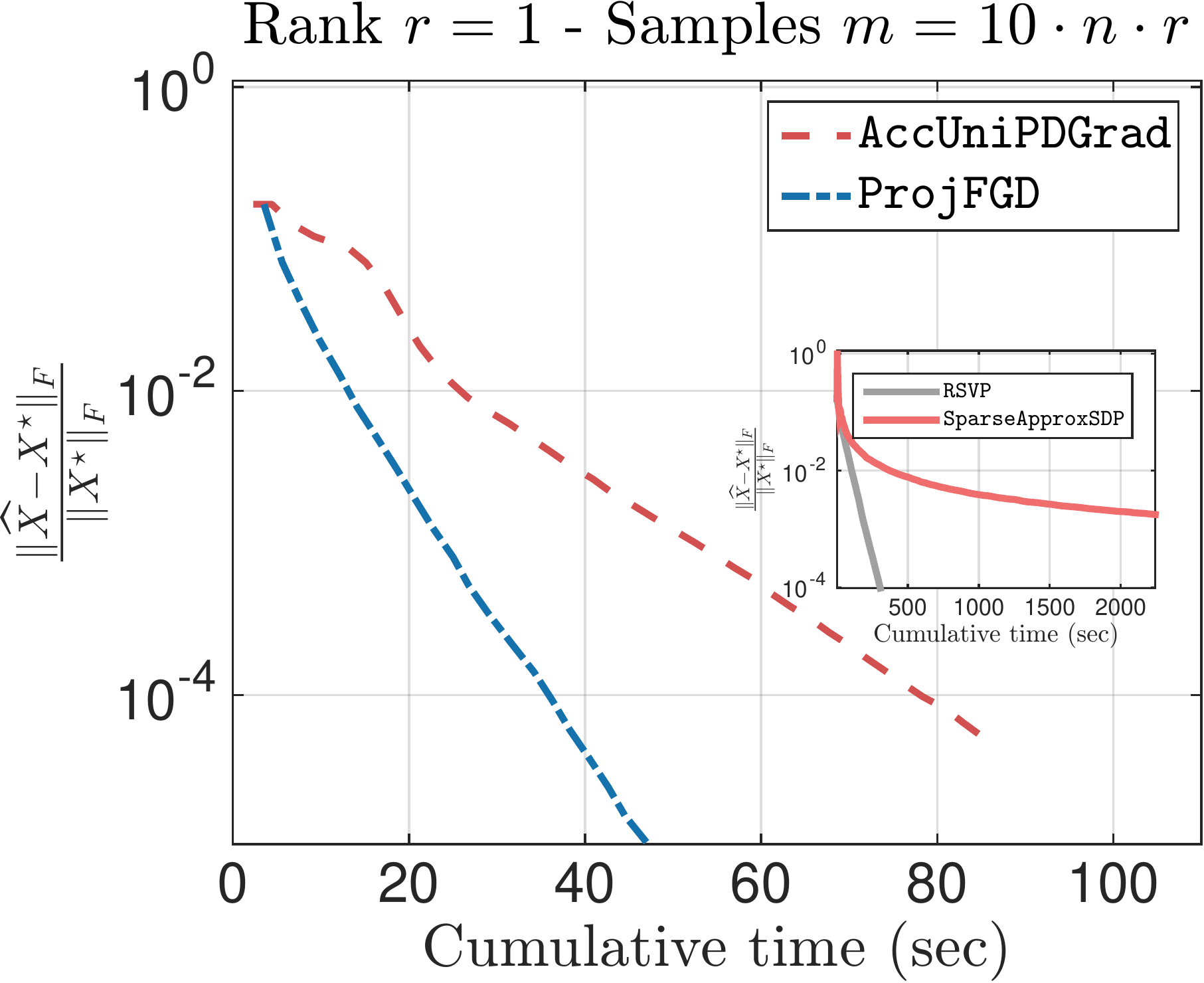} \\
\caption{\textbf{Quantum state tomography:} Convergence performance of algorithms under comparison w.r.t. $\tfrac{\|\widehat{X} - \X^\star\|_F}{\|\X^\star\|_F}$ vs. $(i)$ the total number of iterations (top) and $(ii)$ the total execution time (bottom). First, second and third column corresponds to $C_{\rm sam} = 3, 6$ and $10$, respectively. For all cases, $r = 1$ (pure state setting) and $q = 10$. Initial point is $U_0 = \Pi_{\C}(\widetilde{\U_0})$ such that $\X_0 = \widetilde{\U}_0 \widetilde{\U}_0^\top$ where $\X_0 = \Pi_{+}\left(-\mathcal{A}^*(y)\right)$.
}
\label{fig:app_exp3}
\end{figure*}

\begin{figure*}[t!]
\centering
\includegraphics[width=0.32\textwidth]{./figs/pureQST_q_12_C_3_init_ours_XcurvesIter} 
\includegraphics[width=0.32\textwidth]{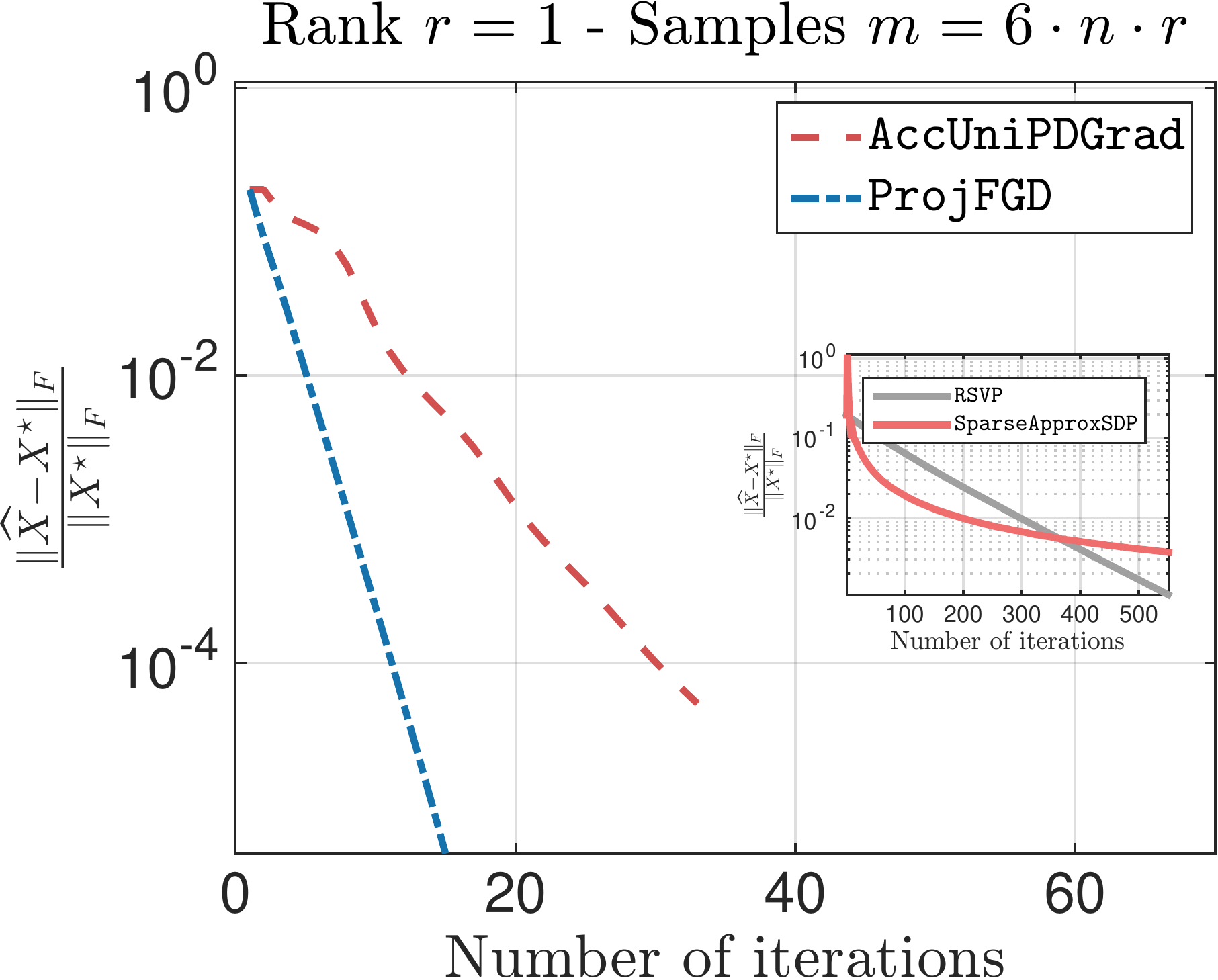} 
\includegraphics[width=0.32\textwidth]{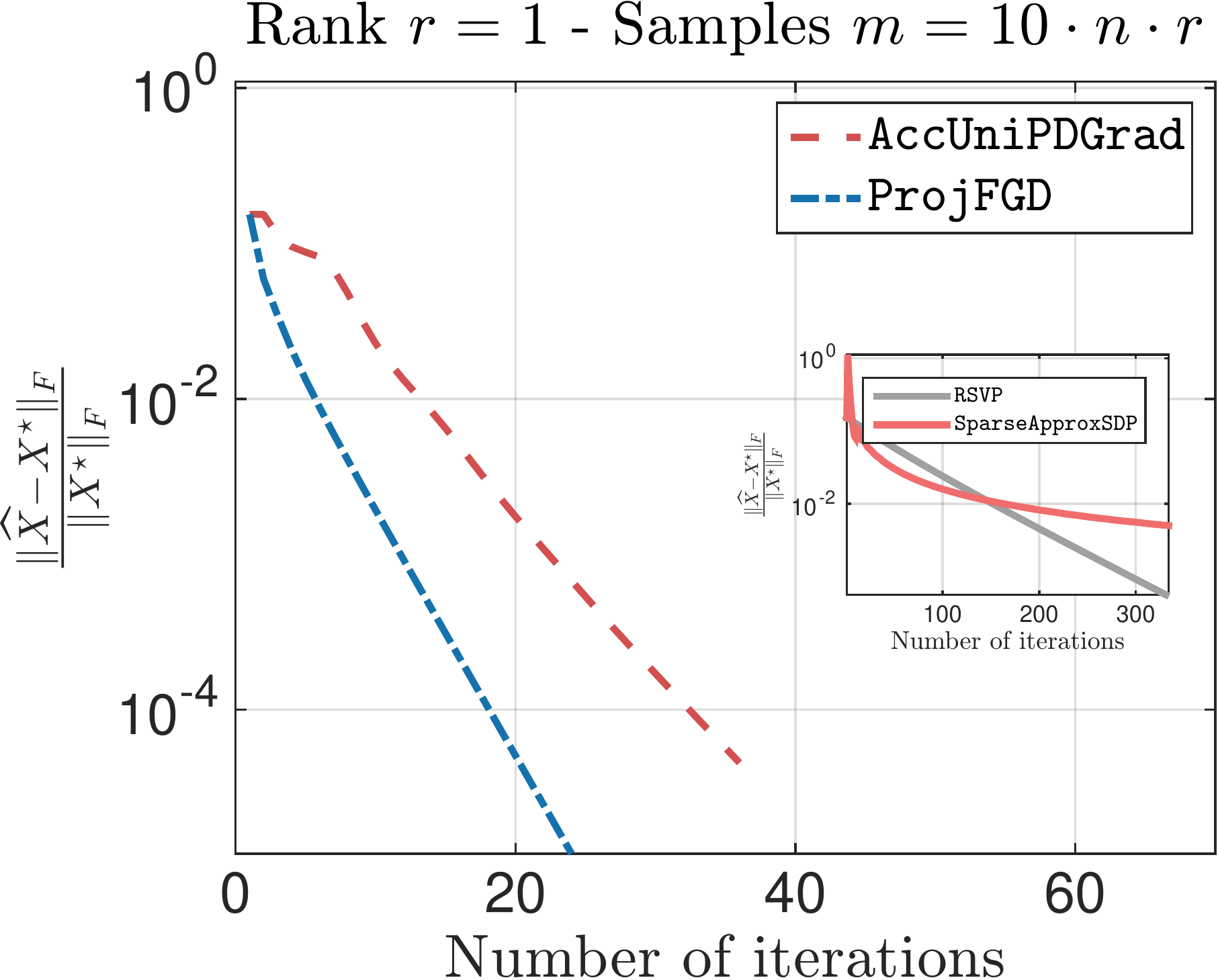} \\
\includegraphics[width=0.32\textwidth]{./figs/pureQST_q_12_C_3_init_ours_XcurvesTime} 
\includegraphics[width=0.32\textwidth]{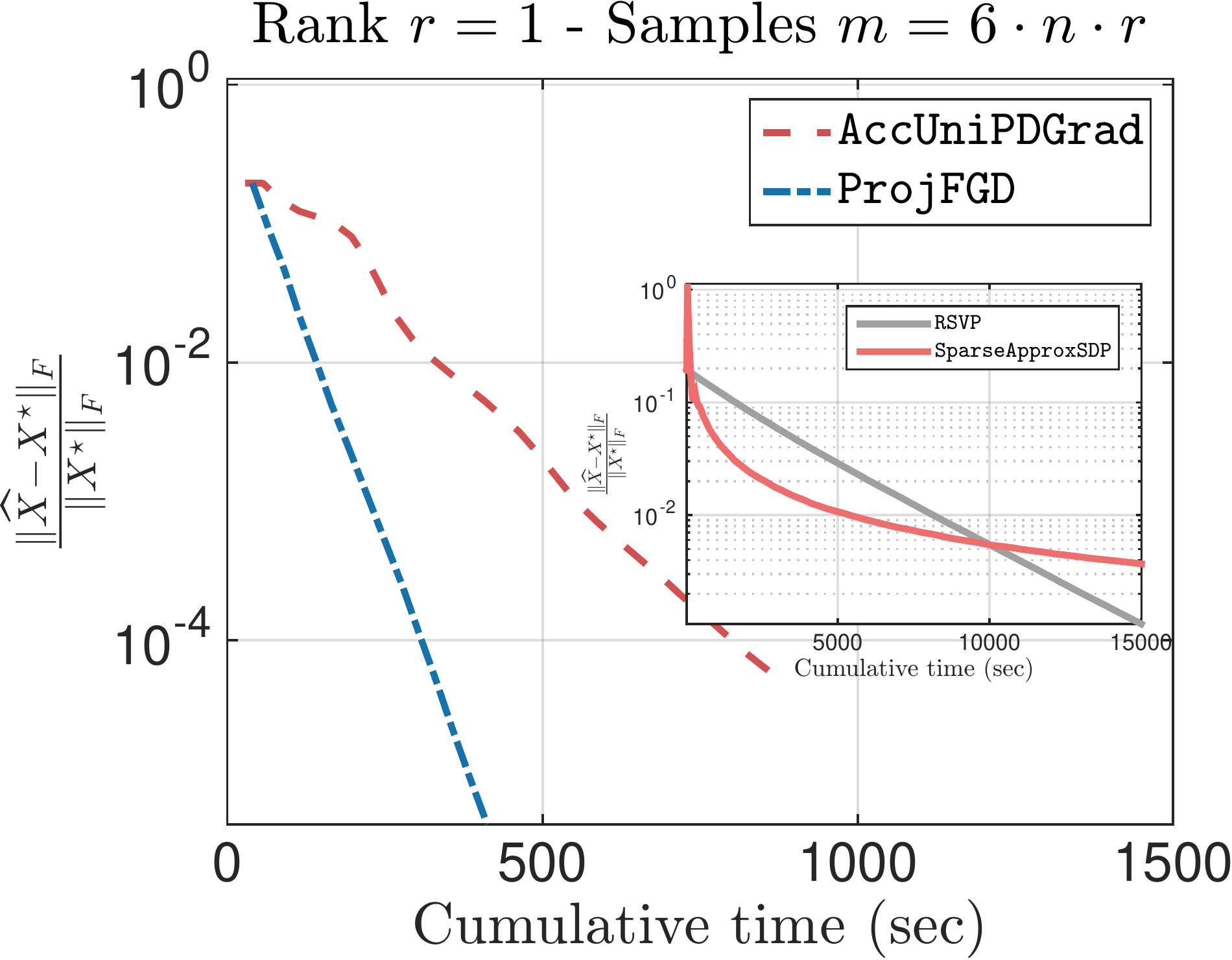} 
\includegraphics[width=0.32\textwidth]{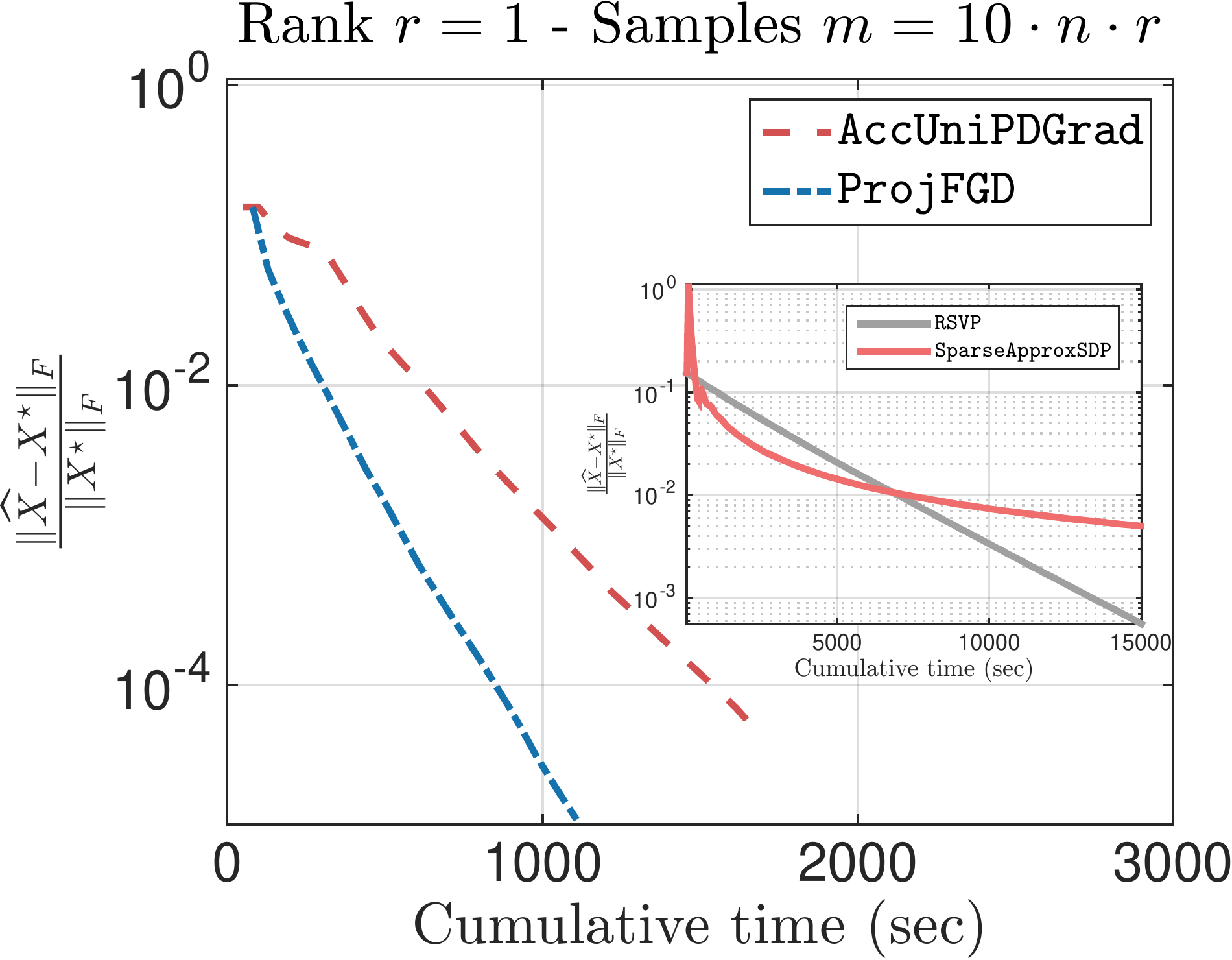} \\
\caption{\textbf{Quantum state tomography:} Convergence performance of algorithms under comparison w.r.t. $\tfrac{\|\widehat{X} - \X^\star\|_F}{\|\X^\star\|_F}$ vs. $(i)$ the total number of iterations (top) and $(ii)$ the total execution time (bottom). First, second and third column corresponds to $C_{\rm sam} = 3, 6$ and $10$, respectively. For all cases, $r = 1$ (pure state setting) and $q = 12$. Initial point is $U_0 = \Pi_{\C}(\widetilde{\U_0})$ such that $\X_0 = \widetilde{\U}_0 \widetilde{\U}_0^\top$ where $\X_0 = \Pi_{+}\left(-\mathcal{A}^*(y)\right)$.
}
\label{fig:app_exp4}
\end{figure*}

\begin{table*}[!h]
\centering
\ra{1.3}
\begin{scriptsize}
\rowcolors{2}{white}{black!05!white}
\begin{tabular}{c c c c c c c c c c c c c} \toprule
& \phantom{a} & \multicolumn{3}{c}{$q = 6$, $C_{\rm sam} = 3$.} & \phantom {a} & \multicolumn{3}{c}{$q = 6$, $C_{\rm sam} = 6$.} & \phantom {a} & \multicolumn{3}{c}{$q = 6$, $C_{\rm sam} = 10$.}\\
\cmidrule {3-5} \cmidrule{7-9} \cmidrule{11-13} 
Algorithm & \phantom{a} & $\tfrac{\|\widehat{\X} - \X^\star\|_F}{\|\X^\star\|_F}$ & \phantom{a}  & Total time 
			   & \phantom{a} & $\tfrac{\|\widehat{\X} - \X^\star\|_F}{\|\X^\star\|_F}$ & \phantom{a}  & Total time 
			   & \phantom{a} & $\tfrac{\|\widehat{\X} - \X^\star\|_F}{\|\X^\star\|_F}$ & \phantom{a}  & Total time  \\
\cmidrule{1-1} \cmidrule {3-3} \cmidrule{5-5} \cmidrule{7-7} \cmidrule{9-9} \cmidrule {11-11} \cmidrule{13-13}
\texttt{RSVP} & & 5.1496e-05 & & 0.7848 & & 1.8550e-05 & & 0.3791 & & 6.6328e-06 & & 0.1203 \\ 
\texttt{SparseApproxSDP} & & 4.6323e-03 & & 3.7404 & & 2.2469e-03 & & 4.3775 & & 1.4776e-03 & & 3.8536 \\ 
\texttt{AccUniPDGrad} & & 4.0388e-05 & & 0.3634 & & 2.4064e-05 & & 0.3311 & & 1.9032e-05 & & 0.4911\\ 
\texttt{ProjFGD} & & 2.4116e-05 & & 0.0599 & & 1.6052e-05 & & 0.0441 & & 1.1419e-05 & & 0.0446 \\ 
\midrule 
& \phantom{a} & \multicolumn{3}{c}{$q = 8$, $C_{\rm sam} = 3$.} & \phantom {a} & \multicolumn{3}{c}{$q = 8$, $C_{\rm sam} = 6$.} & \phantom {a} & \multicolumn{3}{c}{$q = 8$, $C_{\rm sam} = 10$.}\\
\cmidrule {3-5} \cmidrule{7-9} \cmidrule{11-13} 
 \texttt{RSVP} & & 1.5774e-04 & & 5.7347 & & 5.2470e-05 & & 3.8649 & & 2.9583e-05 & & 4.6548 \\ 
\texttt{SparseApproxSDP} & & 4.1639e-03 & & 16.1074 & & 2.2011e-03 & & 33.7608 & & 1.7631e-03 & & 85.0633 \\ 
\texttt{AccUniPDGrad} & & 3.5122e-05 & & 1.1006 & & 2.4634e-05 & & 1.8428 & & 1.7719e-05 & & 3.9440 \\ 
\texttt{ProjFGD} & & 2.4388e-05 & & 0.6918 & & 1.5431e-05 & & 0.8994 & & 1.0561e-05 & & 1.8804 \\ 
 \midrule 
 & \phantom{a} & \multicolumn{3}{c}{$q = 10$, $C_{\rm sam} = 3$.} & \phantom {a} & \multicolumn{3}{c}{$q = 10$, $C_{\rm sam} = 6$.} & \phantom {a} & \multicolumn{3}{c}{$q = 10$, $C_{\rm sam} = 10$.}\\
 \cmidrule {3-5} \cmidrule{7-9} \cmidrule{11-13} 
\texttt{RSVP} & & 4.6056e-04 & & 379.8635 & & 1.8017e-04 & & 331.1315 & & 9.7585e-05 & & 307.9554 \\ 
\texttt{SparseApproxSDP} & & 3.6310e-03 & & 658.7082 & & 2.1911e-03 & & 1326.5374 & & 1.7687e-03 & & 2245.2301 \\ 
\texttt{AccUniPDGrad} & & 3.0456e-05 & & 33.3585 & & 1.9931e-05 & & 56.9693 & & 4.5022e-05 & & 88.2965 \\ 
\texttt{ProjFGD} & & 9.2352e-06 & & 13.9547 & & 5.8515e-06 & & 19.3982 & & 1.0460e-05 & & 49.4528 \\ 
 \midrule 
 & \phantom{a} & \multicolumn{3}{c}{$q = 12$, $C_{\rm sam} = 3$.} & \phantom {a} & \multicolumn{3}{c}{$q = 12$, $C_{\rm sam} = 6$.} & \phantom {a} & \multicolumn{3}{c}{$q = 12$, $C_{\rm sam} = 10$.}\\
 \cmidrule {3-5} \cmidrule{7-9} \cmidrule{11-13} 
\texttt{RSVP} & & 4.7811e-03 & & 14029.1525 & & 1.0843e-03 & & 15028.2836 & & 5.6169e-04 & & 15067.7249 \\ 
\texttt{SparseApproxSDP} & & 3.1717e-03 & & 13635.4238 & & 3.6954e-03 & & 15041.6235 & & 5.0197e-03 & & 15051.4497 \\ 
\texttt{AccUniPDGrad} & & 8.8050e-05 & & 461.2084 & & 5.2367e-05 & & 904.0507 & & 4.5660e-05 & & 1759.6698  \\ 
\texttt{ProjFGD} & & 8.4761e-06 & & 266.8203 & & 4.7399e-06 & & 440.7193 & & 1.1871e-05 & & 1159.2885 \\ 
\bottomrule
\end{tabular}
\end{scriptsize}
\caption{\textbf{Quantum state tomography:} Summary of comparison results for reconstruction and efficiency. As a stopping criterion, we used $\sfrac{\|\X_{i+1} - \X_{i}\|_2}{\|\X_{i+1}\|_2} \leq 5 \cdot 10^{-6}$, where $\X_i$ is the estimate at the $i$-th iteration. Time reported is in seconds. Initial point is $U_0 = \Pi_{\C}(\widetilde{\U_0})$ such that $\X_0 = \widetilde{\U}_0 \widetilde{\U}_0^\top$ where $\X_0 = \Pi_{+}\left(-\mathcal{A}^*(y)\right)$.} \label{tbl:Comp}
\end{table*}

\begin{figure*}[h!]
\centering
\includegraphics[width=0.24\textwidth]{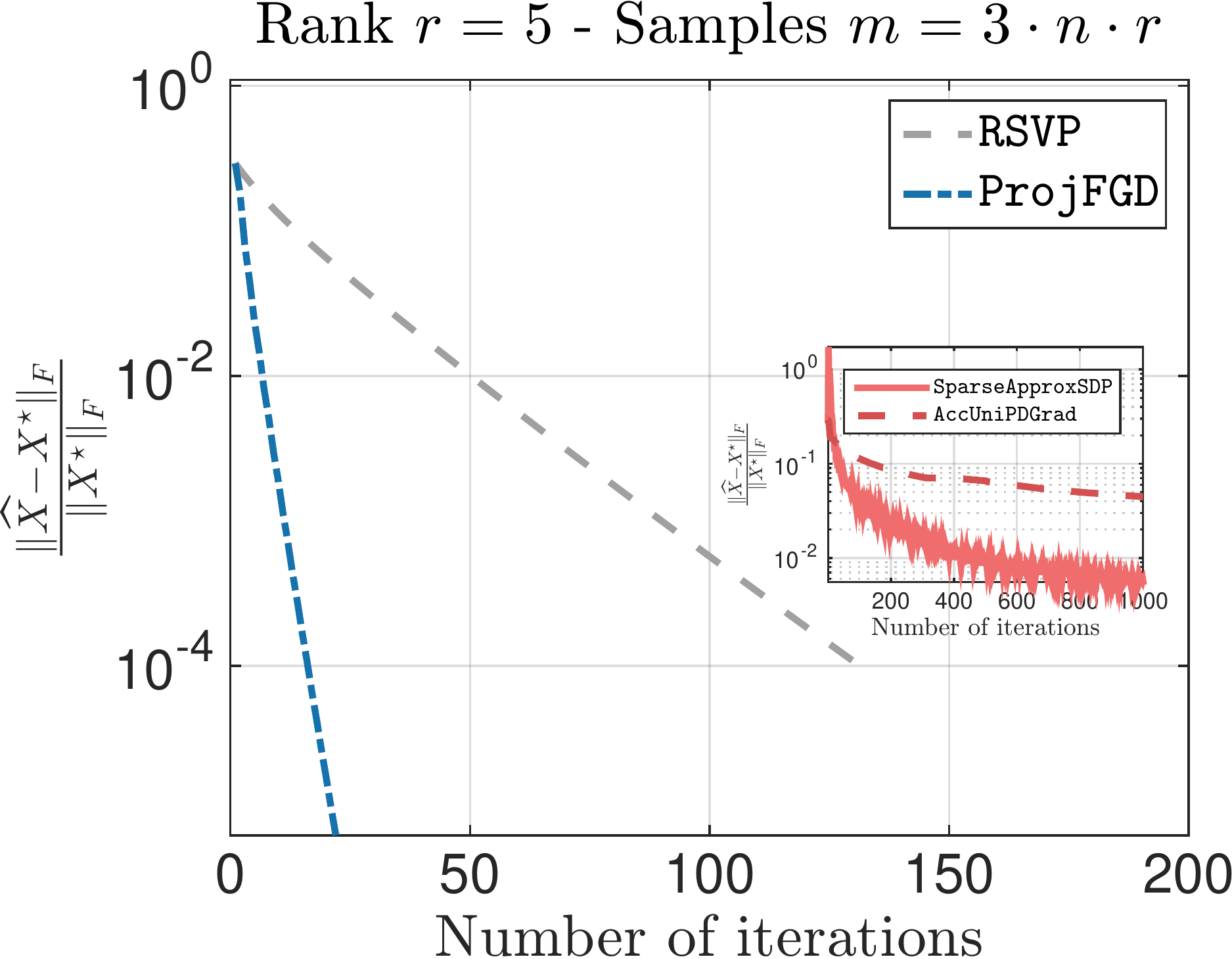} 
\includegraphics[width=0.24\textwidth]{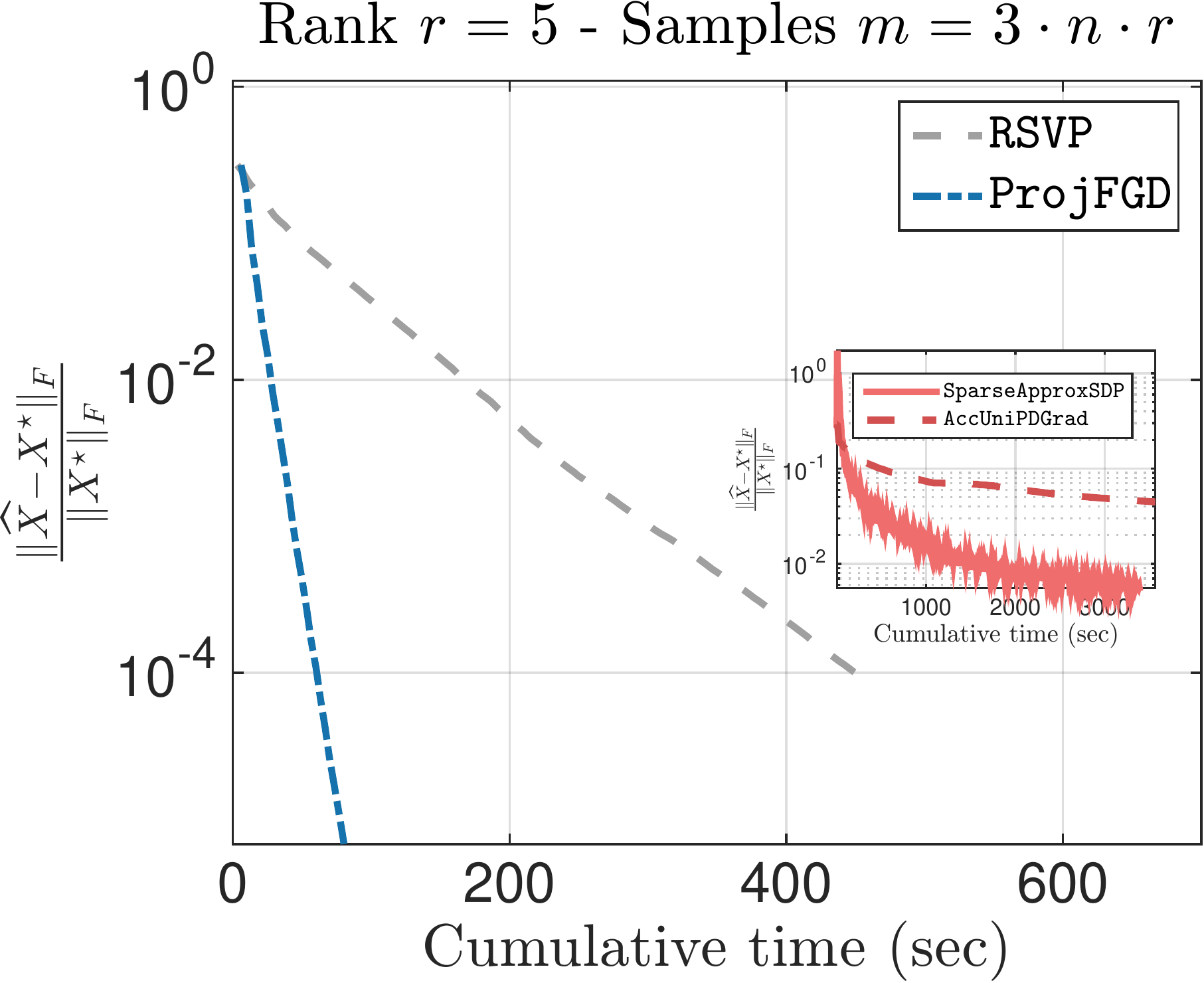}
\includegraphics[width=0.24\textwidth]{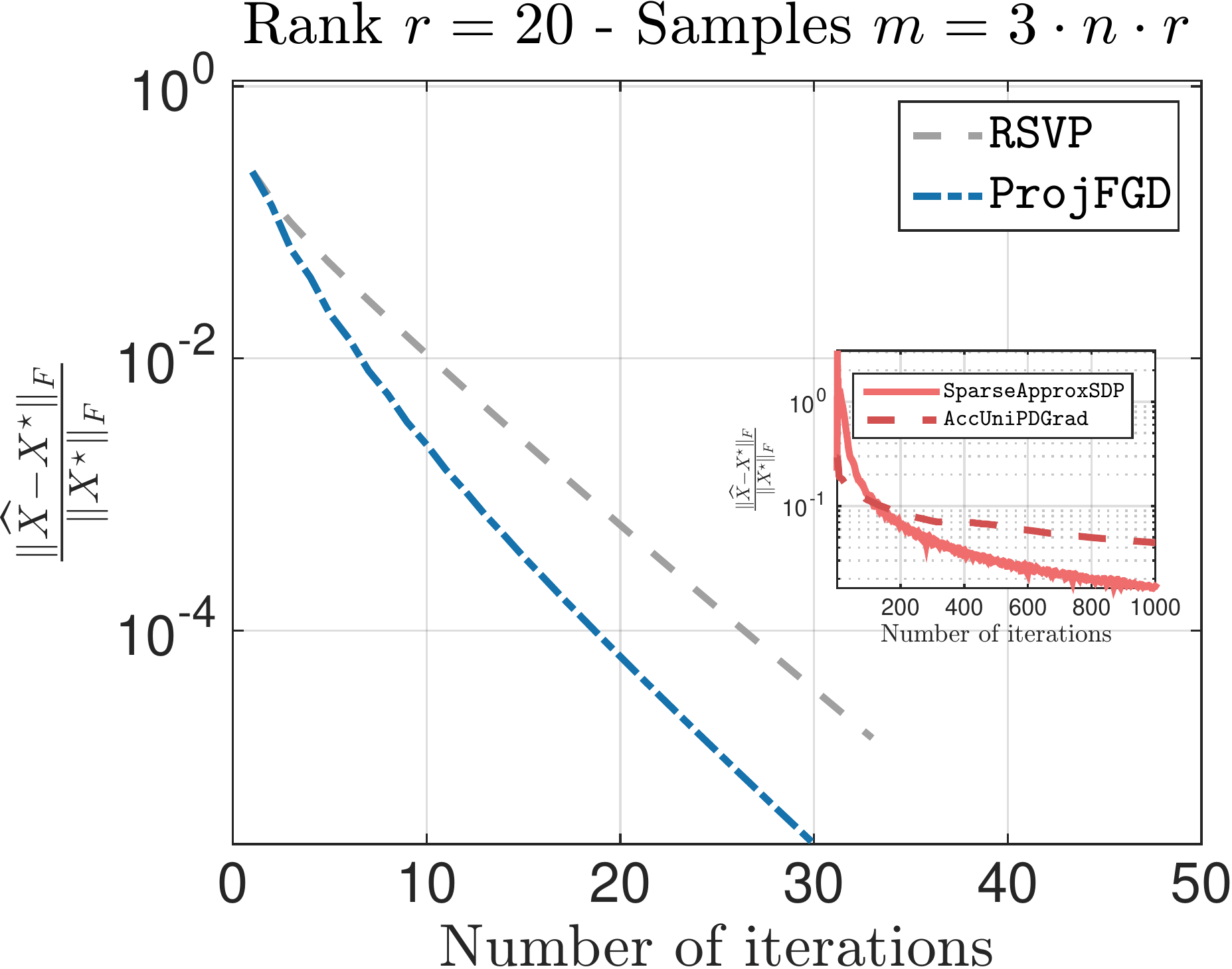} 
\includegraphics[width=0.24\textwidth]{./figs/QST_q_10_C_3_r_20_init_ours_XcurvesTime} 
\caption{\textbf{Quantum state tomography:} Convergence performance of algorithms under comparison w.r.t. $\tfrac{\|\widehat{X} - \X^\star\|_F}{\|\X^\star\|_F}$ vs. $(i)$ the total number of iterations (left) and $(ii)$ the total execution time (right). The two left plots correspond to the case $r = 5$ and the two right plots to the case $r = 20$. In all cases $C_{\rm sam} = 3$ and $q = 10$. Initial point is $U_0 = \Pi_{\C}(\widetilde{\U_0})$ such that $\X_0 = \widetilde{\U}_0 \widetilde{\U}_0^\top$ where $\X_0 = \Pi_{+}\left(-\mathcal{A}^*(y)\right)$.
}
\label{fig:app_exp5}
\end{figure*}

\begin{figure*}[t!]
\centering
\includegraphics[width=0.32\textwidth]{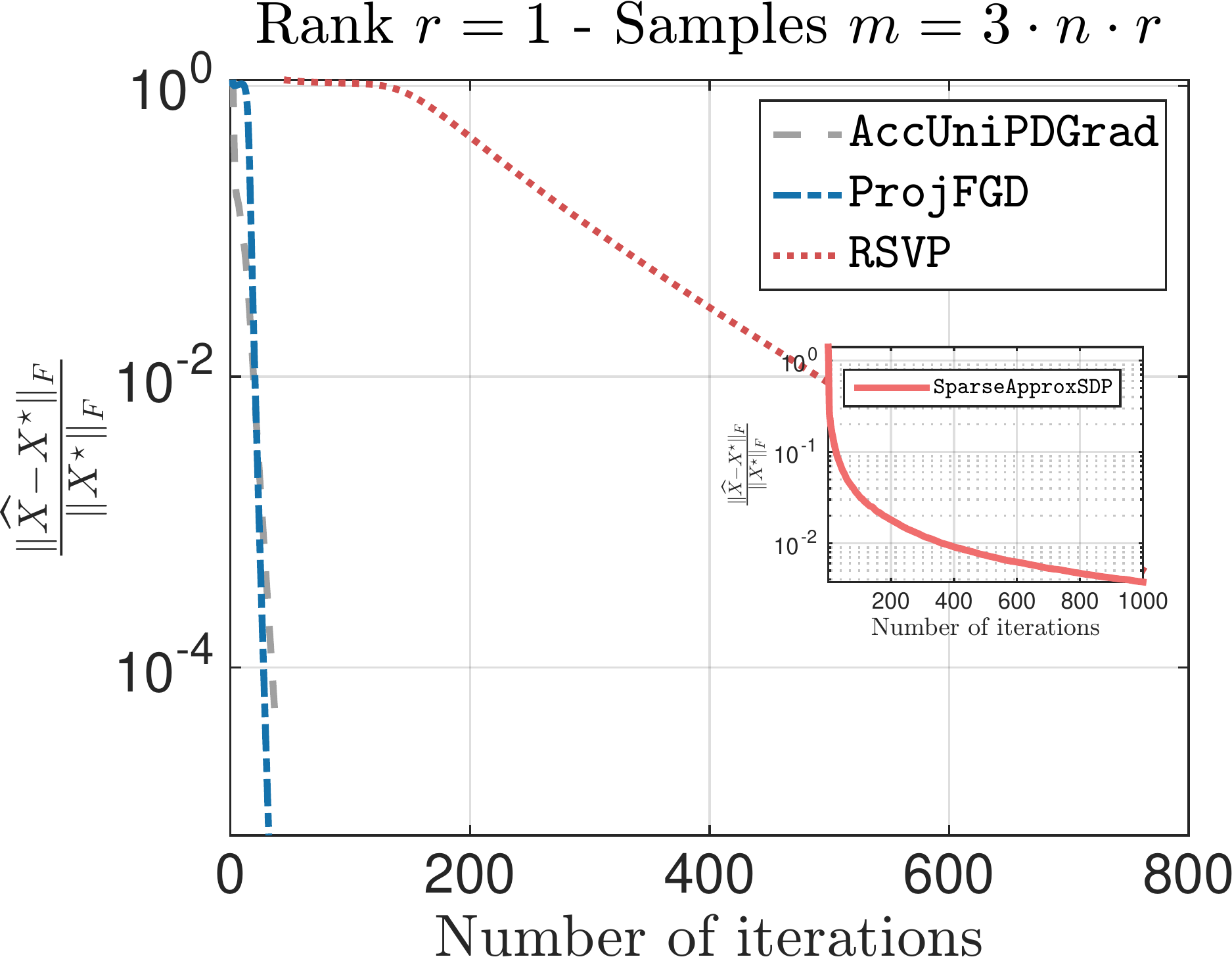} 
\includegraphics[width=0.32\textwidth]{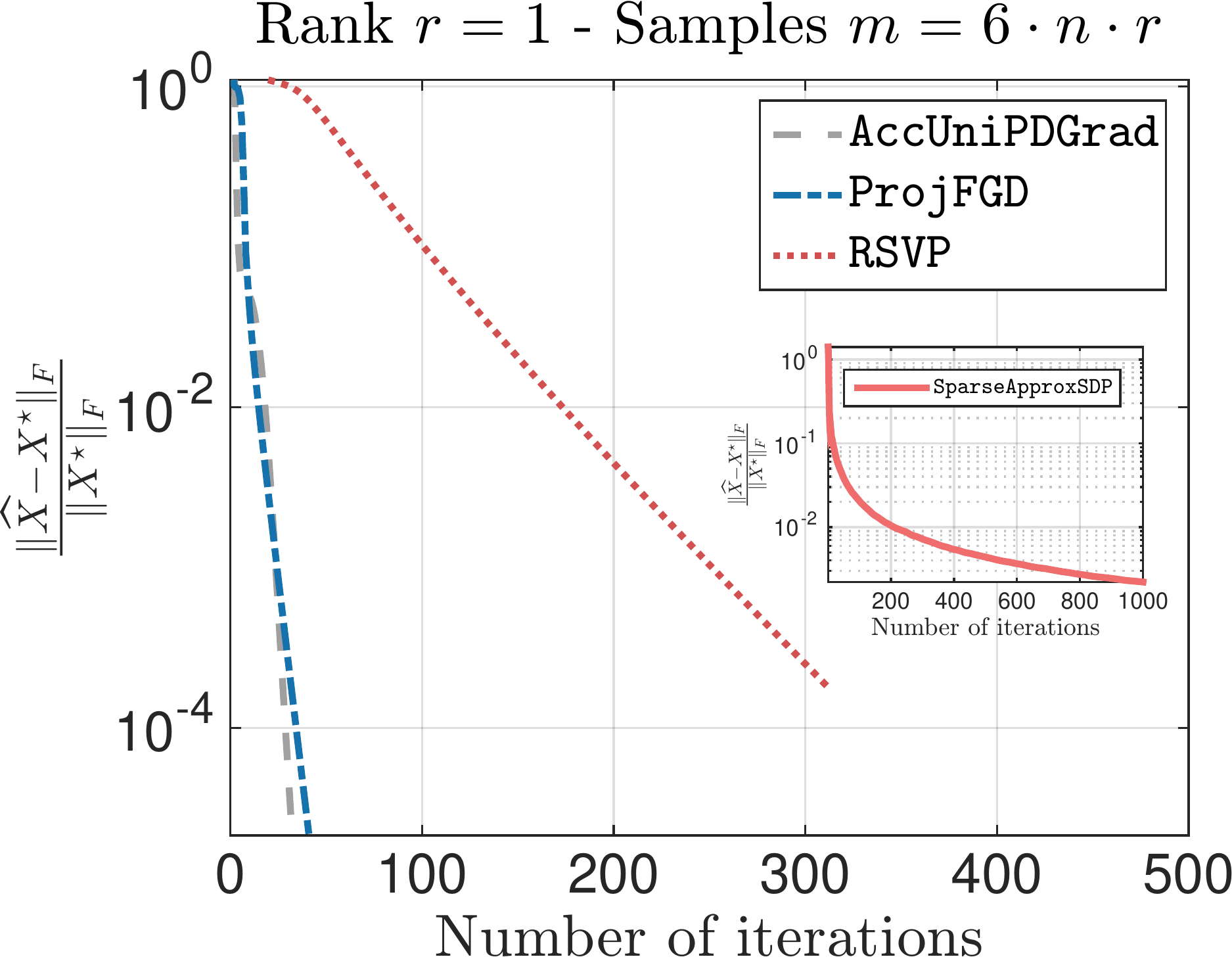}
\includegraphics[width=0.32\textwidth]{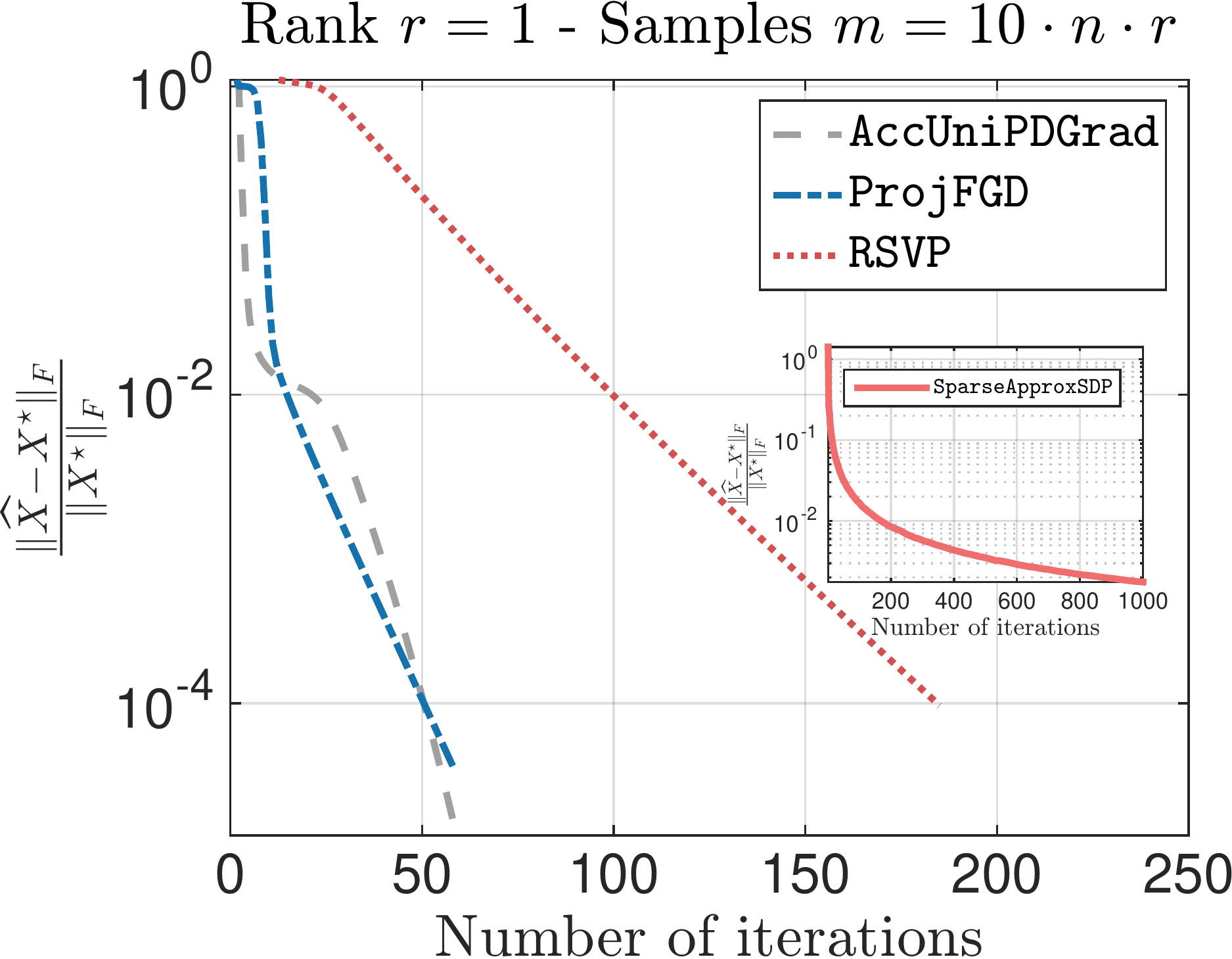} \\
\includegraphics[width=0.32\textwidth]{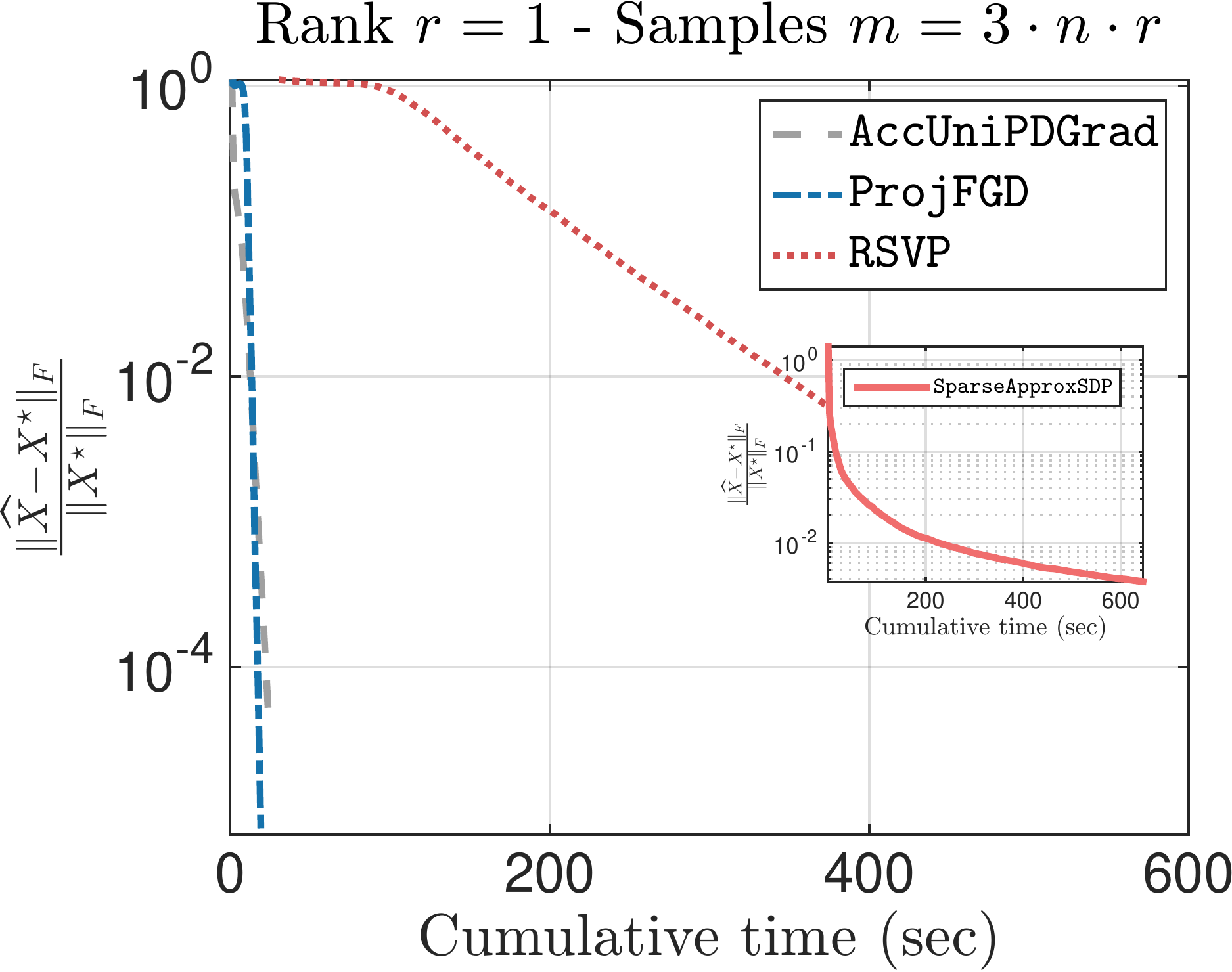} 
\includegraphics[width=0.32\textwidth]{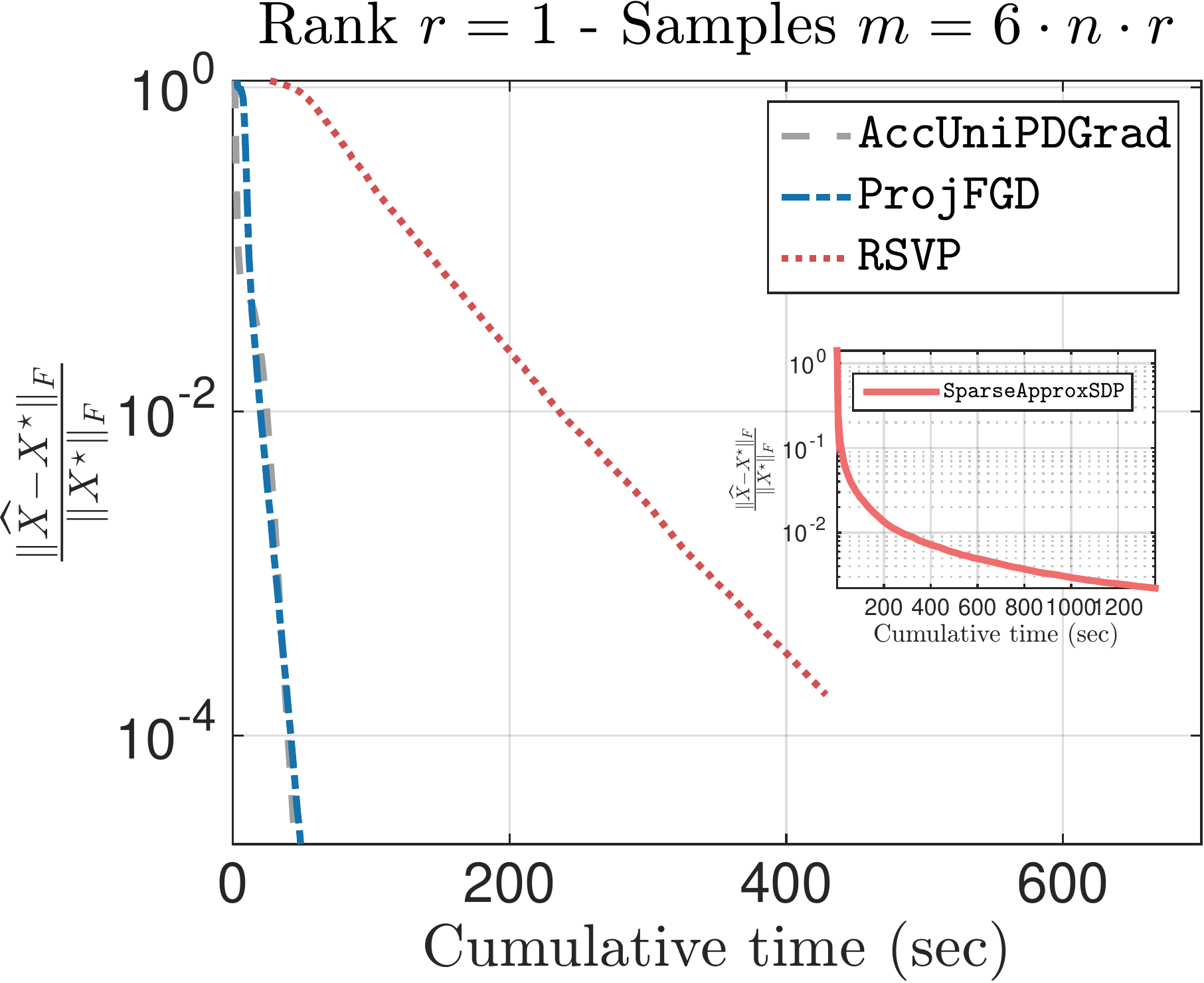}
\includegraphics[width=0.32\textwidth]{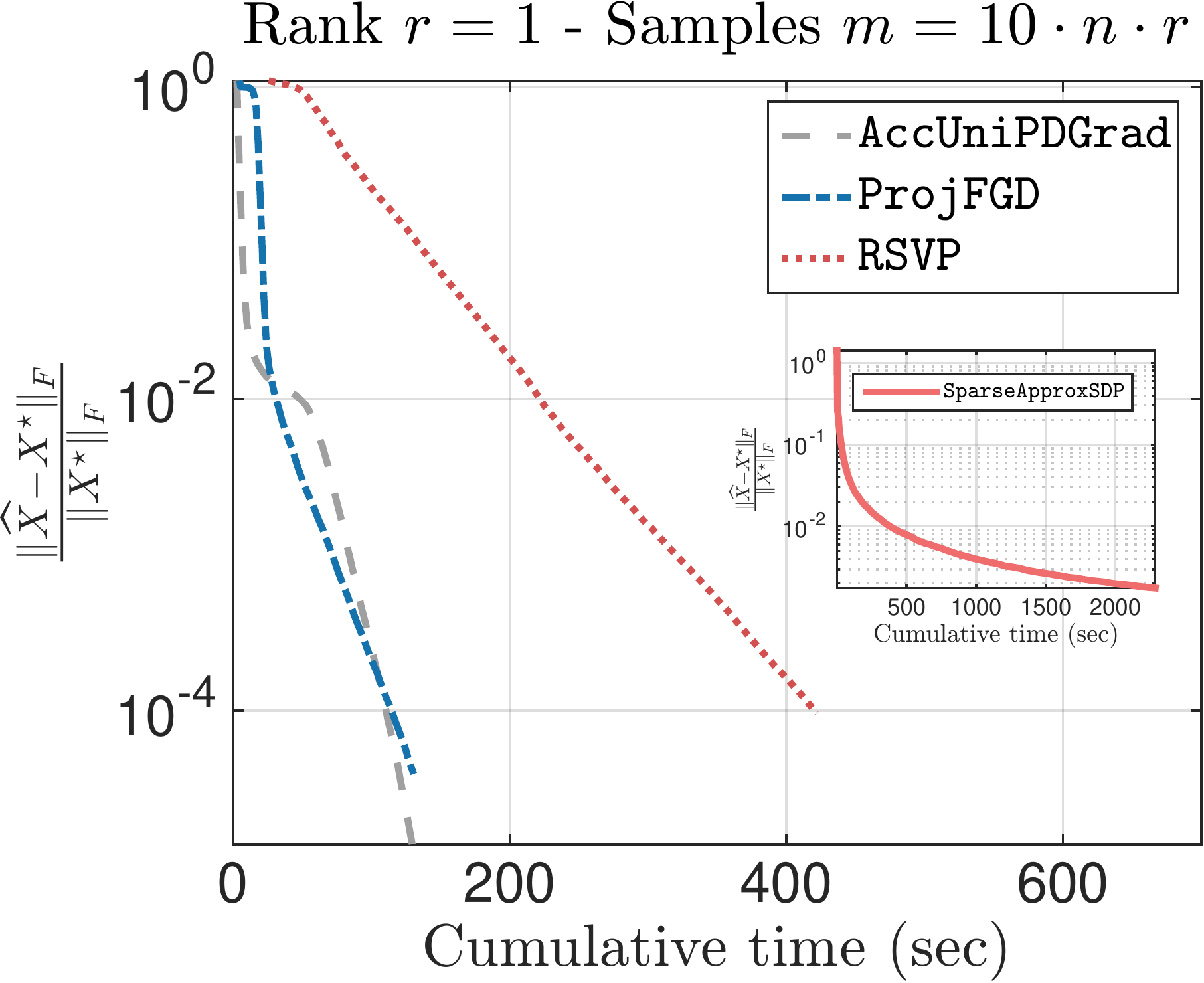}
\caption{\textbf{Quantum state tomography:} Convergence performance of algorithms under comparison w.r.t. $\tfrac{\|\widehat{X} - \X^\star\|_F}{\|\X^\star\|_F}$ vs. $(i)$ the total number of iterations (left) and $(ii)$ the total execution time (right). All results correspond to executions starting from a \textbf{random initialization} (but common to all algorithms). In all cases $r = 1$ and $q = 10$.
}
\label{fig:app_exp6}
\end{figure*}

\begin{table*}[!th]
\centering
\ra{1.3}
\begin{scriptsize}
\rowcolors{2}{white}{black!05!white}
\begin{tabular}{c c c c c c c c c c c c c} \toprule
& \phantom{a} & \multicolumn{3}{c}{$q = 10$, $C_{\rm sam} = 3$.} & \phantom {a} & \multicolumn{3}{c}{$q = 10$, $C_{\rm sam} = 6$.} & \phantom {a} & \multicolumn{3}{c}{$q = 10$, $C_{\rm sam} = 10$.}\\
\cmidrule {3-5} \cmidrule{7-9} \cmidrule{11-13} 
Algorithm & \phantom{a} & $\tfrac{\|\widehat{\X} - \X^\star\|_F}{\|\X^\star\|_F}$ & \phantom{a}  & Total time 
			   & \phantom{a} & $\tfrac{\|\widehat{\X} - \X^\star\|_F}{\|\X^\star\|_F}$ & \phantom{a}  & Total time 
			   & \phantom{a} & $\tfrac{\|\widehat{\X} - \X^\star\|_F}{\|\X^\star\|_F}$ & \phantom{a}  & Total time  \\
\cmidrule{1-1} \cmidrule {3-3} \cmidrule{5-5} \cmidrule{7-7} \cmidrule{9-9} \cmidrule {11-11} \cmidrule{13-13}
\texttt{RSVP} & &  4.5667e-04 & &  545.5525 & & 1.8550e-05 & & 0.3791 & & 1.5774e-04 & & 5.7347 \\ 
\texttt{SparseApproxSDP} & & 3.7592e-03 & &  646.3486 & &  2.2469e-03 & &  4.3775 & &  4.1639e-03 & &  16.1074 \\ 
\texttt{AccUniPDGrad} & & 3.6465e-05 & &  24.8531 & &  2.4064e-05 & &  0.3311 & &  3.5122e-05 & &  1.1006 \\ 
\texttt{ProjFGD} & & 7.0096e-06 & &  19.5502 & &  1.6052e-05 & &  0.0441 & &  2.4388e-05 & &  0.6918 \\ 
\bottomrule
\end{tabular}
\end{scriptsize}
\caption{\textbf{Quantum state tomography:} Summary of comparison results for reconstruction and efficiency for \textbf{random initialization}. As a stopping criterion, we used $\sfrac{\|\X_{i+1} - \X_{i}\|_2}{\|\X_{i+1}\|_2} \leq 5 \cdot 10^{-6}$, where $\X_i$ is the estimate at the $i$-th iteration. Time reported is in seconds.} \label{tbl:Comp3}
\end{table*}

\clearpage

\section*{Proofs of local convergence of the \palgo }{\label{sec:proof}}
Here, we present the full proof of Theorem \ref{thm:projFGD_guarantees}. 
For clarity, we re-state the problem settings: We consider problem cases such as
\begin{equation}{\label{eq:appe_00}}
\begin{aligned}
	& \underset{\X \in \R^{n \times n}}{\text{minimize}}
	& & f(\X) \quad \quad \text{subject to} \quad \X \succeq 0, ~\X \in \C'.
\end{aligned}
\end{equation} 
We assume the optimum $\X^\star$ satisfies $\text{rank}(\X^\star) = r^\star$.
For our analysis, we assume we know $r^\star$ and set $r^\star \equiv r$. 
We solve \eqref{eq:appe_00} in the factored space, by considering the criterion:
\begin{equation}{\label{eq:appe_01}}
\begin{aligned}
	& \underset{\U \in \R^{n \times r}}{\text{minimize}}
	& & f(\U\U^\top) \quad \quad \text{subject to} \quad \U \in \C.
\end{aligned}
\end{equation} 
By faithfulness of $\C$ (Definition \ref{prelim:def_02}), we assume that $\E \subseteq \C$. 
This means that the feasible set $\C$ in \eqref{eq:appe_01} contains all matrices $\Uo$ that lead to $\Xo = \Uo \U^{\star\top}$ in \eqref{eq:appe_00}.
Moreover, we assume both $\C, ~\C'$ are convex sets and there exists a ``mapping" of $\C'$ onto $\C$, such that the two constraints are ``equivalent": for any $\U \in \C$, we are guaranteed that $\X = \U\U^\top \in \C'$. 
We restrict our discussion on norm-based sets for $\C$	such that \eqref{intro:eq_proj} is satisfied.
As a representative example, in our analysis consider the case where, for any $\X = \U\U^\top$, $\trace(\X) \leq 1 \Leftrightarrow \|\U\|_F^2 \leq 1$. 

For our analysis, we will use the following step sizes: 
\begin{align*}
\widehat{\eta} = \tfrac{1}{128(L\|\X_t\|_2 + \|\Q_{\U_t} \Q_{\U_t}^\top\gradf(\X_t)\|_2)}, ~~\eta^\star = \tfrac{1}{128(L\|\X^\star\|_2 + \|\gradf(\X^\star)\|_2)}.
\end{align*}
By Lemma A.5 in \cite{bhojanapalli2015dropping}, we know that $\widehat{\eta} \geq \tfrac{5}{6} \eta$ and $\tfrac{10}{11}\eta^\star \leq \eta \leq \tfrac{11}{10} \eta^\star$. 
In our proof, we will work with step size $\widehat{\eta}$, which is equivalent --up to constants-- to the original step size $\eta$ in the algorithm. 

For ease of exposition, we re-define the sequence of updates:
$\U_t $ is the current estimate in the factored space, 
$\widetilde{\U}_{t+1} = \U_t - \widehat{\eta} \gradf(\X_t)\U_t$ is the putative solution after the gradient step (observe that $\widetilde{\U}_{t+1}$ might belong in $\C$),
and $\U_{t+1} = \Pi_{\C}(\widetilde{U}_{t+1})$ is the projection step onto $\C$. 
Observe that for the constraint cases we consider in this paper, $\U_{t+1} = \Pi_{\C}(\widetilde{\U}_{t+1}) = \xi_t(\widetilde{\U}_{t+1}) \cdot \widetilde{U}_{t+1}$, where $\xi_t(\cdot) \in (0, 1)$; in the case $\xi_t(\cdot) = 1$, the algorithm boils down to the \textsc{Fgd} algorithm.
For simplicity, we drop the subscript and the parenthesis of the $\xi$ parameter; these values are apparent from the context.

We assume that \palgo is initialized with a ``good'' starting point $\X_0 = \U_0 \U_0^\top$, such that:
\begin{itemize}
\item [$(A1)$] \quad $\U_0 \in \C$ \quad and \quad $ \dist(\U_0, \Uo) \leq \rho' \sigma_{r}(\Uo)$ ~~\text{for } $\rho' := c \cdot \tfrac{\mu}{L} \cdot \tfrac{\sigma_r(\Xo)}{\sigma_1(\Xo)}  $, where $c \leq \tfrac{1}{200}$.
\end{itemize}
By the assumptions above, $\X_0 = \U_0\U_0^\top \in \C'$. 
Next, we show that the above lead to a local convergence result. 
A practical initialization procedure is given in Section \ref{sec:init} and follows from  \cite{bhojanapalli2015dropping}; this also is used in the experimental section \ref{sec:motiv}.
	
\begin{wrapfigure}[9]{r}{4cm}
	\centering
	\vspace{-1cm}
	\includegraphics[width=0.2\columnwidth]{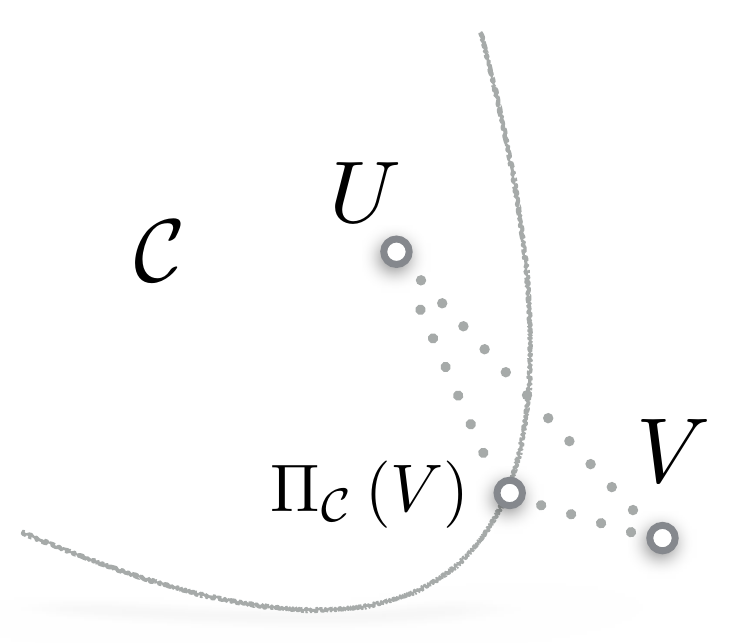}
	\caption{Illustration of Lemma \ref{lem:proj}}{\label{fig:proj}}
\end{wrapfigure}
\subsection{Proof of Theorem \ref{thm:projFGD_guarantees}}
For our analysis, we make use of the following lemma \cite[Chapter 3]{bubeck2014theory}, which characterizes the effect of projections onto convex sets w.r.t. to inner products, as well as provides 
a type-of triangle inequality for such projections; see also Figure \ref{fig:proj} for a simple illustration.
\begin{lemma}{\label{lem:proj}}
Let $\U \in \C \subseteq \R^{n \times r}$ and $\V \in \R^{n \times r}$ where $\V \notin \C$. Then,
\begin{align}\label{eq:proj_00}
\left\langle \Pi_\C(\V) - \U, \V - \Pi_\C(\V) \right\rangle \geq 0.
\end{align} 
\end{lemma}

\begin{proof}[Proof of Theorem \ref{thm:projFGD_guarantees}]
We start with the following series of (in)equalities:
\begin{align*}
\dist\left(\U_{t+1}, ~\Uo\right)^2 \nonumber &= \min_{R \in \mathcal{O}} \|\U_{t+1} - \Uo R\|_F^2 \nonumber \\ 
											  &\stackrel{(i)}{\leq} \|\U_{t+1} - \Uo \Rus\|_F^2 \\
											  &\stackrel{(ii)}{=} \|\U_{t+1} - \Uw_{t+1} + \Uw_{t+1} - \Uo \Rus\|_F^2 \\
											  &= \|\U_{t+1} - \Uw_{t+1}\|_F^2 + \|\Uw_{t+1} - \Uo \Rus\|_F^2 \\
											  &\quad \quad \quad \quad \quad \quad \quad \quad ~+ 2\left\langle \U_{t+1} - \Uw_{t+1}, ~\Uw_{t+1} - \Uo \Rus \right\rangle,
\end{align*} where $(i)$ is due to the fact $\Rus := \argmin_{R \in \mathcal{O}} \|\U_{t} - \Uo R\|_F^2$, $(ii)$ is obtained by adding and subtracting $\Uw_{t+1}$.

Focusing on the second term of the right hand side, we substitute $\Uw_{t+1}$ to obtain:
\begin{align*}
\|\Uw_{t+1} - \Uo \Rus\|_F^2 \nonumber &= \|\U_t - \weta \gradf\left(\U_t \U_t^\top\right)\U_t - \Uo \Rus\|_F^2 \\
										 &= \|\U_t - \Uo \Rus\|_F^2 + \weta^2 \|\gradf \left(\U_t \U_t^\top\right) \U_t\|_F^2 \nonumber \\ 
										 &\quad \quad \quad \quad \quad \quad \quad \quad~- 2\weta\left\langle \gradf \left(\U_t \U_t^\top\right) \U_t,~\U_t - \Uo \Rus \right\rangle
\end{align*} 
Then, our initial equation transforms into:
\begin{align*}
\dist\left(\U_{t+1}, ~\Uo\right)^2 \nonumber &\leq \|\U_{t+1} - \Uw_{t+1}\|_F^2 + \dist\left(\U_t, ~\Uo\right)^2 + \weta^2 \|\gradf \left(\U_t \U_t^\top\right) \U_t\|_F^2 \nonumber \\ 
&\quad \quad \quad \quad \quad \quad \quad \quad~- 2\weta \left\langle \gradf \left(\U_t \U_t^\top\right) \U_t,~\U_t - \Uo \Rus \right\rangle \nonumber \\ 
&\quad \quad \quad \quad \quad \quad \quad \quad~+ 2\left\langle \U_{t+1} - \Uw_{t+1}, ~\Uw_{t+1} - \Uo \Rus \right\rangle
\end{align*}
Focusing further on the last term of the expression above, we obtain:
\begin{align*}
\left\langle \U_{t+1} - \Uw_{t+1}, ~\Uw_{t+1} - \Uo \Rus \right\rangle &= \left\langle \U_{t+1} - \Uw_{t+1}, ~\Uw_{t+1} - \U_{t+1} + \U_{t+1} - \Uo \Rus \right\rangle \\
&= \left\langle \U_{t+1} - \Uw_{t+1}, ~\Uw_{t+1} - \U_{t+1} \right\rangle \nonumber \\ 
&\quad \quad \quad \quad \quad \quad \quad+ \left\langle \U_{t+1} - \Uw_{t+1}, ~\U_{t+1} - \Uo \Rus \right\rangle
\end{align*} 
Observe that, in the special case where $\Uw_{t+1} \equiv \U_{t+1}$ for all $t$, \emph{i.e.}, the iterates are always within $\C$ before the projection step, the above equation equals to zero and the recursion is identical to that of \cite{bhojanapalli2015dropping}[Proof of Theorem 4.2]. 
Here, we are more interested in the case where $\Uw_{t+1} \not\equiv \U_{t+1}$ for some $t$---thus $\Uw_{t+1} \not\in \C$. 
By faithfulness (Definition \ref{prelim:def_04}), observe that $\Uo \Rus \in \C$ and $\Xo = \Uo \Rus \left(\Uo \Rus\right)^\top = \Uo \U^{\star \top}$.
Moreover, $\U_{t+1} = \Pi_{\C}(\Uw_{t+1})$:
Then, according to Lemma \ref{lem:proj} and focusing on eq. \eqref{eq:proj_00}, for $U := \Uo \Rus$ and $V := \Uw_{t+1}$, the last term in the above equation satisfies: 
\begin{align*}
\left\langle \U_{t+1} - \Uw_{t+1}, ~\U_{t+1} - \Uo \Rus \right\rangle \leq 0,
\end{align*} 
and, thus, the expression above becomes:
\begin{align*}
\left\langle \U_{t+1} - \Uw_{t+1}, ~\Uw_{t+1} - \Uo \Rus \right\rangle \leq - \|\U_{t+1} - \Uw_{t+1}\|_F^2.
\end{align*}
Therefore, going back to the original recursive expression, we obtain:
\begin{align*}
\dist\left(\U_{t+1}, ~\Uo\right)^2 &\leq - \|\U_{t+1} - \Uw_{t+1}\|_F^2 + \dist\left(\U_t, ~\Uo\right)^2 + \weta^2 \|\gradf \left(\U_t \U_t^\top\right) \U_t\|_F^2 \nonumber \\ 
&\quad \quad \quad \quad \quad  \quad \quad \quad \quad- 2\weta\left\langle \gradf \left(\U_t \U_t^\top\right) \U_t, \U_t - \Uo \Rus \right\rangle 
\end{align*} 

For the last term, we use the descent lemma~\ref{lem:gradU,U-U_r_ bound} in the main text; the proof is provided in Section \ref{descent_lemma_proof}.
Thus, we can conclude that:
\begin{align*}
\dist\left(\U_{t+1}, ~\Uo\right)^2 &\leq \left(1 - \tfrac{3\weta \mu}{10} \cdot \sigma_r(\Xo)\right) \cdot \dist(\U_t, \Uo)^2.
\end{align*} 


The expression for $\alpha$ is obtained by observing $\weta \geq \tfrac{5}{6} \eta$ and $\tfrac{10}{11}\eta^\star \leq \eta \leq \tfrac{11}{10}\eta^\star$, from Lemma 20 in \cite{bhojanapalli2015dropping}. 
Then, for $\eta^\star \leq \frac{C}{L\norm{\X^\star}_2 + \norm{\gradf(X^\star) }_2}$ and $C = \sfrac{1}{128}$, we have:
\begin{align*}
1 - \frac{3\weta \mu}{10} \cdot \sigma_r(\Xo) &\leq 1 - \frac{3 \cdot \tfrac{10}{11} \cdot \tfrac{5}{6} \eta^\star \mu}{10} \cdot \sigma_r(\Xo) \nonumber \\ 
															&= 1 - \frac{15}{66} \eta^\star \mu \cdot \sigma_r(\Xo) \nonumber \\ 
															&= 1 - \frac{15}{66} \frac{\mu \cdot \sigma_r(\Xo)}{128(L \|\Xo\|_2 + \|\gradf(\Xo)\|_2)} \nonumber \\ 
															&\leq 1 - \frac{\mu \cdot \sigma_r(\Xo)}{550(L \|\Xo\|_2 + \|\gradf(\Xo)\|_2)} =: \alpha
\end{align*} where $\alpha < 1$.

Concluding the proof, the condition $\dist(\U_{t+1}, \Uo)^2 \leq \rho' \sigma_{r}(\Uo)$ is naturally satisfied, since $\alpha < 1$.
\end{proof}

\subsection{Proof of Lemma \ref{lem:gradU,U-U_r_ bound}} {\label{descent_lemma_proof}}
First we recall the definition of \emph{restricted strong convexity}:
\begin{definition}{\label{app:def_00}}
Let $f: \R^{n \times n} \rightarrow \R$ be convex and differentiable. Then, $f$ is $(\mu, r)$-restricted strongly convex if: 
\begin{equation}\label{eq:app_sc}
f(\Y) \geq f(\X) + \ip{\gradf\left(\X\right)}{\Y - \X} + \tfrac{\mu}{2} \norm{Y - \X}_F^2, \quad \text{$\forall \X, \Y \in  \R^{n \times n}$, rank-$r$ matrices.}
\end{equation} 
\end{definition}
The statements below apply also for standard $\mu$-strong convex functions, as defined in Definition \ref{prelim:def_00}.

Recall $\widetilde{U}_{t+1} = \U_t - \weta\gradf(X_t)U_t$ and define $\Delta := \U_t - \Uo R_{\U_t}^\star$. 
Before presenting the proof, we need the following lemma that bounds one of the error terms arising in the proof of Lemma~\ref{lem:gradU,U-U_r_ bound}. 
This is a variation of Lemma 6.3 in \cite{bhojanapalli2015dropping}. 
The proof is presented in Section \ref{sec:DD_proof}.

\begin{lemma}\label{lem:DD_bound_sc}
Let $f$ be $L$-smooth and $(\mu, r)$-restricted strongly convex. 
Then, under the assumptions of Theorem \ref{thm:projFGD_guarantees} and assuming step size $\widehat{\eta} = \tfrac{1}{128(L\|\X\|_2 + \|\gradf(\X_t)\Q_{\U_t} \Q_{\U_t}^\top\|_2)}$,
the following bound holds true:
\begin{align}{\label{eq:appe_combine1}}
\ip{\gradf(X_t) }{ \Delta \Delta^\top} \geq - \tfrac{\weta}{5} \|\gradf(\X_t) U_t \|_F^2 - \tfrac{\mu\sigma_{r}(\Xo)}{10} \cdot \dist(\U_t, \Uo)^2.
\end{align} 
\end{lemma} 

Now we are ready to present the proof of Lemma~\ref{lem:gradU,U-U_r_ bound}.
\begin{proof}[Proof of Lemma \ref{lem:gradU,U-U_r_ bound}]
First we rewrite the inner product as shown below.
\begin{align}
\ip{\gradf(X_t)U_t}{\U_t - \Uo R_{\U_t}^\star}&= \ip{\gradf(X_t) }{\X_t - \Uo R_{\U_t}^\star \U_t^\top} \nonumber \\
&=\tfrac{1}{2}\ip{\gradf(X_t) }{\X_t -\Xo}  + \ip{\gradf(X_t) }{\tfrac{1}{2}(\X_t + \Xo) - \Uo R_{\U_t}^\star \U_t^\top} \nonumber \\
&=\tfrac{1}{2}\ip{\gradf(X_t) }{\X_t -\Xo}  + \tfrac{1}{2} \ip{\gradf(X_t) }{\Delta \Delta^\top}, \label{proofsr1:eq_09}
\end{align} 
which follows by adding and subtracting $\tfrac{1}{2}\Xo$.

Let us focus on bounding the first term on the right hand side of \eqref{proofsr1:eq_09}. 
Consider points $\X_t = \U_t\U_t^\top$ and $\X_{t+1} = \U_{t+1} \U_{t+1}^\top$; by assumption, both $\X_t$ and $\X_{t+1}$ are feasible points in \eqref{eq:appe_01}.
By smoothness of $\f$, we get:
\begin{align}
\f(\X_t) &\geq \f(\X_{t+1}) -\ip{\gradf(\X_t)}{\X_{t+1} -\X_t} - \tfrac{L}{2} \norm{\X_{t+1} -\X_t}_F^2 \nonumber \\ 
&\stackrel{(i)}{\geq} \f(\Xo) -\ip{\gradf(\X_t)}{\X_{t+1} -\X_t} - \tfrac{L}{2} \norm{\X_{t+1} -\X_t}_F^2 , \label{eq:gradX,X-X_r_00}
 \end{align}
where $(i)$ follows from optimality of $\Xo$ and since $\X_{t+1}$ is a feasible point $(\X_{t+1} \succeq 0, ~\Pi_{\C'}(\X_{t+1}) = \X_{t+1})$ for problem~\eqref{eq:appe_00}. 

Moreover, by the $(\mu, r)$-restricted strong convexity of $\f$, we get, 
\begin{align} 
\f(\Xo) \geq \f(\X_t) +\ip{\gradf(\X_t)}{\Xo-\X_t} +\tfrac{\mu}{2}\norm{\Xo -\X_t}_F^2 .\label{eq:gradX,X-X_r_02}
\end{align}
Combining equations~\eqref{eq:gradX,X-X_r_00}, and~\eqref{eq:gradX,X-X_r_02}, we obtain: 
\begin{align} 
\ip{\gradf(\X_t)}{\X_t-\Xo} \geq  \ip{\gradf(\X_t)}{\X_t -\X_{t+1}} -\tfrac{L}{2} \norm{\X_{t+1} -\X_t}_F^2+\tfrac{\mu}{2}\norm{\Xo -\X_t}_F^2  \label{eq:gradX,X-X_r_03} 
\end{align}
By the nature of the projection $\Pi_\C(\cdot)$ step, it is easy to verify that $$\X_{t+1} =  \xi^2 \cdot \left(\X_t - \weta \gradf(\X_t) \X_t \Lambda_t - \weta \Lambda_t^\top \X_t^\top \gradf(\X_t)^\top\right),$$ where $\Lambda_t = I - \tfrac{\weta}{2} Q_{U_t} Q_{U_t}^\top \gradf(X_t) \in \R^{n \times n}$ and $Q_{\U_t} Q_{\U_t}^\top$ denoting the projection onto the column space of $\U_t$.
Notice that, for step size $\weta$, we have 
\begin{align*}
\Lambda_t \succ 0, \quad  \sigma_1\left(\Lambda_t\right) \leq  1+ \sfrac{1}{256} \quad \text{and} \quad \sigma_{n}(\Lambda_t) \geq  1- \sfrac{1}{256}.
\end{align*} 

Using the above $\X_{t+1}$ characterization in~\eqref{eq:gradX,X-X_r_03}, we obtain:
\begin{align}\label{eq:appe_000}
 \ip{\gradf(\X_t)}{\X_t-\Xo} &- \tfrac{\mu}{2}\norm{\Xo -\X_t}_F^2 + \tfrac{L}{2} \norm{\X_t -\X_{t+1}}_F^2 \nonumber \\
 								 &\stackrel{(i)}{\geq} \ip{\gradf(X_t)}{\left(1 - \xi^2\right)X_t} + 
 								 2\weta \cdot \xi^2 \cdot \ip{\gradf(\X_t)}{\gradf(X_t)X_t \Lambda_t} \nonumber \\
 								 &\stackrel{(ii)}{\geq}\left(1 - \xi^2\right) \cdot \ip{\gradf(X_t)U_t}{U_t} + 2\weta \cdot \xi^2 \cdot \trace( \gradf(X_t) \gradf(X_t)X_t) \cdot \sigma_{n}(\Lambda_t) \nonumber \\
 								 &\geq \left(1 - \xi^2\right) \cdot \ip{\gradf(X_t)\U_t}{\U_t} + \tfrac{255\cdot\weta \cdot \xi^2}{128} \|\gradf(X_t) U_t\|_F^2,
 \end{align} 
where: $(i)$ follows from symmetry of $\gradf(X_t)$ and $\X_t$ and, $(ii)$ follows from the sequence equalities an inequalites:
\begin{small}
\begin{align*}
\trace(\gradf(X_t)\gradf(X_t)X_t\Lambda_t) &= \trace(\gradf(X_t)\gradf(X_t)U_tU_t^\top) - \tfrac{\weta}{2}\trace(\gradf(X_t)\gradf(X_t)U_tU_t^\top \gradf(X_t)) \nonumber \\
&\geq \left(1 -\tfrac{\weta}{2} \|Q_{\U_t} Q_{\U_t}^\top\gradf(X_t)\|_2\right) \|\gradf(\X_t) U_t\|_F^2 \nonumber \\ 
&\geq \left(1- \sfrac{1}{256} \right) \|\gradf(\X_t) U_t\|_F^2.
\end{align*} 
\end{small}
 Combining the above in the expression we want to lower bound: $2 \weta \left \langle \gradf(\X_t) \cdot \U_t, \U_t - \Uo R_{\U_t}^\star \right \rangle + \|\U_{t+1} - \widetilde{\U}_{t+1}\|_F^2$, we obtain:
\begin{align}
2 \weta \left \langle \gradf(\X_t) \cdot \U_t, \U_t - \Uo R_{\U_t}^\star \right \rangle &+ \|\U_{t+1} - \widetilde{\U}_{t+1}\|_F^2 \nonumber \\ 
&=\weta \ip{\gradf(X_t) }{\X_t -\Xo}  + \weta \ip{\gradf(X_t) }{\Delta \Delta^\top} + \|\U_{t+1} - \widetilde{\U}_{t+1}\|_F^2 \nonumber \\ 
&\geq \left(1 - \xi^2\right) \cdot \weta \ip{\gradf(X_t)\U_t}{\U_t} + \tfrac{255 \cdot \weta^2 \cdot \xi^2}{128} \|\gradf(X_t) U_t\|_F^2 \nonumber \\
&\quad \quad \quad \quad \quad + \tfrac{\weta \mu}{2}\norm{\Xo -\X_t}_F^2 - \tfrac{\weta L}{2} \norm{\X_t -\X_{t+1}}_F^2 \nonumber \\
&\quad \quad \quad \quad \quad - \tfrac{\weta^2}{5} \|\nabla f(\X_t) \U_t\|_F^2 - \tfrac{\weta \mu \sigma_r(\Xo)}{10} \cdot \dist(U_t, \Uo)^2 \nonumber \\
&\quad \quad \quad \quad \quad + \|\U_{t+1} - \widetilde{\U}_{t+1}\|_F^2 \label{eq:new_p_00}
\end{align} 
 
For the last term in the above expression and given $\U_{t+1} = \Pi_\C\left(\Uw_{t+1}\right) = \xi \cdot \Uw_{t+1}$ for some $\xi \in (0, 1)$, we further observe:
\begin{align*}
\|\U_{t+1} - \widetilde{\U}_{t+1}\|_F^2 &= \|\xi \cdot \widetilde{\U}_{t+1} - \widetilde{\U}_{t+1}\|_F^2 \nonumber \\
&=\left(1 - \xi\right)^2 \cdot \|\U_t\|_F^2 + \left(1 - \xi\right)^2 \weta^2 \cdot \|\gradf(X_t) U_t\|_F^2  \nonumber \\ 
&\quad \quad \quad \quad \quad \quad \quad \quad ~- 2\left(1 - \xi\right)^2 \cdot \weta \cdot \ip{\gradf(X_t)\U_t}{\U_t}
\end{align*} 
Combining the above equality with the first term on the right hand side in \eqref{eq:new_p_00}, we obtain:
\begin{small}
\begin{align*}
\left(1 - \xi^2\right) \cdot \weta \ip{\gradf(X_t)\U_t}{\U_t} &+ \left(1 - \xi\right)^2 \cdot \|\U_t\|_F^2 + \left(1 - \xi\right)^2 \weta^2 \cdot \|\gradf(X_t) U_t\|_F^2  \\&- 2\left(1 - \xi\right)^2 \cdot \weta \cdot \ip{\gradf(X_t)\U_t}{\U_t} = \nonumber \\
\left[\left(1 - \xi^2\right) - 2\left(1 - \xi\right)^2 \right] \cdot \weta \ip{\gradf(X_t)\U_t}{\U_t} &+ \left(1 - \xi\right)^2 \cdot \|\U_t\|_F^2 + \left(1 - \xi\right)^2 \weta^2 \cdot \|\gradf(X_t) U_t\|_F^2 \nonumber = \\
\left(3\xi - 1\right) \left(1 - \xi\right) \cdot \weta \ip{\gradf(X_t)\U_t}{\U_t} &+ \left(1 - \xi\right)^2 \cdot \|\U_t\|_F^2 + \left(1 - \xi\right)^2 \weta^2 \cdot \|\gradf(X_t) U_t\|_F^2 \nonumber = \\
\left\| \tfrac{3\xi - 1}{2} \cdot \U_t + (1 - \xi) \cdot \weta \gradf(\X_t) \cdot \U_t\right\|_F^2 &+ \left(\left(1 - \xi\right)^2 - \tfrac{(3\xi - 1)^2}{4}\right) \|\U_t\|_F^2. 
\end{align*}
\end{small}
Focusing on the first term, let $\Theta_t := I + \tfrac{2(1-\xi)}{3\xi - 1} \cdot \weta \cdot \gradf(\X_t)\Q_{\U_t} \Q_{\U_t}^\top$; then, $\sigma_n(\Theta_t) \geq 1 - \tfrac{2(1-\xi)}{3\xi - 1} \cdot \tfrac{1}{128} $, by the definition of $\weta$ and the fact that $\weta \leq \tfrac{1}{128 \|\gradf(\X_t)\Q_{\U_t} \Q_{\U_t}^\top\|_2}$. Then:
\begin{align*}
\left\| \tfrac{3\xi - 1}{2} \cdot \U_t + (1 - \xi) \cdot \weta \gradf(\X_t) \cdot \U_t\right\|_F^2 &= \left\| \tfrac{3\xi - 1}{2} \Theta_t \cdot \U_t \right\|_F^2 \nonumber \\ 
&\geq \tfrac{(3\xi - 1)^2}{4} \cdot \|\U_t\|_F^2 \cdot \sigma_n(\Theta_t)^2 \nonumber \\
&\geq \tfrac{(3\xi - 1)^2}{4} \cdot \left(1 - \tfrac{2(1-\xi)}{3\xi - 1} \cdot \tfrac{1}{128} \right)^2 \cdot \|\U_t\|_F^2 \nonumber
\end{align*}
Combining the above, we obtain the following bound:
\begin{align*}
\left(1 - \xi^2\right) \cdot \weta \ip{\gradf(X_t)\U_t}{\U_t} &+ \|\U_{t+1} - \widetilde{\U}_{t+1}\|_F^2 \nonumber \\ &\geq \left((1 - \xi)^2 - \tfrac{(3\xi - 1)^2}{4} \cdot \left(1 - \left(1 - \tfrac{2(1-\xi)}{3\xi - 1} \cdot \tfrac{1}{128} \right)^2\right)\right) \cdot \|\U_t\|_F^2
\end{align*}
The above transform \eqref{eq:new_p_00} as follows:
\begin{small}
\begin{align}
2 \weta \big \langle &\gradf(\X_t) \cdot \U_t, ~\U_t - \Uo R_{\U_t}^\star \big \rangle +  \|\U_{t+1} - \widetilde{\U}_{t+1}\|_F^2 \nonumber \\ 
&\geq  \left(\tfrac{255 \cdot \xi^2}{128} - \tfrac{1}{5}\right)\cdot \weta^2\|\gradf(X_t) U_t\|_F^2 + \tfrac{\weta \mu}{2}\norm{\Xo -\X_t}_F^2 - \tfrac{\weta \mu \sigma_r(\Xo)}{10} \cdot \dist(U_t, \Uo)^2 \nonumber \\
&\quad \quad~+\left((1 - \xi)^2 - \tfrac{(3\xi - 1)^2}{4} \cdot \left(1 - \left(1 - \tfrac{2(1-\xi)}{3\xi - 1} \cdot \tfrac{1}{128} \right)^2\right)\right) \cdot \|\U_t\|_F^2 - \tfrac{\weta L}{2} \norm{\X_t -\X_{t+1}}_F^2 \label{eq:new_p_01}
\end{align} 
\end{small}
 
Let us focus on the term $\tfrac{\weta L}{2} \norm{\X_t - \X_{t+1}}_F^2$; this can be bounded as follows:
\begin{align*}
\tfrac{\weta L}{2} \norm{X_t - \X_{t+1}}_F^2 &= \tfrac{\weta L}{2} \|\U_t\U_t^\top - \U_{t+1}\U_{t+1}^\top\|_F^2 = \tfrac{\weta L}{2} \|\U_t\U_t^\top - \U_t\U_{t+1}^\top + \U_t\U_{t+1}^\top - \U_{t+1}\U_{t+1}^\top\|_F^2 \\
&= \tfrac{\weta L}{2} \|\U_t\left(\U_t - \U_{t+1}\right)^\top + \left(\U_t - \U_{t+1}\right)\U_{t+1}^\top\|_F^2 \nonumber \\ 
&\stackrel{(i)}{\leq} \weta L \cdot \left(\|\U_t\left(\U_t - \U_{t+1}\right)^\top\|_F^2 + \|\left(\U_t - \U_{t+1}\right)\U_{t+1}^\top\|_F^2\right) \\
&\stackrel{(ii)}{\leq} \weta L \left(\|\U_{t+1}\|_2^2 + \|\U_t\|_2^2\right) \cdot \|\U_{t+1} - \U_t\|_F^2.
\end{align*} where $(i)$ is due to the identity $\|A + B\|_F^2 \leq 2\|A\|_F^2 + 2\|B\|_F^2$ and $(ii)$ is due to the Cauchy-Schwarz inequality. 
By definition of $\U_{t+1}$, we observe that:
\begin{small}
\begin{align*}
\|\U_{t+1}\|_2^2 = \| \xi \cdot \left(\U_t - \weta\nabla f(\X_t)\U_t\right)\|_2^2 \stackrel{(i)}{\leq} \xi^2 \cdot \|\U_t\|_2^2 \cdot \|I - \weta  \nabla f(\X_t)\Q_{\U_t} \Q_{\U_t}^\top\|_2^2 \stackrel{(ii)}{\leq} \left(1 + \tfrac{1}{128}\right)^2 \cdot \|\U_t\|_2^2.
\end{align*}
\end{small} 
where $(i)$ is due to Cauchy-Schwarz and $(ii)$ is obtained by substituting $\weta \leq \tfrac{1}{128\|\nabla f(\X_t)\Q_{\U_t} \Q_{\U_t}^\top\|_2}$ and since $\xi \in (0, 1)$. 
Thus, $\tfrac{\weta L}{2} \norm{\X_t - \X_{t+1}}_F^2$ can be further bounded as follows:
\begin{align*}
\tfrac{\weta L}{2} \norm{\X_t - \X_{t+1}}_F^2 &\leq \weta L \cdot \left( \left(1 + \tfrac{1}{128}\right)^2 + 1\right) \cdot \|\U_t\|_2^2 \cdot \|\U_{t+1} - \U_t\|_F^2 \\ 
&=  \weta L \cdot \left( \left(1 + \tfrac{1}{128}\right)^2 + 1\right) \cdot \|\X_t\|_2 \cdot \|\U_{t+1} - \U_t\|_F^2 \\ 
&\leq \tfrac{ \left(1 + \tfrac{1}{128}\right)^2 + 1}{128 } \cdot \|\U_{t+1} - \U_t\|_F^2 \\
&= \tfrac{ \left(1 + \tfrac{1}{128}\right)^2 + 1}{128 } \cdot \|\xi \cdot \widetilde{\U}_{t+1} - \U_t\|_F^2 \\
&= \tfrac{ \left(1 + \tfrac{1}{128}\right)^2 + 1}{128 } \cdot \|(\xi -1)\U_t - \xi \cdot \weta \gradf(\X_t) \cdot \U_t\|_F^2 \\
&\leq  (1 - \xi)^2 \cdot \tfrac{ \left(1 + \tfrac{1}{128}\right)^2 + 1}{64 } \cdot \|\U_t\|_F^2 + \tfrac{ \left(1 + \tfrac{1}{128}\right)^2 + 1}{64} \cdot \xi^2 \cdot \weta^2 \cdot  \|\gradf(\X_t) \cdot \U_t\|_F^2 \\
\end{align*} 
where in the last inequality we substitute $\weta$; observe that $\weta \leq \tfrac{1}{128 L \|\X_t\|_2}$.
Combining this result with \eqref{eq:new_p_01}, we obtain:
\begin{small}
\begin{align}
2 \weta \big \langle &\gradf(\X_t) \cdot \U_t, ~\U_t - \Uo R_{\U_t}^\star \big \rangle +  \|\U_{t+1} - \widetilde{\U}_{t+1}\|_F^2 \nonumber \\ 
&\geq  \left(\tfrac{255 \cdot \xi^2}{128} - \tfrac{1}{5} -  \tfrac{ \left(1 + \tfrac{1}{128}\right)^2 + 1}{64} \cdot \xi^2\right)\cdot \weta^2\|\gradf(X_t) U_t\|_F^2 + \tfrac{\weta \mu}{2}\norm{\Xo -\X_t}_F^2 - \tfrac{\weta \mu \sigma_r(\Xo)}{10} \cdot \dist(U_t, \Uo)^2 \nonumber \\
&\quad \quad~+\left((1 - \xi)^2\cdot \left(1 - \tfrac{ \left(1 + \tfrac{1}{128}\right)^2 + 1}{64} \right) - \tfrac{(3\xi - 1)^2}{4} \cdot \left(1 - \left(1 - \tfrac{2(1-\xi)}{3\xi - 1} \cdot \tfrac{1}{128} \right)^2\right)\right) \cdot \|\U_t\|_F^2 \nonumber \\ 
&\stackrel{(i)}{\geq} \weta^2\|\gradf(X_t) U_t\|_F^2 + \tfrac{\weta \mu}{2}\norm{\Xo -\X_t}_F^2 - \tfrac{\weta \mu \sigma_r(\Xo)}{10} \cdot \dist(U_t, \Uo)^2 \nonumber \\
&\quad \quad~+\left((1 - \xi)^2\cdot \left(1 - \tfrac{ \left(1 + \tfrac{1}{128}\right)^2 + 1}{64} \right) - \tfrac{(3\xi - 1)^2}{4} \cdot \left(1 - \left(1 - \tfrac{2(1-\xi)}{3\xi - 1} \cdot \tfrac{1}{128} \right)^2\right)\right) \cdot \|\U_t\|_F^2 \nonumber \\ 
&\stackrel{(ii)}{\geq} \weta^2\|\gradf(X_t) U_t\|_F^2 + \tfrac{\weta \mu}{2}\norm{\Xo -\X_t}_F^2 - \tfrac{\weta \mu \sigma_r(\Xo)}{10} \cdot \dist(U_t, \Uo)^2
\label{eq:new_p_02}
\end{align} 
\end{small}
where $(i)$ is due to the assumption $\xi \gtrsim 0.78$ and thus $\big(\tfrac{255 \cdot \xi^2}{128} - \tfrac{1}{5} -  \tfrac{ \left(1 + \tfrac{1}{128}\right)^2 + 1}{64} \cdot \xi^2\big) \geq 1$; see also Figure \ref{fig:appendix_xi} (left panel), 
and $(ii)$ is due to the non-negativity of the constant in front of $\|\U_t\|_F^2$; see also Figure \ref{fig:appendix_xi} (right panel).

\begin{figure}[!ht]
	\centering
	\includegraphics[width=0.36\textwidth]{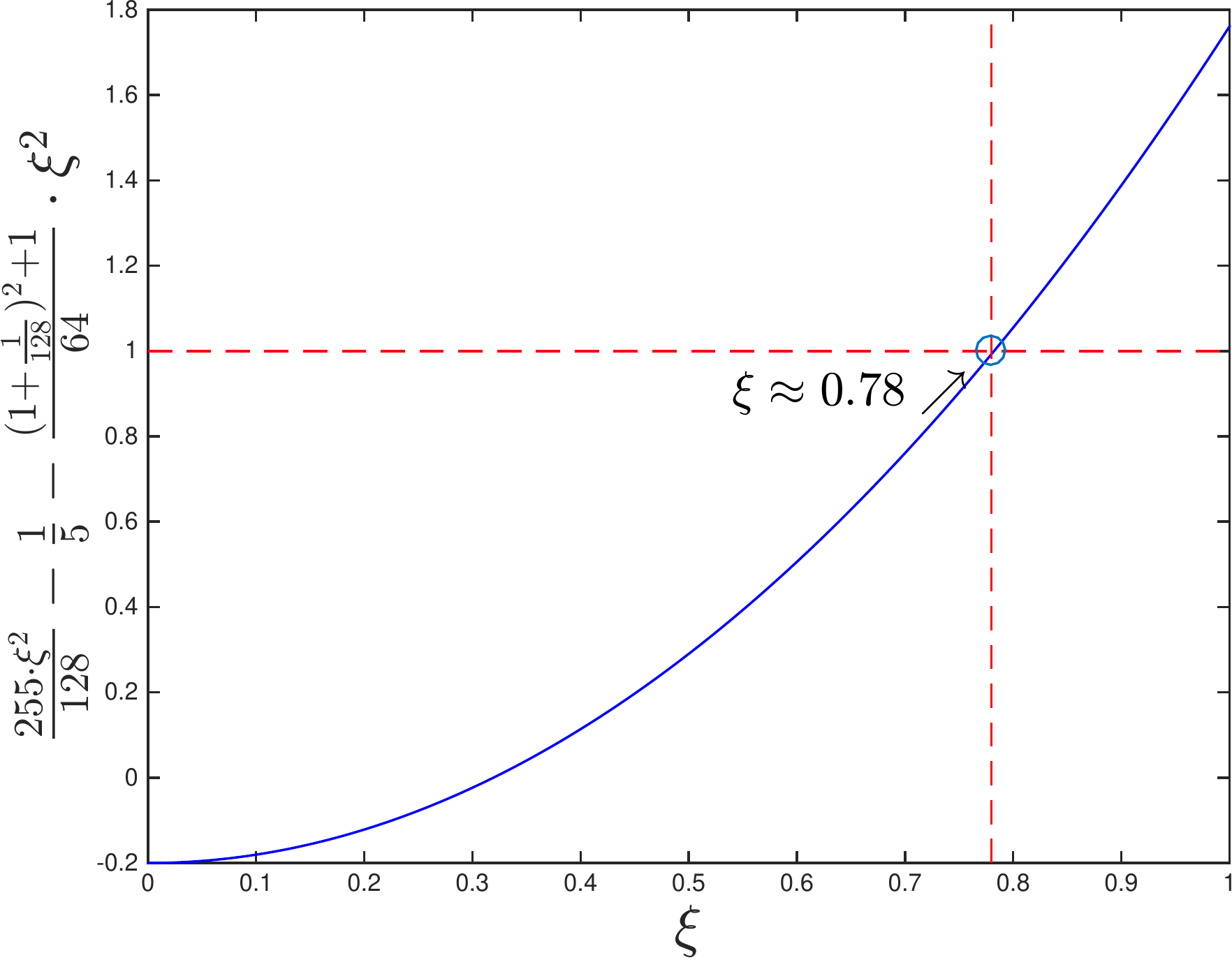} \hspace{0.3cm}
	\includegraphics[width=0.35\textwidth]{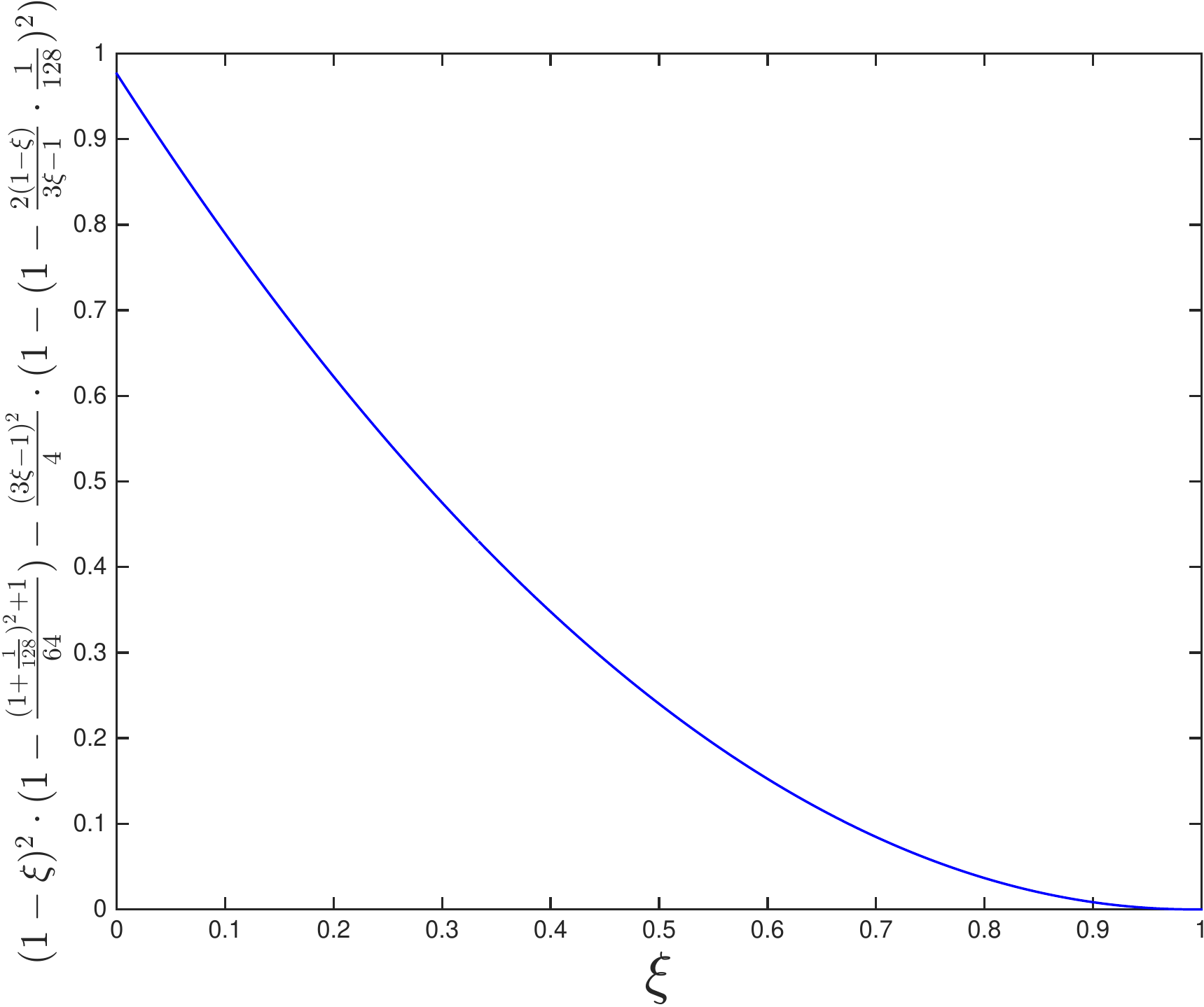}
	\caption{Behavior of constants, depending on $\xi$, in expression \eqref{eq:new_p_02}.} \label{fig:appendix_xi}
\end{figure}

Finally, we bound $\tfrac{\weta \mu}{2}\norm{\Xo -\X_t}_F^2$ using the following Lemma by \cite{tu2015low}:
\begin{lemma}{\label{lem:tu}}
For any $\U, \V \in \R^{n \times r}$, we have:
\begin{align*}
\|\U\U^\top - \V\V^\top\|^2 \geq 2 \cdot \left(\sqrt{2} - 1\right) \cdot \sigma_r(\U)^2 \cdot \dist(\U, \V)^2.
\end{align*} 
\end{lemma}

Thus, 
\begin{align*}
\tfrac{\weta \mu}{2}\norm{\Xo -\X_t}_F^2 \geq \weta \mu \cdot  \left(\sqrt{2} - 1\right) \cdot \sigma_r(\Xo) \cdot \dist(\U_t, \Uo)^2,
\end{align*} and can thus conclude:
\begin{align*}
2 \weta \big \langle &\gradf(\X_t) \cdot \U_t, ~\U_t - \Uo R_{\U_t}^\star \big \rangle +  \|\U_{t+1} - \widetilde{\U}_{t+1}\|_F^2 \nonumber \\ 
&\geq \weta^2\|\gradf(X_t) U_t\|_F^2 +  \weta \mu \cdot  \left(\sqrt{2} - 1\right) \cdot \sigma_r(\Xo) \cdot \dist(\U_t, \Uo)^2 - \tfrac{\weta \mu \sigma_r(\Xo)}{10} \cdot \dist(U_t, \Uo)^2 \\
&= \weta^2\|\gradf(X_t) U_t\|_F^2 +  \left(\sqrt{2} - 1 - \tfrac{1}{10}\right) \cdot \weta \mu\cdot \sigma_r(\Xo) \cdot \dist(\U_t, \Uo)^2 \\
&= \weta^2\|\gradf(X_t) U_t\|_F^2 +  \tfrac{3\weta \mu}{10} \cdot \sigma_r(\Xo) \cdot \dist(\U_t, \Uo)^2 
\end{align*}
This completes the proof.
\end{proof}

\subsection{Proof of Lemma \ref{lem:DD_bound_sc}}{\label{sec:DD_proof}}

\begin{proof}
We can lower bound $\ip{\gradf(\X_t) }{ \Delta \Delta^\top}$ as follows:
 \begin{align}
\ip{\gradf(X_t) }{ \Delta \Delta^\top} &\stackrel{(i)}{=} \ip{Q_{\Delta} Q_{\Delta}^\top \gradf(\X_t) }{ \Delta \Delta^\top} \nonumber \\
&\geq -  \left|\trace\left(Q_{\Delta} Q_{\Delta}^\top \gradf(\X_t) \Delta \Delta^\top\right)\right| \nonumber \\ 
&\stackrel{(ii)}{\geq} - \| Q_{\Delta} Q_{\Delta}^\top \gradf(\X_t) \|_2 \trace( \Delta \Delta^\top)  \nonumber \\
&\stackrel{(iii)}{\geq} - \left( \|Q_{\U_t} Q_{\U_t}^\top\gradf(\X_t)\|_2+  \|Q_{\Uo} Q_{\Uo}^\top\gradf(\X_t)\|_2 \right)  \dist(\U_t, \Uo)^2 .   \label{proofsr1:eq_11}
\end{align}
Note that  $(i)$ follows from the fact $\Delta =Q_{\Delta} Q_{\Delta}^\top \Delta$ and 
$(ii)$ follows from $|\trace(AB)| \leq \|A\|_2 \trace(B)$, for PSD matrix $B$ (Von Neumann's trace inequality~\cite{mirsky1975trace}). 
For the transformation in $(iii)$, we use that fact that the column space of $\Delta$, $\text{\textsc{Span}}(\Delta)$, is a subset of $\text{\textsc{Span}}(\U_t \cup \Uo)$, as $\Delta$ is a linear combination of $\U_t$ and $\Uo R_{U_t}^\star$. 

For the second term in the parenthesis above, we first derive the following inequalities; their use is apparent later on:
\begin{align*}
\|\gradf(\X_t)\Uo\|_2 &\stackrel{(i)}{\leq} \|\gradf(\X_t)\U_t\|_2 +\|\gradf(\X_t)\Delta\|_2 \\
&\stackrel{(ii)}{\leq} \|\gradf(\X_t)\U_t\|_2 +\|\gradf(\X_t)Q_{\Delta}Q_{\Delta}^\top\|_2\|\Delta\|_2 \\
&\stackrel{(iii)}{\leq} \|\gradf(\X_t)\U_t\|_2 +\left( \|\gradf(\X_t)Q_{U_t}Q_{U_t}^\top\|_2 +\|\gradf(\X_t)Q_{\Uo}Q_{\Uo}^\top\|_2 \right) \|\Delta\|_2 \\
&\stackrel{(iv)}{\leq} \|\gradf(\X_t)\U_t\|_2 +\left( \|\gradf(\X_t)Q_{U_t}Q_{U_t}^\top\|_2 +\|\gradf(\X_t)Q_{\Uo}Q_{\Uo}^\top\|_2 \right)\tfrac{1}{200} \sigma_{r}(\Uo)
\end{align*}
\begin{align*}
&\stackrel{(v)}{\leq} \|\gradf(\X_t)\U_t\|_2 +  \tfrac{1}{\left(1-\frac{1}{200}\right)} \cdot \tfrac{1}{200}\|\gradf(\X_t)U_t\|_2 + \tfrac{1}{200} \|\gradf(\X_t)\Uo\|_2 \\
&\leq \tfrac{200}{199}\|\gradf(\X_t)\U_t\|_2 + \tfrac{1}{200} \|\gradf(\X_t)\Uo\|_2.
\end{align*} 
where 
$(i)$ is due to triangle inequality on $\Uo R_{\U_t}^\star = \U_t - \Delta$, 
$(ii)$ is due to generalized Cauchy-Schwarz inequality; we denote as $\Q_\Delta\Q_\Delta$ the projection matrix on the column span of $\Delta$ matrix, 
$(iii)$ is due to triangle inequality and the fact that the column span of $\Delta$ can be decomposed into the column span of $\U_t$ and $\Uo$, by construction of $\Delta$, 
$(iv)$ is due to the assumption $ \dist(\U_t, \Uo) \leq \rho' \cdot \sigma_r(\Uo)$ and 
$$\|\Delta\|_2 \leq \dist(\U_t, \Uo) \leq \tfrac{1}{200} \tfrac{\sigma_r(\Xo)}{\sigma_1(\Xo)} \cdot \sigma_r(\Uo) \leq \tfrac{1}{200} \cdot \sigma_r(\Uo).$$ 
Finally, $(v)$ is due to the facts: $$\|\gradf(\X_t) \Uo\|_2 = \|\gradf(\X_t) \Q_{\Uo}\Q_{\Uo}^\top \Uo\|_2 \geq \|\gradf(\X_t) \Q_{\Uo}\Q_{\Uo}^\top\|_2 \cdot \sigma_r(\Uo),$$ and 
\begin{align*}
\|\gradf(\X_t) \U_t\|_2 &= \|\gradf(\X_t) \Q_{\U_t}\Q_{\U_t}^\top \U\|_2 \geq \|\gradf(\X_t) \Q_{\U_t}\Q_{\U_t}^\top\|_2 \cdot \sigma_r(\U_t) \nonumber \\ 
&\geq \|\gradf(\X_t) \Q_{\U}\Q_{\U}^\top\|_2 \cdot \left(1 - \tfrac{1}{200}\right) \cdot \sigma_r(\Uo),
\end{align*} by the proof of (a variant of) Lemma A.3 in \cite{bhojanapalli2015dropping}. 
Thus, for the term $ \|\gradf(\X_t)Q_{\Uo} Q_{\Uo}^\top\|_2$, we have
\begin{align}\label{proofsjc:eq_067}
 \|\gradf(\X_t)Q_{\Uo} Q_{\Uo}^\top\|_2  &\leq \tfrac{1}{\sigma_r(\Uo)}\|\gradf(\X_t)\Uo\|_2 \nonumber \\ 
 &\leq \tfrac{1}{\sigma_r(\Uo)}\tfrac{201}{199} \|\gradf(\X_t)\U_t\|_2 \nonumber \\ 
 &\leq \tfrac{201 \sigma_1(\Uo)}{200 \sigma_r(\Uo)}\tfrac{201}{199} \|\gradf(\X_t)Q_{U_t} Q_{U_t}^\top\|_2.
\end{align} 

Using \eqref{proofsjc:eq_067} in \eqref{proofsr1:eq_11}, we obtain:
\begin{align}
\ip{\gradf(X_t) }{ \Delta \Delta^\top} &\geq - \left( \|Q_{\U_t} Q_{\U_t}^\top\gradf(\X_t)\|_2 +  \tfrac{201 \sigma_1(\Uo)}{200 \sigma_r(\Uo)}\tfrac{201}{199} \|Q_{U_t} Q_{U_t}^\top\gradf(\X_t)\|_2 \right)  \dist(\U_t, \Uo)^2 \nonumber \\ 
&\geq - \tfrac{21 \cdot \tau(\Uo)}{10} \|Q_{U_t} Q_{U_t}^\top\gradf(\X_t)\|_2 \dist(\U_t, \Uo)^2 \nonumber
\end{align}

We remind that the step size we use here is: $\widehat{\eta} =  \tfrac{1}{128(L\|\X_t\|_2 + \|Q_{U_t} Q_{U_t}^\top\gradf(\X_t)\|_2)} $.
Then, we have:
\begin{small}
\begin{align}
&\tfrac{21 \cdot \tau(\Uo)}{10} \cdot \|Q_{U_t} Q_{U_t}^\top\gradf(\X_t)\|_2 \cdot \dist(\U_t, \Uo)^2 \nonumber \\ 
 &\quad \quad \quad \quad \quad \quad \leq \tfrac{21 \cdot \tau(\Uo)}{10} \cdot \weta \cdot 128 L\|\X_t\|_2 \|Q_{\U_t} Q_{\U_t}^\top\gradf(\X_t)\|_2 \cdot \dist(\U_t, \Uo)^2 \nonumber \\ 
 &\quad \quad \quad \quad \quad \quad \quad \quad \quad \quad \quad \quad + \tfrac{21 \cdot \tau(\Uo)}{10} \cdot \weta \cdot 128 \cdot \|Q_{\U_t} Q_{\U_t}^\top\gradf(\X_t)\|_2^2 \cdot  \dist(\U_t, \Uo)^2 \label{eq:dff_00}
\end{align} 
\end{small}
To bound the first term on the right hand side, we observe that $\|Q_{\U_t} Q_{\U_t}^\top\gradf(\X_t)\|_2 \leq \tfrac{\mu \sigma_r(\X_t)}{\tfrac{21 \cdot \tau(\Uo)}{10} \cdot 10}$ or $\|Q_{\U_t} Q_{\U_t}^\top\gradf(\X_t)\|_2 \geq \tfrac{\mu \sigma_r(\X_t)}{\tfrac{21 \cdot \tau(\Uo)}{10} \cdot 10}$.
This results further into:
\begin{small}
\begin{align*}
\tfrac{21 \cdot \tau(\Uo)}{10} &\cdot \weta \cdot 128 L\|\X_t\|_2 \|Q_{\U_t} Q_{\U_t}^\top\gradf(\X_t)\|_2 \cdot \dist(\U_t, \Uo)^2 \\ 
																				&\leq \max \bigg\{\tfrac{\tfrac{21 \cdot \tau(\Uo)}{10} \cdot 128 \cdot \weta \cdot L \|X_t\|_2 \cdot \mu \sigma_r(\X_t)}{\tfrac{21 \cdot \tau(\Uo)}{10} \cdot 10} \cdot \dist(\U_t, \Uo)^2, \nonumber \\ &\quad \quad \quad \quad \quad \quad \quad \weta \left(\tfrac{21 \cdot \tau(\Uo)}{10}\right)^2 \cdot 128 \cdot 10 \kappa \tau(\X_t) \|Q_{\U_t} Q_{\U_t}^\top\gradf(\X_t)\|_2^2 \cdot  \dist(\U_t, \Uo)^2 \bigg\} \\
																				 &\leq \tfrac{128 \cdot \weta \cdot L \|X_t\|_2 \cdot \mu \sigma_r(\X_t)}{10} \cdot \dist(\U_t, \Uo)^2 \nonumber \\ &\quad \quad \quad \quad \quad \quad \quad+ \weta \left(\tfrac{21 \cdot \tau(\Uo)}{10}\right)^2 \cdot 128 \cdot 10 \kappa \tau(\X_t) \|Q_{\U_t} Q_{\U_t}^\top\gradf(\X_t)\|_2^2 \cdot  \dist(\U_t, \Uo)^2,
\end{align*} 
\end{small} where $\kappa := \tfrac{L}{\mu}$ and $\tau(\X) := \tfrac{\sigma_1(\X)}{\sigma_r(\X)}$ for a rank-$r$ matrix $\X$.
Combining the above with \eqref{eq:dff_00}:
\begin{small}
\begin{align*}
&\tfrac{21 \cdot \tau(\Uo)}{10} \cdot \|Q_{U_t} Q_{U_t}^\top\gradf(\X_t)\|_2 \cdot \dist(\U_t, \Uo)^2 \nonumber \\ 
&\stackrel{(i)}{\leq}  \tfrac{\mu \sigma_{r}(\X_t)}{10} \cdot \dist(\U_t, \Uo)^2 \nonumber \\ 
&\quad \quad \quad \quad \quad \quad \quad \quad+ \left(10 \kappa \tau(\X_t) \cdot \tfrac{21 \cdot \tau(\Uo)}{10} +1 \right)\cdot \tfrac{21 \cdot \tau(\Uo)}{10} \cdot  128 \cdot \weta \|Q_{\U_t} Q_{\U_t}^\top\gradf(\X_t)\|_2^2 \cdot \dist(\U_t, \Uo)^2  \nonumber \\
&\stackrel{(ii)}{\leq}  \tfrac{\mu \sigma_{r}(\X_t)}{10} \cdot \dist(\U_t, \Uo)^2 \nonumber \\ 
&\quad \quad \quad \quad \quad \quad \quad \quad+ \left(11 \kappa \tau(\Xo) \cdot \tfrac{21 \cdot \tau(\Uo)}{10} +1 \right)\cdot \tfrac{21 \cdot \tau(\Uo)}{10} \cdot  128 \cdot \weta \|Q_{\U_t} Q_{\U_t}^\top\gradf(\X_t)\|_2^2 \cdot (\rho')^2 \sigma_{r}(\Xo)   \nonumber \\
&\stackrel{(iii)}{\leq}  \tfrac{\mu \sigma_{r}(\X_t)}{10} \cdot \dist(\U_t, \Uo)^2 + \tfrac{12 \cdot 21^2}{10^2} \cdot \kappa \cdot \tau(\Xo)^2 \cdot 128 \cdot \weta \|\gradf(\X_t)\U_t\|_2^2 \cdot \tfrac{11 \cdot (\rho')^2}{10} \nonumber \\
&\stackrel{(iii)}{\leq}  \tfrac{\mu \sigma_{r}(\X_t)}{10} \cdot \dist(\U_t, \Uo)^2+ \tfrac{\weta}{5} \|\gradf(\X_t)\U_t\|_2^2 \nonumber
\end{align*}
\end{small}
where $(i)$ follows from $\weta \leq \tfrac{1}{128 L \|\X_t\|_2}$, 
$(ii)$ is due to Lemma A.3 in \cite{bhojanapalli2015dropping} and using the bound $\dist(\U_t, \Uo) \leq \rho' \sigma_{r}(\Uo)$ by the hypothesis of the lemma, 
$(iii)$ is due to $\sigma_{r}(\Xo) \leq 1.1 \sigma_{r}(\X_t)$ by Lemma A.3 in \cite{bhojanapalli2015dropping}, due to the facts $\sigma_{r}(\X_t)\|Q_{\U_t} Q_{\U_t}^\top\gradf(\X_t)\|_2^2 \leq  \|U_t^\top\gradf(X_t)\|_F^2$ and $(11 \kappa \tau(\Xo) \cdot \tfrac{21 \cdot \tau(\Uo)}{10} +1) \leq 12 \kappa \tau(\Xo) \cdot \tfrac{21 \cdot \tau(\Uo)}{10}$, and $\tau(\Uo)^2 = \tau(\Xo)$. 
Finally, $(iv)$ follows from substituting $\rho' := c \cdot \tfrac{1}{\kappa} \cdot \tfrac{1}{\tau(\Xo)}$ for $c = \tfrac{1}{200}$ and using Lemma A.3 in \cite{bhojanapalli2015dropping} (due to the factor $\tfrac{1}{200}$, all constants above lead to bounding the term with the constant $\tfrac{1}{5}$).

Thus, we can conclude:
\begin{small}
\begin{align*}
\ip{\gradf(\X_t) }{ \Delta \Delta^\top} \geq - \left(\tfrac{\weta}{5} \|\gradf(\X_t) U_t \|_F^2 + \tfrac{\mu \sigma_{r}(\Xo)}{10} \cdot \dist(\U_t, \Uo)^2\right).
\end{align*}
\end{small}
This completes the proof.
\end{proof}

\subsection{Proof of Corollary \ref{cor:projFGD_schatten}}

We have
\begin{align*}
\| \widetilde{U}_{t+1} \|_F
&\le \| \U_t \|_F + \widehat{\eta} \cdot \| \gradf(X_t) U_t \|_F \\
&\le \| \U_t \|_F + \widehat{\eta} \cdot \| \gradf(X_t) \Q_{\U_t} \Q_{\U_t}^\top\|_2 \cdot \| U_t \|_F \\
&= (1 + \widehat{\eta} \cdot \| \gradf(X_t) \Q_{\U_t} \Q_{\U_t}^\top\|_2) \cdot \lambda \\
&\le (1 + \tfrac{1}{128}) \cdot \lambda
\end{align*}
where the first inequality follows from the triangle inequality, the second holds by the property $\|AB\|_F \le \|A\|_2 \cdot \|B\|_F$, and the third follows because the step size is bounded above by $\widehat{\eta} \le \frac{1}{128 \|\gradf(X_t)\Q_{\U_t} \Q_{\U_t}^\top\|_2}$. Hence, we get $\xi(\widetilde{U}_{t+1}) = \frac{\lambda}{\| \widetilde{U}_{t+1} \|_F} \ge \frac{128}{129}$.

\section*{Initialization}{\label{sec:init}}
In this section, we present a specific initialization strategy for the \palgo. 
For completeness, we repeat the definition of the optimization problem at hand, both in the original space:
\begin{equation}{\label{init:eq_01}}
\begin{aligned}
	& \underset{\X \in \R^{n \times n}}{\text{minimize}}
	& & f(\X) \quad \quad \text{subject to} \quad \X \in \C'.
\end{aligned}
\end{equation} 
and the factored space:
\begin{equation}{\label{init:eq_00}}
\begin{aligned}
	& \underset{\U \in \R^{n \times r}}{\text{minimize}}
	& & f(\U\U^\top) \quad \quad \text{subject to} \quad \U \in \C.
\end{aligned}
\end{equation} 
For our initialization, we restrict our attention to the full rank ($r = n$) case.
Observe that, in this case, $\C'$ is a convex set and  includes the full-dimensional PSD cone, as well as other norm constraints, as described in the main text.
Let us denote $\Pi_{\C'}(\cdot)$ the corresponding projection step, where all constraints are satisfied simultaneously.
Then, the initialization we propose follows similar motions with that in \cite{bhojanapalli2015dropping}:
We consider the projection of the weighted negative gradient at $0$, \emph{i.e.}, $-\tfrac{1}{L} \cdot \nabla f(0)$, onto $\C'$.\footnote{As in \cite{bhojanapalli2015dropping}, one can approximate easily $L$, if it is unknown.} \emph{I.e.}, 
\begin{equation}
\X_0 = \U_0 \U_0^\top = \Pi_{\C'}\left(\tfrac{-1}{L} \cdot \nabla f(0) \right).\label{eq:projinit}
\end{equation}
Assuming a first-oracle model, where we access $f$ only though function evaluations and gradient calculations, \eqref{eq:projinit} provides a cheap way to find an initial point with some approximation guarantees as follows\footnote{As we show in the experiments section, a random initialization performs well in practice, without requiring the additional calculations involved in \eqref{eq:projinit}. However, a random initialization provides no guarantees whatsoever.}:

\begin{lemma}
Let $\U_0 \in \R^{n \times n}$ be such that $\X_0 = \U_0\U_0^\top = \Pi_{\C'}\left(\tfrac{-1}{L} \cdot \nabla f(0) \right)$. 
Consider the problem in \eqref{init:eq_00} where $f$ is assumed to be $L$-smooth and $\mu$-strongly convex, with optimum point $\X^\star$ such that $\text{rank}(\X^\star) = n$. 
We apply \palgo algorithm with $\U_0$ as the initial point. 
Then, in this generic case, $\U_0$ satisfies:
\begin{align*}
\dist(\U_0, \Uo) \leq \rho' \cdot \sigma_r(\Uo),
\end{align*} where $\rho' = \sqrt{\tfrac{1 - \sfrac{\mu}{L}}{2(\sqrt{2}-1)}} \cdot \tau^2(\Uo) \cdot \sqrt{\texttt{srank}(\Xo)}$ and $\texttt{srank}(\X) = \tfrac{\|X\|_F}{\|X\|_2}$.
\end{lemma}

\begin{proof}
To show this, we start with:
\begin{align}
\|X_0 -\Xo\|_F^2 = \|\Xo\|_F^2 + \|X_0\|_F^2 -2\ip{X_0}{\Xo}\label{eq:projinit0}.  
\end{align}
Recall that $\X_0 = \U_0\U_0^\top = \Pi_{\C'}\left(\tfrac{-1}{L} \cdot \nabla f(0) \right)$ by assumption, where $\Pi_{\C'}(\cdot)$ is a convex projection.
Then, by Lemma \ref{lem:proj}, we get
\begin{align}
\ip{\X_0 - \Xo}{-\tfrac{1}{L}\cdot \nabla f(0) - \Xo} \Rightarrow \ip{\tfrac{-1}{L}\nabla f(0)}{X^0 -X^*} \geq \ip{X^0}{X^0-\Xo}. \label{eq:projinit1}
\end{align}

Observe that $0 \in \R^{n \times n}$ is a feasible point, since it is PSD and satisfy any common \emph{symmetric} norm constraints, as the ones considered in this paper. 
Hence, using strong convexity of $f$ around $0$, we get,
\begin{align}
    f(\Xo)-\frac{\mu}{2}\|\Xo\|_F^2 &\geq f(0)+\ip{\gradf(0)}{\Xo} \nonumber \\
    &\stackrel{(i)}{=} f(0)+\ip{\gradf(0)}{X_0}+\ip{\gradf(0)}{\Xo - X_0} \nonumber \\
    &\stackrel{(ii)}{\geq} f(0)+\ip{\gradf(0)}{X_0}+\ip{L\cdot X_0}{X_0 - \Xo}.\label{eq:projinit2}
\end{align} 
where $(i)$ is by adding and subtracting $\ip{\nabla f(0)}{\X_0}$, and
$(ii)$ is due to \eqref{eq:projinit1}.
Further, using the smoothness of $f$ around $0$, we get:
\begin{align}
    f(X_0) &\leq f(0) + \ip{\gradf(0)}{X_0} + \tfrac{L}{2}\|X_0\|_F^2 \nonumber\\
    &\stackrel{(i)}{\leq} f(\Xo)-\tfrac{\mu}{2}\|\Xo\|_F^2 + \ip{L \cdot X_0}{\Xo} -\tfrac{L}{2}\|X_0\|_F^2 \nonumber \\
    &\leq f(X_0) - \tfrac{\mu}{2}\|\Xo\|_F^2 + \ip{L \cdot X_0}{\Xo} - \tfrac{L}{2}\|X_0\|_F^2. \nonumber
\end{align} where
$(i)$ follows from \eqref{eq:projinit2} by upper bounding the quantity $f(0) + \ip{\gradf(0)}{X_0}$,
$(ii)$ follows from the assumption that $f(\Xo) \leq f(\X_0)$. 
Hence, rearranging the above terms, we get:
\begin{align*}
    \ip{X_0}{\Xo} \geq \tfrac{1}{2}\|X_0\|_F^2 + \tfrac{\mu}{2L}\|\Xo\|_F^2.
\end{align*}
Combining the above inequality with~\eqref{eq:projinit0}, we obtain,
\begin{align*}
    \| \X_0 - \Xo\|_F \leq  \sqrt{1-\tfrac{\mu}{L}} \cdot \|\Xo\|_F.
\end{align*}
Given, $\U_0$ such that $\X_0 = \U_0\U_o^\top$ and $\Uo$ such that $\Xo = \Uo\U^{\star\top}$, we use Lemma \ref{lem:tu} from \cite{tu2015low} to obtain:
\begin{align*}
\|\U_0\U_0^\top - \Uo\U^{\star \top}\|_F \geq \sqrt{2 (\sqrt{2} - 1)} \cdot \sigma_r(\Uo) \cdot \dist(U_0, \Uo).
\end{align*} 
Thus:
\begin{align*}
\dist(U_0, \Uo) &\leq \tfrac{\|\X_0 - \Xo\|_F}{\sqrt{2(\sqrt{2}-1)} \cdot \sigma_r(\Uo)} \cdot \|\Xo\|_F \nonumber \\ 
					  &\leq \rho' \cdot \sigma_r(\Uo)
\end{align*} where $\rho' =  \sqrt{\tfrac{1 - \sfrac{\mu}{L}}{2(\sqrt{2}-1)}} \cdot \tau^2(\Uo) \cdot \sqrt{\texttt{srank}(\Xo)}$.
\end{proof}
Such initialization, while being simple, introduces further restrictions on the condition number $\tau(\Xo)$, and the condition number of function $f$. 
Finding such simple initializations with weaker restrictions remains an open problem; however, as shown in \cite{bhojanapalli2015dropping, tu2015low, chen2015fast}, one can devise specific deterministic initialization for a given application.

As a final comment, we state the following: In practice, the projection $\Pi_{\C'}(\cdot)$ step might not be easy to compute, due to the joint involvement of convex sets. 
A practical solution would be to sequentially project $-\tfrac{1}{L} \cdot \nabla f(0)$ onto the individual constraint sets.
Let $\Pi_{+}(\cdot)$ denote the projection onto the PSD cone.
Then, we can consider the approximate point:
\begin{align*}
\widetilde{\X}_0 = \widetilde{\U}_0 \widetilde{\U}_0^\top = \Pi_{+} \left( \widetilde{\X}_0 \right);
\end{align*} 
Given $\widetilde{\U}_0$, we can perform an additional step:
$$ 
\U_0 = \Pi_{\C}\left(\widetilde{\U}_0\right),
$$
to guarantee that $\U_0 \in \C$.
\end{document}